\documentclass{article} % For LaTeX2e

\usepackage{colm2024_conference}

\usepackage{microtype}
\usepackage{hyperref}
\usepackage{url}
\usepackage{booktabs}
\usepackage{float}
\usepackage{tocloft} % For customizing the table of contents
\usepackage{placeins}

\usepackage{natbib}
\usepackage{amsmath}
\usepackage{xparse}
\usepackage{titletoc}

\definecolor{mediumblue}{RGB}{100, 149, 237}
\usepackage{colortbl}

\makeatletter
\def\@fnsymbol#1{\ensuremath{\ifcase#1\or \dagger\or \ddagger\or
   \mathsection\or \mathparagraph\or \|\or **\or \dagger\dagger
   \or \ddagger\ddagger \else\@ctrerr\fi}}
    \makeatother

\usepackage[mathscr]{euscript}

\newcommand{\ebt}{\widetilde{\vct{e}}}

\newcommand{\xd}{\bar{\x}}

\newcommand{\Hmm}{\Hb^{\rm{mm}}}
\newcommand{\pih}{\hat{\pi}}
\newcommand{\ph}{\hat{p}}
\newcommand{\pbh}{\hat{\pb}}
\newcommand{\CE}{\operatorname{CE}}
\newcommand{\CEin}{\operatorname{CE}_{\rm{in}}}
\newcommand{\CEinstar}{\operatorname{CE}_{\rm{in},\star}}

\newcommand{\ent}{\Hc}

\newcommand{\Wmm}{\W^{\rm{mm}}}

\newcommand{\NTP}{\text{NTP}\xspace}
\newcommand{\NTPH}{\text{NTP}$_{\ent}$\xspace}

\newcommand{\Eb}{\mtx{E}}
\newcommand{\What}{\widehat{\W}}
\newcommand{\ellb}{\ell}

\newcommand{\ellbb}{\boldsymbol{\ell}}
\newcommand{\ellbhat}{\widehat{\ellbb}}

\newcommand{\Ts}{\mathscr{F}}
\newcommand{\Tc}{\mathcal{T}}

\newcommand{\Lbmm}{\Lb^{\rm{mm}}}

\newcommand{\GHmmo}{\overline{\G}_{\Hb}^{\rm{mm}}}

\newcommand{\GWmmo}{\overline{\G}_{\W}^{\rm{mm}}}

\newcommand{\Lbtina}{\smatbar}

\newcommand{\Lbin}{\Lb^{\rm{in}}}

\newcommand{\Ebin}[1]{\Eb_{\rm{in},#1}}
\newcommand{\Ebint}[1]{\widetilde{\Eb}_{\rm{in},#1}}
\newcommand{\Ebout}[1]{\Eb_{\rm{out},#1}}

\newcommand{\sqrtm}[1]{#1^{\frac{1}{2}}}
\newcommand{\rankset}{\mathcal{R}}
\newcommand{\smat}{\Sb}
\newcommand{\smatbar}{\widetilde{\smat}}
\newcommand{\cosa}[2]{{\operatorname{cos}}\big(#1,#2\big)}
\newcommand{\Qcs}{\mathcal{P}_{\Ts}}
\newcommand{\Qcsperp}{\mathcal{P}_{\perp}}
\newcommand{\GW}{\G_\W}
\newcommand{\GH}{\G_{\Hb}}
\newcommand{\PMI}{\textbf{PMI}}

\newcommand{\Lb}{\vct{L}}
\newcommand{\Lbhat}{\widehat{\vct{L}}}

\newcommand{\rank}[1]{\operatorname{rank}\left(#1\right)}

\newcommand{\laz}{{\la\rightarrow0}}
\newcommand{\Binf}{{B\rightarrow\infty}}
\newcommand{\Rinf}{{R\rightarrow\infty}}

\usepackage{hyperref,graphicx,amsmath,amsfonts,amssymb,bm,url,breakurl,epsfig,epsf,color,fmtcount,semtrans,caption,multirow,comment, boldline, amsthm}
\usepackage{subcaption}
\usepackage{bm}
\usepackage{enumitem}
\usepackage{amssymb}

\setlist[itemize]{leftmargin=5mm}

\usepackage{hyperref}
\usepackage[utf8]{inputenc} % allow utf-8 input
\usepackage[T1]{fontenc}    % use 8-bit T1 fonts
\usepackage{url}            % simple URL typesetting
\usepackage{booktabs}       % professional-quality tables
\usepackage{nicefrac}       % compact symbols for 1/2, etc.
\usepackage{microtype}      % microtypography
\usepackage{dsfont}
\usepackage{xspace}

\newcommand{\vb}{{\vct{v}}}

\newcommand{\Vb}{{\mtx{V}}}
\newcommand{\sft}[1]{\mathbb{S}(#1)}
\newcommand{\sfti}[2]{\mathbb{S}_{#1}(#2)}

\newcommand{\X}{{\mtx{X}}}

\newcommand{\Pbb}{{\mtx{P}}}

\newcommand{\vct}[1]{\bm{#1}}
\newcommand{\mtx}[1]{\bm{#1}}

\usepackage{mathtools}

\usepackage{titlesec}

\usepackage{tikz}
\usepackage{pgfplots}
\usetikzlibrary{pgfplots.groupplots}

\usepackage{movie15}

\usepackage{caption}
\usepackage[bottom,hang,flushmargin]{footmisc} 

\newcommand{\tsn}[1]{{\left\vert\kern-0.25ex\left\vert\kern-0.25ex\left\vert #1 
    \right\vert\kern-0.25ex\right\vert\kern-0.25ex\right\vert}}

\definecolor{darkred}{RGB}{150,0,0}
\definecolor{darkgreen}{RGB}{0,150,0}
\definecolor{darkblue}{RGB}{0,0,200}
\hypersetup{colorlinks=true, linkcolor=darkred, citecolor=darkgreen, urlcolor=darkblue}

\newtheorem{theorem}{Theorem}
\newtheorem{lemma}{Lemma}
\newtheorem{corollary}{Corollary}
\newtheorem{proposition}{Proposition}
\newtheorem{definition}{Definition}
\newcommand{\Rb}{\mathbf{R}}

\newcommand{\diag}[1]{\operatorname{diag}(#1)}

\DeclareMathOperator{\tr}{tr}

\newcommand{\synth}{\texttt{Synthetic}}
\newcommand{\simTS}{\texttt{Simplified TinyStories}}
\newcommand{\TS}{\texttt{TinyStories}}

\newcommand{\Eqref}[1]{Eq.~\eqref{#1}}
\newcommand{\appropto}{\mathrel{\vcenter{
  \offinterlineskip\halign{\hfil$##$\cr
    \propto\cr\noalign{\kern2pt}\sim\cr\noalign{\kern-2pt}}}}}
    
\newcommand{\cut}[1]{\textcolor{red}{}}
\newcommand{\W}{{\vct{W}}}
\newcommand{\Ab}{{\vct{A}}}
\newcommand{\Xb}{{\vct{X}}}

\newcommand{\Yb}{\mathbf{Y}}

\newcommand{\corr}[1]{\textsc{corr}(#1)}
\newcommand{\ssim}[2]{\textsc{sim}\big(#1,#2\big)}
\newcommand{\ssimstar}[2]{\textsc{sim}\big(#1,#2\big)}

\newcommand{\Ub}{\vct{U}}
\newcommand{\G}{\vct{G}}

\newcommand{\Hb}{{\vct{H}}}
\newcommand{\Sb}{\vct{S}}

\newcommand{\A}{\vct{A}}

\newcommand{\pb}{\vct{p}}

\newcommand{\lambdab}{\boldsymbol{\lambda}}

\newcommand{\thetab}{\boldsymbol{\theta}}

\newcommand{\x}{\vct{x}}

\newcommand{\ub}{\vct{u}}
\newcommand{\w}{{\vct{w}}}

\newcommand{\eb}{\vct{e}}

\newcommand{\z}{\vct{z}}

\newcommand{\ab}{\vct{a}}

\newcommand{\oneb}{\mathbf{1}}
\newcommand{\hb}{\vct{h}}

\newcommand{\Vc}{\mathcal{V}}
\newcommand{\Fc}{\mathcal{F}}

\newcommand{\Sc}{{\mathcal{S}}}

\newcommand{\Xc}{\mathcal{X}}

\newcommand{\Nc}{\mathcal{N}}
\newcommand{\Rc}{\mathcal{R}}

\newcommand{\Lc}{\mathcal{L}}

\newcommand{\Pc}{\mathcal{P}}

\newcommand{\Hc}{\mathcal{H}}

\newcommand{\beq}{\begin{equation}}
\newcommand{\eeq}{\end{equation}}
\newcommand{\bea}{\begin{align}}
\newcommand{\eea}{\end{align}}

\newcommand{\R}{\mathbb{R}}
\newcommand{\E}{\mathbb{E}}

\newcommand{\nn}{\notag}

  \newcommand{\Sigmab}{\boldsymbol\Sigma}
    
  \newcommand{\la}{{\lambda}}                     % lambda
  \newcommand{\eps}{\epsilon}

\newcommand{\ind}[1]{\mathds{1}[#1]}

\DeclarePairedDelimiterX{\inp}[2]{\langle}{\rangle}{#1, #2}

\newcommand{\Id}{\mathds{I}}
\newcommand{\ones}{\mathds{1}}

\providecommand{\norm}[1]{\lVert#1\rVert}
\providecommand{\abs}[1]{\left\lvert#1\right\rvert}

\newcounter{myitemcounter}

\makeatletter
\renewcommand\p@myitemcounter{Q}
\makeatother

\usepackage{tikz}
\usepackage{tcolorbox,wrapfig}
\usepackage[labelfont={bf},font=small]{caption}
\usepackage{subcaption}
\makeatletter
\def\thanks#1{\protected@xdef\@thanks{\@thanks
        \protect\footnotetext{#1}}}
\makeatother

\title{Implicit Geometry of Next-token Prediction: \\From Language Sparsity Patterns to Model Representations
}

\author{Yize Zhao$^{\dagger}$, Tina Behnia$^{\dagger}$, Vala Vakilian \& Christos Thrampoulidis \thanks{\hspace{-16pt}$^{\dagger}$ Equal contribution.} \\
Department of Electrical \& Computer Engineering\\
University of British Columbia\\
Vancouver, Canada\\
\texttt{\{zhaoyize,tina.behnia,vaalaa,cthrampo\}@ece.ubc.ca} \\
}

\colmfinalcopy % Uncomment for camera-ready version, but NOT for submission.
\begin{document}

% \begin{document}

% \tableofcontents % Main Table of Contents

\maketitle
% \footnotetext{
% \textsuperscript{2}University of British Columbia. 
% \hspace{-0.1em}Email: {\fontsize{9}{11}\colorlinkmailto{cthrampo@ece.ubc.ca}{black}{\texttt{cthrampo}}@ece.ubc.ca}}. This work is supported by NSERC Discovery Grant RGPIN-2021-03677 and NSF CAREER Award CCF-2239265. The authors acknowledge use of the Discovery cluster by USC Advanced Research Computing.}

\begin{abstract}
Next-token prediction (NTP) over large text corpora has become the go-to paradigm to train large language models. Yet, it remains unclear how NTP influences the mapping of linguistic patterns to geometric properties of the resulting model representations. We frame training of large language models as soft-label classification over sparse probabilistic label vectors, coupled with an analytical approximation that allows unrestricted generation of context embeddings. This approach links NTP training to rank-constrained, nuclear-norm regularized optimization in the logit domain, offering a framework for analyzing the geometry of word and context embeddings. In large embedding spaces, we find that NTP implicitly favors learning logits with a sparse plus low-rank structure.  While the sparse component captures the co-occurrence frequency of context-word pairs, the orthogonal low-rank component, which becomes dominant as training progresses, depends solely on the sparsity pattern of the co-occurrence matrix. Consequently, when projected onto an appropriate subspace, representations of contexts that are followed by the same set of next-tokens collapse—a phenomenon we term subspace-collapse. We validate our findings on synthetic and small-scale real language datasets. Finally, we outline potential research directions aimed at deepening the understanding of NTP's influence on the learning of linguistic patterns and regularities.\footnote{Code available at: \url{https://github.com/YizeZhao/Implicit_Geometry_of_NTP}.}
%we show that NTP logits align with solutions to a convex nuclear-norm minimization, which akin to support-vector-machines, maximizes margin between in-support and out-of-support tokens, while additionally imposing equality among in-support tokens. This allows us to examine how the embeddings' geometry varies with distributional properties of the training data, connecting geometric features to textual structures.
\end{abstract}

% \begin{keywords}%
% Next-token prediction, classification, optimization, overparameterization
% \end{keywords}

% \input{Arxiv_NewTempl/sections/intro_old}
% \section{Introduction}
% \vspace{-10pt}
\vspace{-10pt}
\section{Introduction}
\vspace{-10pt}

Next-token prediction (NTP) is the preferred training paradigm for state-of-the-art language models. The process, elegantly simple, uses a large training corpus to minimize, for each context $\z_{<t}\in\Vc^{t-1}$ of $t-1$ preceding tokens, the cross-entropy (CE) loss between the model's predicted conditional probability distribution over potential next tokens from a vocabulary $\Vc$ and the one-hot encoded actual next token $z_t\in\Vc$. The model's conditional distribution is defined through a softmax map applied to logits $\ellbb_{<t}(\W,\thetab)=\W\hb_{\thetab}(\z_{<t})$, which are generated by mapping \emph{context embeddings} $\hb_{\thetab}(\z_{<t})\in\R^d$—the neural network's $d$-dimensional  representations of contexts $\z_{<t}$—using a matrix $\W\in\R^{|\Vc|\times d}$ of \emph{word embeddings}.%\footnote{\new{Considering the variety of tokenization methods, the rows of $\W$ are broadly considered as token embeddings, but we refer to them as word embeddings here.}}
\looseness=-1

\input{Arxiv_NewTempl/sections/figure_m404_heatmaps}

Rooted in the foundational works of \citet{shannon1948mathematical} and inspired by the ``co-occurrence statistics'' and ``distributional hypothesis'' \citep{harris1954distributional}, the NTP paradigm suggests that a word's meaning is defined by its context. This principle underlies both classical vector-space models \citep{schutze2008introduction} and neural language models \citep{neural_language_models_1,neural_language_models_2,neural_language_models_3,bengio2000taking,word2vec_1}, propelling the development of today's sophisticated large language models \citep{radford2018improving,radford2019language,brown2020language}. However, a fundamental question remains: \looseness=-1
\begin{center}
% \vspace{-5pt}
    \emph{How does the NTP learning objective shape the relationship between \\ language statistics and the geometry of model representations?}
    % \vspace{-5pt}
\end{center}
This inquiry, which we term the \emph{implicit geometry} of NTP—so named because the NTP objective does not explicitly impose any such relationship—explores how distances and angles within the neural network's $d$-dimensional representational space correlate with linguistic patterns at the end of training. %
% \vspace{-10pt}

We postulate that understanding this implicit geometry  is essential to sheding light on functional principles of large language models, since NTP is used for training across diverse architectures, from LSTMs to transformers and state-space models. Specifically, exploring how optimization under NTP shapes representations of words and contexts, which empirically mirror complex human-like patterns, not only has scientific merit but could also enhance model interpretability and explainability. This is crucial for understanding how biases in training data influence learned representations and how such biases can be mitigated. At the same time, revealing how implicit geometry correlates with language statistics could help us reverse-engineer optimization algorithms and the training loss itself to achieve more desirable structures of model representations, addressing challenges such as statistical imbalances in language data. Finally, understanding how state-of-the-art models internalize language to form representations might also enhance our grasp of language itself. \looseness=-1

Answering the question of implicit geometry is inherently challenging, as representations emerge from training complex models (in terms of architecture and size) on diverse, large-scale datasets (varying in source, size, and tokenization) with differing optimization hyperparameters (e.g., learning rate, weight decay, number of iterations). Nonetheless, we demonstrate that characterizing the implicit geometry—at least at a “macroscopic level”—is feasible in a doubly-asymptotic regime of sufficiently large models and extended training periods. \looseness=-1

Specifically, this paper develops an analytical framework to characterize the implicit geometry induced by NTP training on language datasets. Drawing inspiration from seminal studies on the geometry of deep model representations in image recognition \citep{NC}, our framework distinguishes itself from previous studies on language representations by not concentrating on specific architectures such as transformers. Instead, we assume that the model has adequate representation capacity and undergoes effective optimization, making it possible to minimize the NTP loss to its entropy lower-bound. This approach isolates the influence of NTP itself—rather than architectural nuances—in shaping the implicit geometry of the language model. 

The framework reveals the key role that sparsity patterns in language statistics  \citep{ntp} play in shaping the implicit geometry. Building on and extending recent work \citep{ntp}, we demonstrate that the recurrence of only a few tokens from the entire vocabulary as next-tokens in specific contexts leads to an implicit bias in NTP training. This bias favors a matrix of logits that develops a \emph{sparse plus low-rank} structure during training.
The sparse component of this matrix captures the probabilities of co-occurring words and contexts, while the dominant low-rank component reflects the sparsity pattern of the co-occurrence matrix. {Consequently, embeddings of contexts followed by similar sets of words—regardless of their frequencies—increasingly align during training. Simultaneously, the sparse component of the logits ensures accurate prediction of the correct frequencies for each context, which is crucial for achieving the entropy lower bound.} Importantly, both the sparse and low-rank components can be computed  purely from the input data, enabling us to predict the geometric structure of a model's representations trained on the same data a priori, without the need for training; see Fig. \ref{fig:simplified_intro} for an illustration.
Overall, our framework introduces a novel perspective, distinct from traditional analyses of word representations like those in the Word2Vec model \citep{levy}. \looseness=-1

\subsection{Methodology}\label{sec:intro_methodology}
% \vspace{-5pt}
Our methodology integrates three foundational modeling concepts as follows:

%\noindent\textbf{1. \NTP as Soft-label classification over sparse probabilistic labels} 

 Firstly, following \citep{ntp} we frame NTP as soft-label classification with CE loss applied to \emph{sparse} probabilistic label vectors. This isolates $m$ \emph{distinct}  contexts, which could be repeated multiple times throughout the training corpus, and assigns to each a sparse $V=|\Vc|$-dimensional conditional-probability label vector $\pbh_j$, reflecting the frequency of each token following context $j\in[m]$. The sparsity of $\pbh_j$ indicates that certain tokens, which we refer to as \emph{off-support} tokens for the specific $j$-th context, never follow this context.

The second concept facilitates a tractable analysis of context embeddings by assuming expressive (enough) neural networks can produce unconstrained embeddings $\hb_j\in\R^d$, independent of the architecture's specific complexities \citep{yang2017breaking,mixon2020neural}. This redefines NTP as a minimization, over word and context embedding matrices $\W$ and $\Hb$, of the NTP loss across a training set $\Tc$ determined by the \emph{sparse} matrix of next-token conditional probabilities $$\Pbb=\Pbb(\Tc)
=\left[\pbh_{1},\pbh_{2},\ldots,\pbh_{m}\right]
\in[0,1]^{V\times m}.$$
Overall, this leads to the following log-bilinear model:
 % \vspace{-2pt}
\begin{align}\label{eq:NTP intro}
\min_{\W\in\R^{V\times d},\,\Hb\in\R^{d\times m}}\, \Lc_{\rm{NTP}}\big(\W\Hb;\Pbb\big)\,.
 \vspace{-2pt}
\end{align}
This way, our goal to study word-word, context-context, and word-context geometric relationships becomes that of characterizing the Gram matrices $\GW=\W\W^\top$ and $\GH=\Hb^\top\Hb$, as well as the logit matrix $\Lb=\W\Hb$ at the minimizers of \Eqref{eq:NTP intro}. This task is complicated by the non-convex nature of the minimization and the sparsity of the probabilistic label vectors. Specifically, as we show, the sparsity of $\Pbb$ may lead to multiple minimizers, potentially making geometric characterization ambiguous.

To address this, we leverage a third concept: focusing on specific minimizers identified through the \emph{regularization path} \citep{rosset2003margin}. This entails following the solution trajectory of the empirical risk minimization in \Eqref{eq:NTP intro} as an additive ridge regularization for $\W$ and $\Hb$ diminishes to zero.  For the purpose of comparing the analysis outcomes to our numerical evaluations, we interpret the prediction obtained from the regularization path analysis as a proxy
% \footnote{For linear models, the regularization path aligns with the gradient-descent path: the converging limit of gradient descent parameters as the number of iterations $k\rightarrow\infty$ is the same in direction as the limit of the regularized solution as $\lambda\rightarrow0$ \citep{ji2020directional}. Our setting is non-convex, so we use this equivalence as a heuristic proxy; see future work discussion in Sec.~\ref{sec:conclusion}.} 
for the solution found by gradient-based optimization when the NTP loss approaches its empirical entropy lower bound.

% For linear models, the regularization path of the NTP loss aligns with the gradient-descent path, i.e. the trajectory pursued by gradient descent optimization \citep{ji2020gradient}. 

% The underlying rationale behind this is the current practice of deep-net training with very small weight decay \ct{I will continue... ok, I need to be careful. Apprently LLMs use weight decay on the order of $1e^-1$. Maybe can still justify by appealing to recent works suggesting that the effect of weight decay is training stability rather than regularization} The underlying rationale is that, in various contexts, this regularization path aligns with the trajectory pursued by gradient descent optimization \citep{ji2020gradient}, which is the practically relevant approach to optimzing the NTP objective.

% \vspace{-8pt}

\subsection{Summary of findings}\label{sec:summary}
% \vspace{-5pt}
\vspace{-5pt}
\begin{wrapfigure}{r}{0.4\textwidth}
  \centering
  \vspace{-30pt}\includegraphics[width=0.38\textwidth]{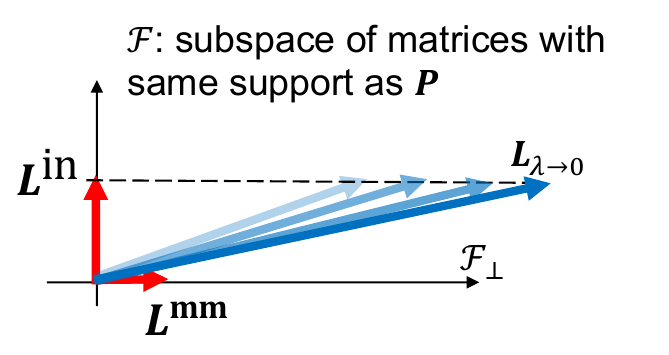}
  \captionsetup{width=0.4\textwidth}
  \vspace{-5pt}\caption{In the limit of vanishing regularization, the ridge-regularized \NTP objective solution $\Lb_{\la}$ (aka the logit matrix) decomposes into two orthogonal components. The component $\Lbin$ inherits the sparsity of  the matrix $\Pbb$ of next-token conditional probabilities. In the orthogonal complement, $\Lb_\la$ diverges in norm but converges in direction to $\Lbmm$, which results from a low-rank-promoting nuclear-norm minimization that depends only on the sparsity pattern of $\Pbb$.}
  \label{fig:Lla}
\end{wrapfigure}
Building on the above methodology, our analytical framework leads to the following results:
\noindent\textbf{Formulation in  logit space.}~Sec.~\ref{sec:setup} presents an equivalent formulation of the \NTP objective $$\Lc_{\rm{NTP}}\big(\W\Hb;\Pbb\big)+\la\|\W\|^2+\la\|\Hb\|^2$$ into the logit space, given in terms of a rank-constrained and nuclear-norm regularized minimization. From this, word and context matrices can be obtained through matrix factorization of logits $\Lb$. 

% 
% \vp
\vspace{-3pt}
\noindent\textbf{Logit Convergence.}~
Focusing on $d\geq V$\footnote{This does \emph{not} constrain $d$ relative to the number of distinct contexts $m$, which could be significantly larger than the vocabulary size $V$.  Also,  despite large embedding dimensions, our results still reveal how the geometry of learned word embeddings encodes fine-grained similarities between words as materialized in the patterns of language data; {e.g., refer to Fig.~\ref{fig:sim_W_lable_404} for a visualization.}
\looseness=-1}, in Sec.~\ref{sec:analysis}, we demonstrate that the logit matrix $\Lb_\la$ for regularization $\la\rightarrow 0$ (which is proxy for training iteration $k\rightarrow\infty$), behaves for some $R(\la)\rightarrow\infty$ as shown by the following claim (see Fig. \ref{fig:Lla} for visualization):
\vspace{-3pt}
% %, to a potentially large number of contexts $m\gg V$.
% \ct{Add a comment on this:  despite $d>V$, we will see that still model learns vector representations that encode fine-grained similarities between words as materialized in the data patterns. We should explicitly relate this to the experiments where $\Lbmm$ is low rank -- thus $W$s still live on low-dimensional subspace}

%%%%%%%%%%%%%%%%%%%%%%%%%%%%
\makeatletter
\newcommand{\itemref}[1]{%
	\item[#1]\protected@edef\@currentlabel{#1}%
}
\makeatother
%%%%%%%%%%%%%%%%%%%%%%%%%%%%%%

\begin{enumerate}

\itemref{\textbf{(C1)}}\label{P logits} \textbf{Logits' sparse plus low-rank decomposition:}  As $\laz$, for some $R(\la)\rightarrow\infty$:
\begin{align}\label{eq:Lb intro}
~~~\Lb_\la~~~\approx \underbrace{~~~\Lbin~~~}_{\text{\textbf{sparse} component:}\atop\text{sets in-support token probabilities}}~ + ~~~~~ R(\la)\cdot\underbrace{~~~\Lbmm~~~}_{\text{\textbf{low-rank} component:}\atop\text{separates in- from off-support tokens}},
\end{align}
where the orthogonal components \(\Lbin\) and \(\Lbmm\) have distinct operational roles.
\end{enumerate}
\(\Lbin\) inherits the sparsity of the data matrix \(\Pbb\) and encodes information regarding frequencies of in-support tokens. In contrast, \(\Lbmm\) is independent of these frequencies and  is  influenced only by the sparsity pattern \(\Sb \in \{0,1\}^{V \times m}\). Thus, its value is guided by both the observed `company' of each context (in-support tokens), but not their frequency, and, by the `company' it lacks (off-support tokens). {More specifically, 
$\Lbmm$ functions to maximize the logit margin between in-support and off-support tokens, where the margin is defined in terms of the nuclear norm that promotes low rank, rather than the Euclidean norm. To highlight this property, we refer to 
$\Lbmm$
as the `low-rank' component, though it could also be described as the `max-margin' component.} The decomposition in \Eqref{eq:Lb intro} also shows the logits become unbounded in norm when projected on the subspace of $\Lbmm$.\looseness=-1

% arises from solving a convex problem of low-rank minimization of logits, referred to as \ref{eq:ufm-svm}, which aims to maximize the logit margin between in-support and out-of-support tokens, while enforcing equal margins across in-support tokens.
%
% Thus, unlike $\Lbin$, the dominant component $\Lbmm$, essentially determined by the sparsity pattern of the support sets,  is guided by both the observed `company' of each context (in-support tokens) and, notably, the `company' it lacks (off-support tokens). 

%, encapsulating the essence of the distributional hypothesis within our analytical model.

% \vp
% \vspace{3pt}
\noindent\textbf{Context and word embeddings' geometry.}
Similar to logits, word/context embeddings $\W$ and $\Hb$ also grow in norm as $\laz$. This occurs in a way that simultaneously guarantees the resulting logit matrix abides by Claim \ref{P logits},
% \ct{Can someone please fix the itemization so that references here to items work properly?}
which leads to the following additional claims:\looseness=-1
\begin{enumerate}%[leftmargin=-10pt]
    % \item[\textbf{P2.}] 
    % \setlength{\itemindent}{-10pt}
    
    \itemref{\textbf{(C2)}}
    \label{P directional} \textbf{Directional convergence:} Word and context embedding matrices converge \emph{directionally}  to matrices $\Wmm:=\Ub\Sigmab^{1/2}\Rb$,  $\Hmm:=\Rb^\top\Sigmab^{1/2}\Vb^\top$, where $\Rb$ is a rotation matrix and $\Ub\Sigmab\Vb^\top$ is the singular value decomposition (SVD) of the low-rank max-margin component $\Lbmm$ in \Eqref{eq:Lb intro}. Concretely, letting $\overline{\A}:=\A/\|\A\|$ for any matrix $\A$:
    \begin{align}    \overline{\GW}\rightarrow\Ub\overline{\Sigmab}\Ub^\top\quad\text{and}\quad
        \overline{\GH}\rightarrow\Vb\overline{\Sigmab}\Vb^\top\,.
    \end{align}
    % \item[\textbf{P3.}] 
    \itemref{\textbf{(C3)}}\label{P  collapse}  \textbf{Subspace collapse:} Since $\Lbmm$  depends only on the sparsity patterns of next-token distributions (not their frequencies), a consequence of \ref{P directional} is a property termed \emph{subspace collapse}: embeddings $\hb_j$,  $\hb_{j
    '}$ of contexts $j\neq j'\in[m]$ that are followed by the same set of next-tokens (although their frequencies may differ), converge to the same limiting direction, i.e.,  for all $j,j'\in[m]$
    % \vspace{-5pt}
    \begin{align}\label{eq:}
    {\rm{support}}(\pbh_j)={\rm{support}}(\pbh_{j'}) ~\Longrightarrow\cos\big({\hb_j},{\hb_{j'}}\big)={\hb_j^\top\hb_{j'}}\big/{\big(\|\hb_j\|\,\|\hb_{j'}\|\big)}\rightarrow\, 1.
    \end{align}
    % \item[\textbf{P4.}] 
    \itemref{\textbf{(C4)}}
    \label{P soft-labels}  
    \textbf{Soft-label interpolation:} When projected on the  subspace of in-support tokens, the logits $\W\Hb$ interpolate the corresponding soft-labels ensuring that the NTP loss reaches the entropy lower bound. Concretely, for all $j\in[m]$,
    \begin{align}\label{eq:interpolation}
        (\w_z-\w_{z'})^\top\hb_j = \log\left({\ph_{j,z}}/{\ph_{j,z'}}\right),\quad \forall z,z' \in {\rm{support}}(\pbh_j)\,.
    \end{align}
    Since the subspace of in-support tokens is orthogonal to the subspace of $\Lbmm$, the subspace collapse \ref{P  collapse} does \emph{not} prevent logits $\w_z^\top\hb_j$ and $\w_z^\top\hb_{j'}$ of in-support tokens $z\in{\rm{support}}(\pbh_j)={\rm{support}}(\pbh_{j'})$ to interpolate (in the sense of \Eqref{eq:interpolation}) potentially different conditional probabilities (soft-labels) $\ph_{j,z}$ and $\ph_{j',z}$, respectively. \looseness=-1
\end{enumerate}

\begin{figure*}[t]
    \vspace{-25pt}
	\hspace{20pt} 
    \resizebox{0.85\textwidth}{!}{
    \begin{subfigure}{0.74\textwidth}
		\centering
		\begin{tikzpicture}
			\node at (-0,-1.4) 
			{\includegraphics[scale=0.20]{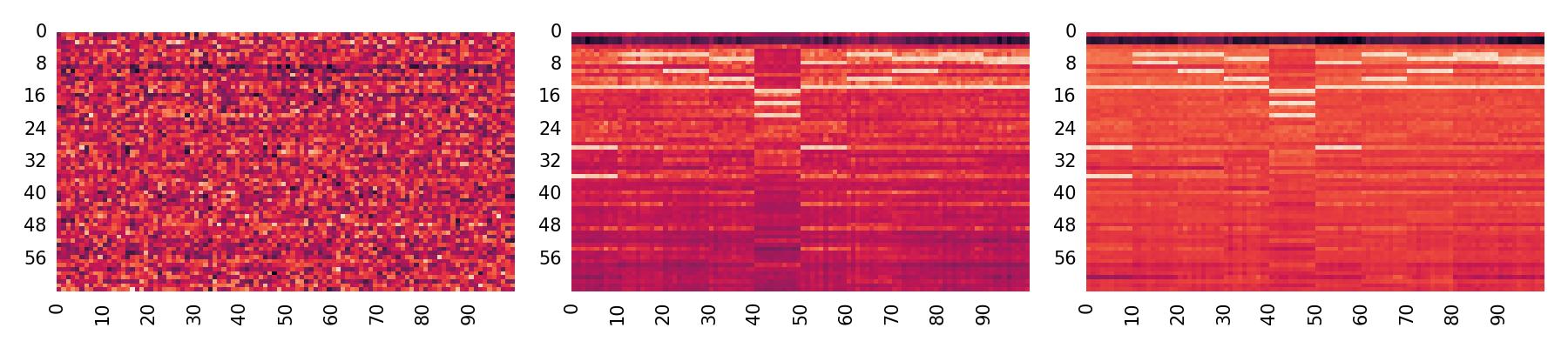}};
			\node at (-3.,-0.2) [scale=0.7]{Epoch 0};
            \node at (-0.0,-0.2) [scale=0.7]{Epoch 100};
            \node at (3.,-0.2) [scale=0.7]{Epoch 600};
			\node at (-0.0,0.3) [scale=0.9]{\textbf{(a) Training checkpoints}};
			\node at (6.2,0.3) [scale=0.9]{\textbf{(b) Proxy}};

            \node at (6.4,-0.2) [scale=0.8]{$\smatbar$};
			% \node at (6.9,0.8) [scale=0.6]{$\smat$};
			\node at (-4.7,-1.25) [scale=0.8, rotate=90]{$\Lb$};
			% \node at (-8.7,0.0) [scale=0.6, rotate=90]{};
   %          \node at (3.5,1.45) [scale=0.6]{$\corr{{\Hmm}}$};
			% \node at (6.9,1.45) [scale=0.6]{$\corr{{\smat}}$};
			% \node at (-8.65,0.0) [scale=0.6, rotate=90]{$\corr{\Hb}$};
		\end{tikzpicture}
    \end{subfigure}
    \hspace{-15pt} \begin{subfigure}{0.24\textwidth}
		\centering
		\begin{tikzpicture}
			\node at (-0,-1.6) 
			{\includegraphics[scale=0.2]{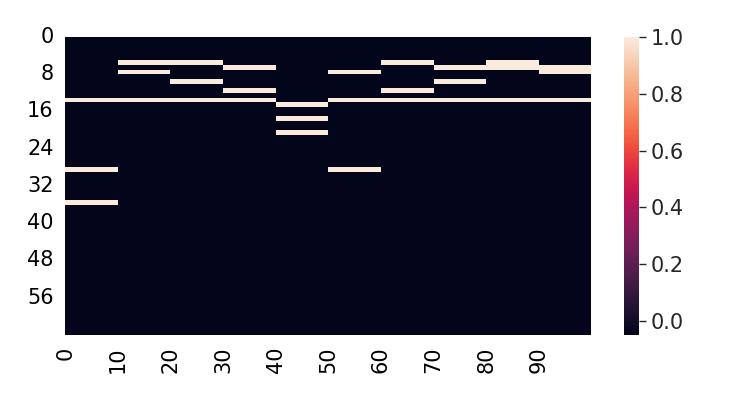}};
		\end{tikzpicture}
    \end{subfigure}
    }
    \\
    
    \hspace{20pt}
    \resizebox{0.85\textwidth}{!}{ \begin{subfigure}{0.74\textwidth}
		\centering
		\begin{tikzpicture}
			\node at (-0,-1.6) 
			{\includegraphics[scale=0.2]{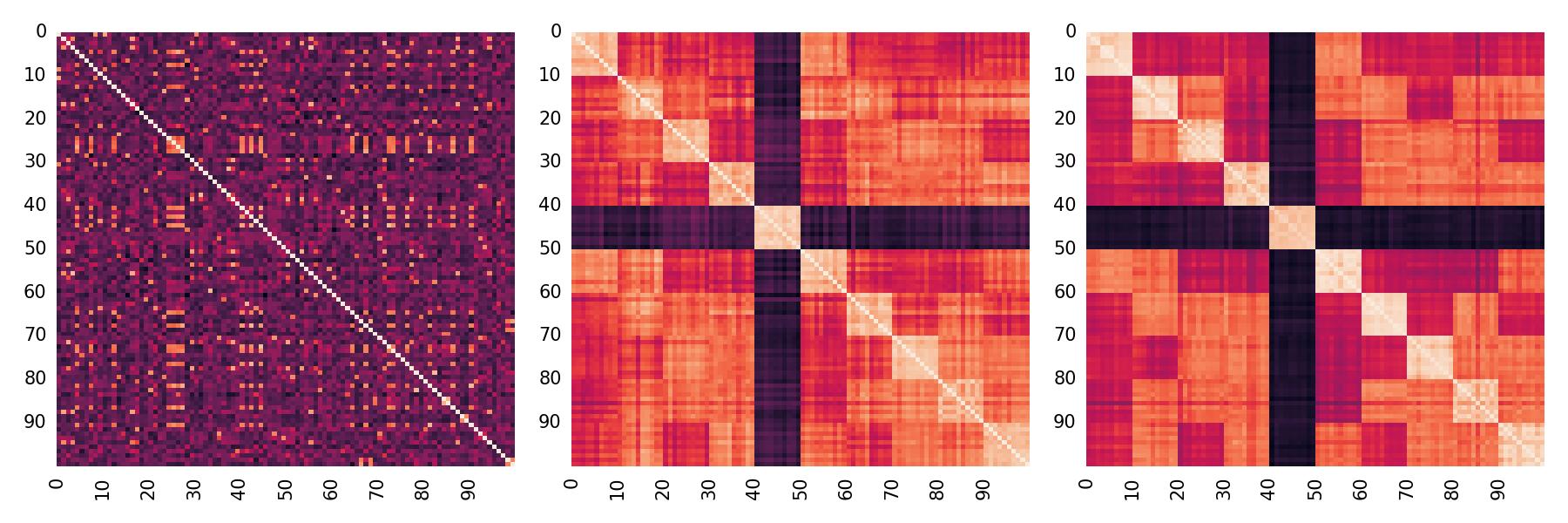}};
            \node at (6.4,0.1) [scale=0.7]{$\corr{{\smatbar}}$};
            \node at (-3.,0.1) [scale=0.7]{Epoch 0};
            \node at (-0.0,0.1) [scale=0.7]{Epoch 100};
            \node at (3.,0.1) [scale=0.7]{Epoch 600};
            \node at (-4.7,-1.5) [scale=0.8, rotate=90]{$\corr{\Hb}$};
		\end{tikzpicture}
    \end{subfigure}
    \hspace{-15pt} \begin{subfigure}{0.24\textwidth}
		\centering
		\begin{tikzpicture}
			\node at (-0,-1.4) 
			{\includegraphics[scale=0.2]{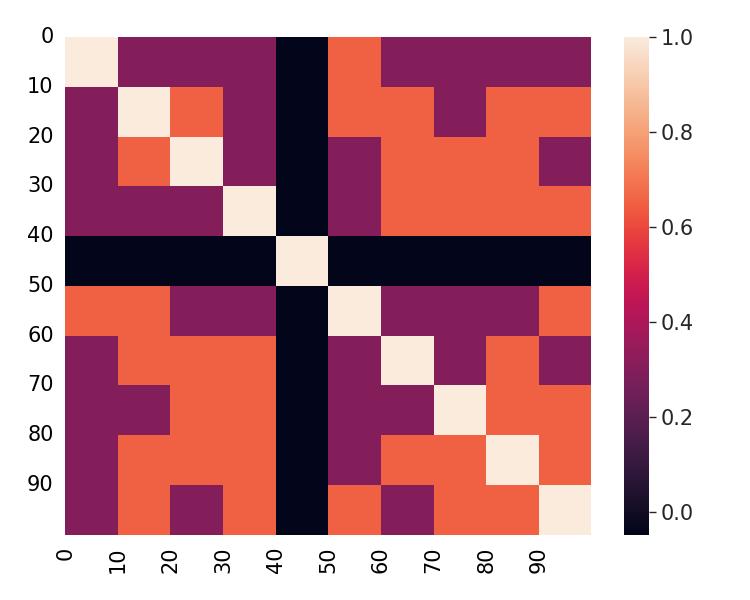}};
		\end{tikzpicture}
    \end{subfigure}
    }
    % \\
  %   \hspace{-0pt} \begin{subfigure}{0.24\textwidth}
		% \centering
		% \begin{tikzpicture}
		% 	\node at (-0,-1.4) 
		% 	{\includegraphics[scale=0.2]{New_Figures_TinyStories/Paper_LossMinusEntropy_plots_e_470_V7_ctxLen6_voc64_dim32_numStories100_Test.jpg}};
		% \end{tikzpicture}
  %   \end{subfigure}
  %   \hspace{-0pt} \begin{subfigure}{0.24\textwidth}
		% \centering
		% \begin{tikzpicture}
		% 	\node at (-0,-1.4) 
		% 	{\includegraphics[scale=0.2]{New_Figures_TinyStories/Paper_NormGrowth_plots_e_598_V7_ctxLen6_voc64_dim32_numStories100_Test.jpg}};
		% \end{tikzpicture}
  %   \end{subfigure}
  %   \hspace{-0pt} \begin{subfigure}{0.24\textwidth}
		% \centering
		% \begin{tikzpicture}
		% 	\node at (-0,-1.4) 
		% 	{\includegraphics[scale=0.2]{New_Figures_TinyStories/Paper_L_proj_to_L_in_plots_e_598_V7_ctxLen6_voc64_dim32_numStories100_Test.jpg}};
		% \end{tikzpicture}
  %   \end{subfigure}
  %   \hspace{-0pt} \begin{subfigure}{0.24\textwidth}
		% \centering
		% \begin{tikzpicture}
		% 	\node at (-0,-1.4) 
		% 	{\includegraphics[scale=0.2]{New_Figures_TinyStories/Paper_SSIM_H_W_Conject_plots_e_598_V7_ctxLen6_voc64_dim32_numStories100_Test.jpg}};
		% \end{tikzpicture}
  %   \end{subfigure}
	% %%%%%%%%%%%%%%%%%%%%%%%%%%%%%%%%%%%%%%%
    % \vspace{-10pt}
    \captionsetup{width=\textwidth}
    \caption{
    Similar to Fig.~\ref{fig:simplified_intro}, this time on a 12-layer TF trained on a subset of $100$ stories from the \TS~dataset. Here, computing the theoretical prediction $\Hmm$ is computationally expensive. Thus, we compare the embeddings geometry with the Proxy \ref{P proxy}. Details in Sec.~\ref{sec:exp_main}. \looseness=-1}
    % 
    % \tina{should we keep this in intro?} Embeddings' geometry correlates with the structural properties of text data as captured by the sparsity pattern of the support set of distinct contexts within the dataset: $12$-layer transformer (TF) trained on a subset of $1000$ stories from the Tiny Stories \citep{eldan2023tinystories}. \textbf{(Left)} Loss approaches the empirical entropy $\Hc$ of the dataset. \textbf{(Middle)} The intersection between the support sets of $100$ contexts in the dataset. Entry $(i,j)$ equals $|\Sc_j\cap\Sc_i|$, where $\Sc_i$ is the set of possible next tokens for the $i$-th context. \textbf{(Right)} Gram matrix of the corresponding $100$ context embeddings $\Hb$ learned by the TF. Contexts that share a larger common set of next tokens, are arranged closer to each other in the embedding space. Contexts that share the exact same support set, i.e., $\Sc_j=\Sc_i$, approximately align with each other (observed on the diagonal blocks). See text for details.   }
	\label{fig:intro_TinyStories} 
	% \vspace{-15pt}
\end{figure*}

% \vspace{3pt}
\noindent\textbf{Experiments.} Sec.~\ref{sec:exp_main} presents experiments on controlled settings that validate our findings. 
An example is shown in Fig.~\ref{fig:simplified_intro}: The experiment involves training a 4-layer transformer (TF) on a training set extracted and curated from the \TS~dataset \citep{eldan2023tinystories}, with the following characteristics: a vocabulary size $V=104$ and a total number of contexts $n\approx3050$, out of which $m\approx400$ are distinct. We choose TF embedding dimension $d=128>V$ and train until the empirical entropy lower bound is reached within an order of $10^{-4}$. The leftmost Panel (a) depicts the geometry of the context embeddings $\Hb$ (Top), word embeddings $\W$ (Middle), and logits $\Lb$ (Bottom) learned by the TF at the end of training.
Panel (b) compares these with the same quantities learned by the log-bilinear model in \Eqref{eq:NTP intro} evaluated on the same training dataset. The apparent resemblance of the patterns validates that the proposed log-bilinear analysis model is a good proxy for the TF model. Panel (c) compares the TF output to our analytical predictions: Comparing TF's context/word embeddings and logits to $\Hmm$/$\Wmm$ and $\Lbmm$ confirms Claims \ref{P directional} and \ref{P logits}. 
% See Sec.~\ref{sec:exp_main} for additional plots confirming that the loss reaches the empirical entropy and that logits projected on the subspace of in-support tokens interpolate the soft labels as per property \ref{P soft-labels}. 
The bright yellow regions in the Gram matrices of embeddings help visualize the subspace collapse Claim \ref{P  collapse}. See Sec.~\ref{sec:exp_main} for additional details and verification of Claim \ref{P soft-labels}. \looseness=-1

Finding the analytical prediction \ref{P directional}, illustrated in Panel (c), requires solving a semi-definite program for $\Lbmm$ and computing its SVD. In large scales, this can be computationally prohibitive, motivating the following heuristic proxy for the embedding geometry.\looseness=-1
% Thus, in Panel (d), we show a heuristic proxy for the embeddings geometry, as detailed below:
% \vspace{-3pt}
\begin{enumerate}
    \itemref{\textbf{(P)}}
    \label{P proxy} \textbf{Empirical proxy for directional convergence:} Let $\smatbar:=(\Id_V-\frac{1}{V}\ones_V\ones_V^\top)\smat$ be the (column-wise) centered support-set matrix. The directional component of
    word and context embedding matrices can be well-approximated by $\smatbar$ as follows:
    \looseness=-1
    \begin{align}    \Wmm{\Wmm}^\top\,\approx\,\smatbar\smatbar^\top\quad\text{and}\quad
        {\Hmm}^\top\Hmm\,\approx\,\smatbar^\top\smatbar\,.
    \end{align}
\end{enumerate}
Panel (d) of Fig.~\ref{fig:simplified_intro} suggests proxy \ref{P proxy} as a good enough approximator of the embeddings structure that  simply only depends on the sparsity pattern $\smat$ of the next-tokens' conditional probability matrix. %, which is computationally easier. We further comment on this proxy in App.~\ref{sec:proxy}. \looseness=-1

% Finally, we find empirically that the sparsity pattern matrix $\smat$ of the conditional probability matrix of next-tokens already provides us with a good enough proxy for the embeddings geometry at the end of training. Concretely, let $\smatbar:=(\Id_V-\frac{1}{V}\ones_V\ones_V^\top)\smat$ be the centered support set matrix. Panel (d) compares correlation matrices of context and word embeddings to the correlation matrices of $\smatbar^\top$ and $\smatbar$ respectively. \tina{removed the intuitive explanation for now.}While the analytical prediction (claim \ref{P directional}) provides a more accurate geometry of learned embeddings, the heuristic metric based on the sparsity matrix $\smat$ is simpler and computationally easier than solving the semidefinite program for $\Lbmm$ and computing its SVD. 
A second example supporting proxy \ref{P proxy}
% demonstrating the strong dependence of the learned embeddings' geometry on the sparsity pattern of $\smat$ 
is shown in Fig.~\ref{fig:intro_TinyStories}: Here, we train a larger TF on a subset of the \TS, where computing $\Lbmm$ and its SVD is expensive. Instead, we observe that the structure of the support set closely captures the structure of the learned embeddings: Along the diagonal blocks, where contexts with the exact same support sets are situated, the context embeddings align closely (subspace collapse in Claim \ref{P  collapse}). Conversely, when support sets  have zero intersection (dark entries in Panel (b)), context embeddings exhibit low correlation. This finding is consistent with classical intuitions of the ``distributional assumption'' of words/contexts and is formalized by our framework via the implicit bias of \NTP to promote low-rank logits subject to margin conditions for in/off-support tokens.  \looseness=-1

\section{Related work}
% \vspace{-0.1in}
%Given space constraints, we reserve an in-depth discussion and thorough comparison with closely related work for App. \ref{app:related}. 
We pinpoint three main related research areas;  see App.~\ref{app:related} for more in-depth discussion.

First, our research conceptually mirrors the seminal work by \cite{levy} who framed the Skip-Gram with Negative Sampling (SGNS) training objective of Word2Vec  \citep{word2vec_1,word2vec_2} as weighted matrix factorization. Specifically for large $d$, they demonstrated that SGNS implicitly factorizes the pointwise mutual information (PMI) matrix. Their analysis relies on the fact that Word2Vec architecture already is a log-biliniear model, while the log-bilinear model in \eqref{eq:NTP intro} is only used by us as an analytical proxy for more complex architectures.
% Also, their analysis essentially presumes unconstrained features. 
Also, different to them, we focus on \NTP, which employs softmax instead of sigmoids as in SGNS. (However, our analysis applies also to SGNS; see App.~\ref{sec:word2vec}). More importantly, we contribute a fresh perspective on this line of inquiry by: {(i)} confronting the sparsity of probabilistic labels head-on, identifying that it leads to diverging embeddings, a scenario where setting the loss gradient to zero, as in \cite{levy}, is infeasible; {(ii)} examining embeddings through the lens of the regularization path—a surrogate for gradient descent optimization— which unveils that embeddings emerge from the factorization of 
$\Lbin+R\Lbmm$, where 
$\Lbin$ captures frequencies akin to the PMI, while the directional component 
$\Lbmm$, becoming dominant as weights diverge, reflects the explicit sparsity patterns of the context-word co-occurrences. We envision this fresh perspective could similarly  motivate further research leveraging the geometric insights of embeddings to uncover linguistic phenomena such as the linear relationships underlying word analogies \citep{word2vec_2,glove}.

Second, our exploration of NTP's regularization path is inspired by the \emph{implicit bias/regularization} research on the preferred solutions of optimizers like GD in overparameterized systems. \citet{ji2020directional} showed that GD's  trajectory in linear one-hot encoding models aligns with the regularization path, providing a lens for examining GD dynamics.  \cite{ntp} recently extended this analysis to NTP, framing it as sparse soft-label classification. By lifting their assumption of fixed context embeddings, we delve into the more complex non-convex domain, establishing explicit connections between the implicit geometry and linguistic patterns, which we also verify empirically.
%we jointly optimize both 
%applied these concepts to NTP, albeit when training \emph{only} word embeddings while keeping contexts fixed. Instead, we jointly optimize both word and context embeddings, 

%also support  our theoretical insights with empirical validations.

Finally, our research intersects with the study of \emph{neural-collapse (NC)}, which delves into the geometry of last-layer features and weights in deep networks trained in the interpolating regime \citep{NC}, utilizing the unconstrained festures model (UFM) for analysis \citep{mixon2020neural,fang2021exploring,zhu2021geometric}. Our work is particularly aligned with \cite{seli}: We also examine the UFM's regularization path but within the \NTP framework, extending their one-hot classification findings, which can be seen as special cases of ours. Additionally, our findings have some parallels with the multilabel NC geometry explored in \citet{multilabel_nc,fisher2024pushing}. However, our 
setting is more general raising stringent assumptions on the label distribution. For detailed comparison, see App.~\ref{app:related}. \looseness=-1
\section{Formulation}\label{sec:setup}
% \vspace{-7pt}
% \new{We formalize the problem in this section. For a full description of the notations, see App.~\ref{sec:notation}}
% \vspace{-0.05in}
% \subsection{Next-token prediction}
% \vspace{-0.12in}
% \end{align}
% where the model's output probability for the next token prediction is given as  $\hat{q}_\W(\cdot|\x)=\sft{f_\W(\x)}$, where $\sft{\cdot}:\R^V\rightarrow\Delta^{V-1}$ is the softmax. 
% 

% \vspace{3pt}
\noindent\textbf{Notations.}~% \ct{consider removing if everything is kept in main body}
Throughout, lowercase and uppercase bold letters (e.g., $\ab$ and $\A$) represent vectors and matrices, respectively. 
We use $\overline{\ab}$ and $\overline{\A}$ to denote vectors/matrices normalized by their Euclidean norm. 
% We use $\vct{\overline{a}}$ and $\vct{\overline{A}}$ to denote vectors/matrices normalized by their Euclidean norm. 
We denote $\A[i,j]$ the $(i,j)$-th entry of matrix $\A$ and $\ab_j$ its $j$-th column. We let $\Rc(\Ab)$ and $\Nc(\Ab)$ its range-space (aka column space) and null-space, respectively. $\inp{\cdot}{\cdot}$ and $\norm{\cdot}$ denote Euclidean inner product and  norm, respectively. We use $\|\cdot\|_*$ to denote the nuclear-norm (i.e. sum of singular values).
% 
%We let $V$ denote the size of the vocabulary and $m$ the number of distinct contexts (see Sec. \ref{sec:ntp as soft}). 
$\Id_V$ represents the identity matrix of size $V$ and $\ones_V$, the all ones vector of size $V\times 1$ (subscripts are removed when clear from context). $\Delta^{V-1}$ denotes the $V$-dimensional unit simplex and $\sft{\cdot}:\R^V\rightarrow\Delta^{V-1}$ the softmax map:
$$
\sft{\ab} = [\sfti{1}{\ab},\ldots,\sfti{V}{\ab}]^\top,~~ \text{ with }~~ \sfti{v}{\ab}=\frac{\exp(\eb_v^\top\ab)}{\sum_{v'\in[V]}\exp({\eb_{v'}^\top\ab})}\,,
$$
where $\eb_v$ is the $v$-th standard basis vector in $\R^V$. We also denote $\ebt_{j}$ the $j$-th standard basis vector in $\R^m$.
% Also, let $\rankset_d:=\{\Lb\in\R^{V\times m}\,:\,\rank{\Lb}\leq d\}$ be the manifold of rank-constrained matrices.  
All logarithms are natural logarithms (base $e$). 

\subsection{NTP objective as sparse soft-label classification}\label{sec:ntp as soft}

We let $\Vc=[V]:=\{1,\ldots,V\}$ represent a finite vocabulary of tokens (we use the terms `word' and `token' interchangeably). We denote by $\z_{1:t}=(z_1,\ldots,z_t)$ a sequence of $t$ tokens $z_t\in\Vc$ and focus, for simplicity, on prediction of the last $T$-th token $z:=z_T$ given context $\x:=\z_{1:T-1}$. For this, we assume 
access to a training set consisting of $n$ sequences $\Tc_n:=\{(\x_i,z_i)\}_{i\in[n]}$, such that  $\x_i\in\Xc:=\Vc^{T-1}$ and $z_i\in\Vc$. This is used to train model $f_{\thetab'}:\Xc\rightarrow\Vc, f_{\thetab'}(\x)=\W \hb_{\thetab}(\x)$ parameterized by $\thetab'=\{\W,\thetab\}$, where  $\W\in\R^{V\times d}$ is a decoding matrix and  $\thetab$ parameterizes a map $\hb_{\thetab}:\Xc\rightarrow\R^d$ from contexts  to $d$-dimensional embeddings. We impose no restrictions on the specific form of the embedding map, which may, for example, be produced by an MLP, an LSTM, or a TF.
We refer to row $\w_v,v\in\Vc$ of $\W$ as \textbf{word embedding}  of token $v$ and $\hb_{\thetab}(\x)$ as \textbf{context embedding} of context $\x.$ The model is found by minimizing the empirical CE loss
% \begin{align}\label{eq:CE}
$
\CE({\thetab'})
%:=\frac{1}{n}\sum_{i\in[n]}-\log\left(\hat{q}_{\thetab'}(z_i|\x_i)\right) 
= \frac{1}{n}\sum_{i\in[n]}-\log\left(\sfti{z_i}{f_{\thetab'}(\x_i)}\right)
$.

% \vspace{-0.05in}
Following \cite{ntp}, we reframe the NTP training objective as classification over $m\leq n$ \emph{distinct} contexts, each associated with a \emph{sparse} probabilistic label vector $\pbh_j\in\Delta^{V-1}$. Concretely, 
 we denote $\xd_1,\ldots,\xd_m$ the $m\leq n$ \emph{distinct} contexts among the (large number of) total $n$ contexts $\Tc_n$. Also, we let $\pih_j=\frac{1}{n}\sum_{i\in[n]}\ind{\x_i=\xd_j}$ denote the empirical probability of distinct context $\xd_j$. Accordingly,   let $\pbh_j\in\Delta^{V-1}$ denote the probability vector of conditional next-token distribution, i.e., for all $z\in\Vc$: 
$
\ph_{j,z} := \frac{\sum_{i\in[n]: \x_i=\xd_j}\ind{z_i=z}}{\sum_{i\in[n]} \ind{x_i=\xd_j}}, \,\, j\in[m].
$
% \end{align}
In words, $n\cdot\pih_j\cdot\ph_{j,z}$ is the number of occurrences of token $z$ as a follow-up to context $\xd_j$. 
Define  the support sets of these probability vectors as
% \begin{align}
$
\Sc_{j}:=\{ z\in\Vc \,|\, \ph_{j,z}>0\}$ and let $ S_j:=|\Sc_{j}|.
$

It is also convenient to define the \textbf{conditional probability matrix} $$\Pbb=[\pbh_1,\ldots,\pbh_m]\in\R^{V\times m}\,,$$ and its corresponding \textbf{support matrix} $\smat\in\{0,1\}^{V\times m}$, such that $\smat[z,j] =1$, \emph{iff} $z\in\Sc_j$.
%,  and is $0$, otherwise.
% 0 &\text{if } z\notin\Sc_j\,
% \end{cases}\,.$
% In a natural language setting, 
In typical scenarios, $S_j<V$; thus, $\Pbb$ and $\smat$ are \emph{sparse} matrices. For given context index $j\in[m]$, we say word $z$ is \textbf{in-support token} if $\smat[z,j]=1 \text{ (eqv. } z\in\Sc_j)$; otherwise, we say $z$ is \textbf{off-support token}.
% Finally, it is convenient to define the \textbf{support matrix} $\smat\in\{0,1\}^{V\times m}$ such that each column encodes the indices of the support set of the corresponding embedding, i.e.
% \[
% \smat[z,j] = \begin{cases}
% 1 &\text{if } z\in\Sc_j\\
% 0 &\text{if } z\notin\Sc_j\,
% \end{cases}\,.
%\]
%

 With the above notation, the training loss becomes \citep{ntp},
\begin{align}     \label{eq:CE2}
    \CE({\thetab'})=-\sum\nolimits_{j\in[m]}\pih_j\sum\nolimits_{z\in\Vc}\ph_{j,z}\log\left(\sfti{z}{\W \hb_{\thetab}(\xd_j)}\right)\,.
\end{align}
This is lower bounded by the  empirical $T$-gram entropy (referred to hereafter as entropy) of the data \citep{shannon1948mathematical}, i.e., for all $\thetab'$:
$$\CE(\thetab')\geq \ent:= -\sum\nolimits_{j\in[m]}\pih_j\sum\nolimits_{z\in\Vc}\ph_{j,z}\log\left(\ph_{j,z}\right).$$
    %\label{eq:entropy}
% \end{align}
% Entropy lower bounds the CE loss since the KL divergence $\KL\left(\pbh\,||\,\hat{\qb}_\W\right)\geq0$ is nonnegative and 
% \begin{align}\label{eq:CE lb}
%     \CE(\W)=\ent+\KL\left(\pbh\,||\,\hat{\qb}_\W\right).
% \end{align}

% \vspace{-0.1in}
\subsection{Unconstrained features model for NTP training}\label{sec:ce_h}
% \vspace{-0.05in}
To gain insights into the geometry of solutions to CE minimization in \Eqref{eq:CE2}, we assume sufficient model expressivity, allowing us to optimize embeddings freely, instead of abiding by their architecture-specific parameterization. {We formalize this concept below.}
% This concept has previously been utilized in analysis of word2vec \citep{levy} \new{(see Remark \ref{rem:word2vec})} and one-hot classification tasks \citep{mixon2020neural,zhu2021geometric}.
% Here, %within the NTP framework seen as soft-class classification, 
% we extend the concept's adoption to NTP as formalized below.
\begin{definition}[NTP-UFM]
    The unconstrained features model (UFM) for NTP training over training set $\Tc_m = \{\pih_j,\pbh_{j}\}_{j\in[m]}$ refers to the following log-bilinear optimization problem:
    % \footnote{For simplicity, we apply equal regularization to $\W$ and $\Hb$. While our results can easily extend to the general scenario, such extensions do not yield additional insights.}
    \begin{align}\label{eq:ufm}
        \min\nolimits_{\W,\Hb} ~~\CE(\W\Hb)+\frac{\la}{2}\|\W\|^2 + \frac{\la}{2}\|\Hb\|^2,\tag{NTP-UFM}
    \end{align}
    where minimization is over word and context embedding matrices $\W\in\R^{V\times d}$ and $\Hb:=[\hb_1,\ldots,\hb_m]\in\R^{d\times m}$,  and 
    % with some abuse of notation, we define the CE loss (see \Eqref{eq:CE2}) under the UFM as:
    $  \CE(\W\Hb):= -\sum_{j\in[m]}\pih_j\sum_{z\in\Sc_j}\ph_{j,z}\log\left(\sfti{z}{\W \hb_j}\right)\,.
    $
\end{definition}
The key difference of \ref{eq:ufm} compared to Eq.~\eqref{eq:CE2} is that embeddings $\hb_{\thetab}(\xd_j)$ are now optimized unconstrainedly through free variables $\hb_j$.  Furthermore, we have introduced ridge regularization.\footnote{For simplicity, we apply equal regularization to $\W$ and $\Hb$. While our results can easily extend to the general scenario, such extensions do not yield additional insights.}
    Although our focus is on understanding the geometry arising from minimizing the unregularized CE objective in Eq.~\eqref{eq:CE2}, we utilize ridge regularization as a proxy to examine the behavior of gradient descent.\footnote{The equivalence between GD and regularization paths has been rigorously established in certain convex settings \citep{ji2020gradient}, but it may not hold generally. Our setting is non-convex, so we use this equivalence as a heuristic proxy; see future work discussion in Sec.~\ref{sec:conclusion}.}
% 
% \begin{remark}[Connections to classic language models ?? ]\label{rem:word2vec}
% \new{This concept has previously been utilized in analysis of one-hot classification tasks \citep{mixon2020neural,zhu2021geometric}. Here, we adopt this model for language tasks in a soft-label classification setting. The resulting model in \ref{eq:ufm} is a log-bilinear model that is reminiscent of the word2vec model \citep{word2vec_1}  architecture. While in \ref{eq:ufm} we learn an individual embedding vector $\hb$ for each context in the dataset, in word2vec, each word in the context window is individually assigned to a static embedding. Then, the static word and context embeddings are optimized in an unconstrained fashion, similar to UFM, to maximize the probability of co-occurrence of relevant target and context words. We describe the connections in more details in App.~\ref{}.} {\ct{\red{move this to separate section about word2vec objective. It is distracting here}}}
% \end{remark}
Specifically, we use the vanishing-regularization solution of \ref{eq:ufm} as a proxy for the parameters $(\W_k,\Hb_k)$ learned in large iteration $k\rightarrow\infty$ through GD  on the unregularized objective \eqref{eq:CE2}.
\begin{definition}[Regularization path] The regularization path of the \ref{eq:ufm} minimization corresponds to the solution $(\W_\la,\Hb_\la)$ of the respective log-bilinear program in the limit of vanishing regularization $\la\rightarrow0$.
\end{definition}

%%%%%%%%%%%%%%%%%%%%%%%%%%%%%%%%%%%%%
% \vspace{-0.1in}
\subsection{Reformulation in terms of logits}
% \vspace{-0.1in}
%%%%%%%%%%%%%%%%%%%%%%%%%%%%%%%%%%%%%
% In the final phase of constructing our framework for examining the geometric properties of word and context embeddings, w
We now introduce a reformulation of  \ref{eq:ufm} utilizing the logit matrix $\mathbf{L} = \mathbf{WH}$. This  is grounded on the well-known fact concerning the nuclear norm of a matrix \citep{srebro2004maximum, Fazel}:\looseness=-1
$
\|\Lb\|_* = \min_{\Lb=\W\Hb} \frac{1}{2}\|\W\|^2 + \frac{1}{2}\|\Hb\|^2 \,$.\looseness=-1
%Building on this, we derive the following equivalent formulation of \eqref{eq:ufm}.
\begin{lemma}\label{lem:relax}
Denote logit matrix $\Lb=[\ellbb_1,\ldots,\ellbb_m]\in\R^{V\times m}$ and let $\Lb_\la$ be a minimizer of
\begin{align}\label{eq:ufm relax}
        \min\nolimits_{\,\Lb \,:\, \rank{\Lb}\leq d} ~~ \Big\{-\sum\nolimits_{j\in[m]}\pih_j\sum\nolimits_{z\in\Sc_j}\ph_{j,z}\log\left(\sfti{z}{\ellbb_j}\right)+\la\|\Lb\|_*\Big\},
    \end{align}
    with SVD $\Lb_\lambda=\Ub\Sigmab\Vb^\top$, where $\Ub\in\R^{V\times r}, \Sigmab\in\R^{r\times r}, \Vb\in\R^{m\times r}$ and $r=\rank{\Lb_\lambda}\leq d$. 
    Then:
\begin{enumerate}[noitemsep,leftmargin=*]
%\vspace{-\baselineskip} % Adjusts space before the enumerate list
%\vspace{3pt}
    \item [(i)] The optimal cost of \ref{eq:ufm}   is the same as that of \Eqref{eq:ufm relax}.
    \vspace{3pt}
    \item [(ii)] $(\W_\lambda,\Hb_\lambda)$ minimizes \ref{eq:ufm} if and only if there exists a minimizer $\Lb_\lambda$ of \Eqref{eq:ufm relax} such that: \looseness=-1 
    $$\W_\la\Hb_\la^\top=\Lb_\lambda,\,~~     
    \W_\la\W_\la^\top = \Ub\Sigmab\Ub^\top,\,~~ \text{\emph{and}} \quad \Hb_\la^\top\Hb_\la = \Vb\Sigmab\Vb^\top\,.$$
\end{enumerate}    
% Concretely, (i) The optimal costs of the two optimizations are the same; (ii) $(\W_\la,\Hb_\la)$ is a minimizer of \eqref{eq:ufm} if and only if there exists minimizer $\Lb_\la$ of \eqref{eq:ufm relax} with  singular-value decomposition $\Lb_\la=\Ub\Sigmab\Vb$ where $\Ub\in\R^{V\times r}, \Sigmab\in\R^{r\times r}, \Vb\in\R^{r\times m}$ and $r=\rank{\Lb_\la}\leq d$ such that for some (partial) orthonormal matrix $\Rb\in\R^{r\times d}$ the pair  $(\W_\la,\Hb_\la)=(\Ub_\la\Sigmab^{1/2}\Rb,\Rb^\top\Sigmab^{1/2}\Vb_\la^\top)$.
% minimizer $\Lb_\lambda$ of \eqref{eq:ufm relax} corresponds to a minimizer $(\W_\lambda, \Hb_\lambda)$ of \eqref{eq:ufm}  \footnote{and vice-versa \ct{comment this is easy}}. The correspondence works as follows:

% Let singular-value decomposition $\Lb_\la=\Ub\Sigmab\Vb$ where $\Ub\in\R^{V\times r}, \Sigmab\in\R^{r\times r}, \Vb\in\R^{r\times m}$ and $r=\rank{\Lb_\la}\leq d$. Then, for any (partial) orthonormal matrix $\Rb\in\R^{r\times d}$ the pair  $(\W_\la,\Hb_\la)=(\Ub_\la\Sigmab^{1/2}\Rb,\Rb^\top\Sigmab^{1/2}\Vb_\la^\top)$ minimizes \eqref{eq:ufm}.
\end{lemma}

The constraint $\rank{\Lb}\leq d$ in \Eqref{eq:ufm relax} ensures the logit matrix can be factorized as $\W\Hb$ with inner factor dimension being $d$. Thus, in general, \Eqref{eq:ufm relax} is non-convex. \looseness=-1
% However, when $d\geq V$ (large embedding dimension)\footnote{This can be slightly improved to $d\geq V-1$. This is because strong convexity of ridge-regularization and translational invariance of softmax, imply any minimizer $\What$ of \ref{eq:ufm} is centered, i.e., $\ones_V^\top\W=0$. This then implies $\ones_V^\top\Lb=0$, which can be added in \eqref{eq:ufm relax} without changing Lemma \ref{lem:relax}.}, the constraint becomes redundant, and \Eqref{eq:ufm relax} becomes convex. \looseness=-1

% \subsection{}

% \vspace{-5pt}
\section{Analysis of unconstrained features model for NTP training}\label{sec:analysis}
% \vspace{-5pt}
% In Sec. \ref{sec:ufm}, we introduce \eqref{eq:ufm} as a model to analyze the geometry of word and context embedding matrices $\mathbf{W}$ and $\mathbf{H}$. 
This section analyzes the regularization path of \Eqref{eq:ufm relax} and, subsequently, of \ref{eq:ufm}. Throughout, we assume $d\geq V$,\footnote{This can be slightly improved to $d\geq V-1$: Strong convexity of ridge-regularization and translational invariance of softmax imply any minimizer $\What$ of \ref{eq:ufm} is centered, i.e., $\ones_V^\top\W_\la=0$. This implies $\ones_V^\top\Lb=0$ can be added in \Eqref{eq:ufm relax} without changing Lemma \ref{lem:relax}.}
in which case the constraint $\rank{\Lb}\leq d$ becomes redundant. {All proofs are deferred to the appendix and  numerical evaluations to Sec.~\ref{sec:ufmexp}.}

% \subsection{NTP as rank-constrained nuclear-norm minimization}
\vspace{-5pt}
\subsection{NTP-SVM logits}\label{sec:ntp svm}
\vspace{-3pt}
% Key role in our result plays the following nuclear norm-minimization problem.  
The following nuclear norm-minimization problem plays a key role in our results.\looseness=-1

\begin{definition} Define the \NTP-SVM logit matrix $\Lbmm$ as solution to the following optimization:
    \begin{align}\label{eq:ufm-svm}
            \Lbmm &\in \arg\min\nolimits_{\Lb\in\R^{V\times m}
            }~\|\Lb\|_* \tag{$\text{NTP-SVM}_\star$}
            \\
            &\quad~~~~\text{\emph{subj. to}}~~\Lb[z,j]-\Lb[z',j]=0,~~\forall j\in[m], z\neq z'\in\Sc_j, \label{eq:svm equalities}
            \\
            &\quad\quad\quad\quad\quad~~~~\Lb[z,j]-\Lb[v,j]\geq 1,~~\forall j\in[m], z\in\Sc_j, v\notin \Sc_j \,.\label{eq:svm inequalities}
        \end{align}    
\end{definition}

%A few remarks are in place regarding \eqref{eq:ufm-svm}. 

The optimization in \ref{eq:ufm-svm} is a semidefinite program (SDP) and always feasible, as shown by the straightforward feasibility of the centered support matrix $\smatbar=(\Id_V-\frac{1}{V}\ones_V\ones_V^\top)\smat$. Its solution $\Lbmm$ depends solely on the support matrix $\smat$ of the data. Interestingly, we prove in Sec.~\ref{sec:ntpsvm_sol_main} that $\smatbar$ solves \ref{eq:ufm-svm} under perfect symmetry assumptions on $\smat$. More generally, our empirical results suggest that $\smatbar$ closely approximates the true solution $\Lbmm$, supporting the use of Proxy \ref{P proxy} as explained in Sec.~\ref{sec:candidate_main}. 

{We remark that while the constraints are defined for all pairs of tokens, for each context $j$, the number of linearly independent constraints imposed by Eqs.~\eqref{eq:svm equalities} and \eqref{eq:svm inequalities} is $V-1$. This reduction is possible because the constraints can be reformulated relative to an anchor $z_j\in\Sc_j$ given the equality constraints \eqref{eq:svm equalities}.}

%We construct a simple feasible solution in App.~\ref{sec:ntpsvm_sol}. %Moreover, we will soon demonstrate that in this case, the minimizer can be obtained analytically. 

% Secondly, we comment on the two types of affine constraints. The constraints in \Eqref{eq:svm equalities} require in-support logits to be equal to each other for every context. For fixed $j\in[m]$,  the set of equality constraints for all $z\neq z'\in\Sc_j$  is equivalent to the set of the same constraints for any anchor $z_j\in\Sc_j$ and $z'\neq z_j\in\Sc_j$. That is, there is an effective total of $S_j-1$ linearly independent constraints for each $j\in[m]$. Thus,  defining $\Ebin{j}\in\R^{(S_j-1)\times V}$ with $S_j-1$ independent rows $(\eb_{z_j}-\eb_z)^\top, z\neq z_j, z\in\Sc_j$, then
% \Eqref{eq:svm equalities} reduces to $\Ebin{j}\ellbb_j = \textbf{0}, j\in[m]$.
% % \begin{align}\label{eq:ntp comp lin sys}
% % \end{align}
% The inequality constraints in \Eqref{eq:svm inequalities}, impose the requirement that in-support logits must be greater than out-of-support logits. Considering \Eqref{eq:svm equalities}, we can again select any anchor $z_j \in \Sc_j$ and reformulate \Eqref{eq:svm inequalities} as $\Ebout{j}\ellbb_j \geq \textbf{1}, \, j\in[m]$.
% % \begin{align}\label{eq:ntp comp lin sys ineq}
% % \Ebout{j}\ellbb_j \geq \textbf{1}, \quad j\in[m]\,,
% % \end{align}
% Here, $\Ebout{j} \in \mathbb{R}^{(V-S_j) \times V}$ has $V-S_j$ independent rows $(\eb_{\z_j}-\eb_v)^\top, v \notin \Sc_j$. In total, there are  $V-1$ constraints per context in \ref{eq:ufm-svm}.\looseness=-1
%
Finally, when $S_j=1, \forall j\in [m]$, then \Eqref{eq:svm equalities} vanishes and \Eqref{eq:svm inequalities} lower bounds the margin between the correct token and the rest. This constraint appears in the classical hard-margin SVM, which explains our naming. Instead, the NTP setting results in additional equality constraints and minimizes nuclear-norm rather than Euclidean norm.

%This also gives a justification for the heuristic Proxy \ref{P proxy} as detailed in the same section.
\looseness=-1
%\red{For an intuitive explanation behind the nature of the constraints see Appendix \ref{app:intuit svm}.}

\subsection{Regularization-path analysis of logits}

We now characterize the solutions of \eqref{eq:ufm relax} when regularization vanishes and $d\geq V$.\looseness=-1
%Having defined the \ref{eq:ufm-svm} logits, we are now ready to characterize the regularization path of \eqref{eq:ufm relax}.

\begin{theorem}\label{thm:reg path}
    Suppose $d\geq V$. In the limit of vanishing regularization, the solution $\Lb_\la$ of optimization \eqref{eq:ufm relax} diverges in norm and converges in direction to $\Lbmm$. Formally, $\lim_{\la\rightarrow0} \|\Lb_\la\| = +\infty$, and,  there exists minimizer $\Lbmm$ of \ref{eq:ufm-svm} such that 
    $$
        \lim_{\laz}\,\,\big\|{\frac{\Lb_\la}{\|\Lb_\la\|_*} - \frac{\Lbmm}{\norm{\Lbmm}_*}}\big\|=0\,.
    $$
    % \begin{enumerate}
    %     % \item Firstly,
    %     %  \[
    %     % \lim_{\la\rightarrow0} \Qc_\Omega(\hat\Lb_\la) = \Lbstar
    %     % \]   
    %     \item The loss converges to the entropy lower bound. That is, $\lim_{\la\rightarrow0} \CE(\hat\Lb_\la) = \Hc$.
    %     \item The solution diverges in norm. That is, $\lim_{\la\rightarrow0} \|\hat\Lb_\la\| = +\infty$.
    %     \item The solution converges in direction to the optimal set of \eqref{eq:ufm-svm}. That is, there exists minimizer $\Lbmm$ of \eqref{eq:ufm-svm} such that for the normalized  it holds that
    %     \[
    %     \lim_{\la\rightarrow0}\,\,\Big\|{\frac{\Lbo_\la}{\|\Lbo_\la\|_*} - \frac{\Lbmm}{\norm{\Lbmm}_*}}\Big\|=0.
    %     \]
    %     \end{enumerate}    
\end{theorem}
Thus, with vanishing regularization, \ref{eq:ufm-svm} captures the structure of the logits. By  Lemma \ref{lem:relax}, it also captures the structure of $\W$ and $\Hb$. We explore this further in Sec.~\ref{sec:embeddings geo from logits}. 

Now, recall that \ref{eq:ufm-svm} only depends on the support matrix of the training data, raising the questions: How do the probabilities $\ph_{j,z}$ influence the learned logit matrix and what implications this has for the loss? The theorem below complements Thm.~\ref{thm:reg path} and addresses these questions. To state the result, define the following matrix subspace, as the span of rank-one matrices each being zero except entries $(z,j)$ and $(z',j)$, which equal $1$ and $-1$ respectively for all context indices $j\in[m]$ and in-support word indices $z,z'\in\Sc_j$:
\begin{align}
    \Ts&=\Ts(\smat)
    =\operatorname{span}\big(\big\{\, (\eb_z-\eb_{z'})\,\ebt_{j}^\top
    \,:\,z\neq z'\in\Sc_j, j\in[m]\,
\big\}\big)\nn\subset\R^{V\times m}%\label{eq:subspace}
    \,.
\end{align}
%recall the matrices $\Ebin{j}, j\in[m]$ from Sec.~\ref{sec:ntp svm} and define the subspace of matrices with their $j$-th column on the range space $\Rc(\Ebin{j}^\top)$:
% For the second equality, recall that $\eb_z$ (resp. $\ebt_j$) is the $z$-th (resp. $j$-th) standard basis vector in $\R^V$ (resp. $\R^m$).
Note for  any $\Lb$,  the projection $\Qcs(\Lb)$ onto $\Ts$ is sparse  with  support matching $\smat$.
%Letting $\Qcs$ be the projection to $\Ts$, note that for any $\Lb$,  $\Qcs(\Lb)$ is sparse  with same support as $\smat$.

\begin{theorem}\label{thm:reg path finite}
    Under the setting of Theorem \ref{thm:reg path}, the loss approaches its lower bound $\lim_{\laz}\CE(\Lb_\la)=\ent$. Additionally, $\lim_{\laz}\Qcs(\Lb_\la)=\Lbin$, where $\Lbin\in\Ts$ is the unique $\Lb\in\Ts$ that guarantees differences of logits equal log-odds, i.e., 
    $$\Lb[z,j]-\Lb[z',j]=\log\left(\frac{\ph_{j,z}}{\ph_{j,z'}}\right),\text{ } \forall z\neq z'\in\Sc_j, j\in[m]\,.
    $$
\end{theorem}
The loss reaches the lower-bound $\ent$ because on the subspace $\Ts$, $\Lb_\la$ satisfies the log-odds linear equations of the theorem. This gives rise to Claim \ref{P soft-labels}: the model outputs the correct probabilities for in-support tokens. 
% These equations are equivalent to the condition $\frac{\sfti{z}{\W\hb_j}}{\sfti{z'}{\W\hb_j}}=\frac{\ph_{j,z}}{\ph_{j,z'}}$, i.e., the model outputs the correct probabilities for in-support tokens. This statement corresponds to Claim \ref{P soft-labels}.
There is no contradiction of this claim to Thm.~\ref{thm:reg path} because $\Lbmm$ belongs to the orthogonal complement $\Fc_\perp$ (which follows from \Eqref{eq:svm equalities}; see App. \ref{app:notations recap}). 

Combining the two theorems, $\Lb_\la$ converges to $\Lbin$ in $\Ts$ and diverges in $\Ts_\perp$, where it converges directionally to $\Lbmm$. This justifies Claim \ref{P logits}. %that $\Lb_\la\approx\Lbin+R(\la)\cdot\Lbmm$ where $R(\la)\rightarrow\infty$ as $\laz$. \looseness=-1
% \vspace{-0.05in}
\subsection{Geometry of embeddings}\label{sec:embeddings geo from logits}
% \vspace{-0.05in}
In view of Lemma \ref{lem:relax} we can obtain  the embeddings by factorizing the logit matrix. Since logits diverge, the same is true for the embeddings, which also converge directionally. %Formally, we have the following immediate corollary of Thm.~\ref{thm:reg path} and Lemma \ref{lem:relax}, which shows that the matrices $(\Ub\Sigmab\Ub^\top,\Ub\Sigmab\Vb^\top,\Vb\Sigmab\Vb^\top)$ define  the geometry of word and context embeddings, subject to a scaling factor. 
\begin{corollary}\label{cor:WH}
    Denote $(\W_\la,\Hb_\la)$ any minimizer of \ref{eq:ufm}.  Consider the SVD $\Lbmm=\Ub\Sigmab\Vb^\top$. Then, {using $\overline{\,\cdot\,}$ notation to denote normalized quantities}, in the limit of vanishing regularization
    $$\overline{\W_\la\W_\la^\top}\rightarrow \Ub\overline{\Sigmab}\Ub^\top=:\GWmmo,\,~~\overline{\Hb_\la^\top\Hb_\la}\rightarrow \Vb\overline{\Sigmab}\Vb^\top=:\GHmmo,\,~~ \text{\emph{and}} \quad  \overline{\W_\la\Hb_\la}\rightarrow\Ub\overline{\Sigmab}\Vb^\top\,.$$
\end{corollary}
% 
% \vspace{5pt}
This corollary supports Claim \ref{P directional}. Another direct consequence of it is Claim \ref{P  collapse}: embeddings of contexts whose support sets are identical asymptotically collapse to the same embedding as $\laz$. 
% We formalize this in Proposition \ref{propo:collapse} in App.~\ref{sec:collapse_app}.\looseness=-1 
% The proposition supports theoretically the empirical observations in Figs. \ref{fig:intro_TinyStories} and \ref{fig:larger_TinyStories} \tina{}. 

\begin{proposition}[Subspace collapse]\label{propo:collapse}
 Assume \ref{eq:ufm-svm} is feasible. There exists minimizer $\Hb=[\hb_1,\ldots,\hb_m]$  of \ref{eq:ufm} such that for all contexts $j,j'\in[m]$ with same support set $\Sc_j=\Sc_{j'}$, their embeddings are same in direction in the limit of vanishing regularization, i.e., $\lim_{\laz}\left\|\overline{\hb_j}-\overline{\hb_{j'}}\right\|=0$. 
\end{proposition}

This is consistent with our empirical observations in Figs.~\ref{fig:simplified_intro} and \ref{fig:intro_TinyStories}: As the loss approaches its lower bound, context embeddings with the same support set converge in the same direction and exhibit maximum correlation. This is indicated by the bright diagonal blocks in $\corr{\Hb}$ heatmaps.

% As specified by Thm.~\ref{thm:reg path finite}, the on-support entries of the logit $\Lb_{\la\to0}$ satisfy the log-odds constraints on a subspace orthogonal to the subspace where the limiting direction of $\Lb_\laz$ resides. See Secs.~\ref{sec:ufmexp} and \ref{sec:exp_main} for more numerical results.

% \ct{remove this paragraph? Anyways we push all proofs to app}\new{The proposition can be proved by inspecting \ref{eq:ufm-svm}: Since the affine constraints \eqref{eq:svm equalities} for each sample $j\in[m]$ depend solely on the support set $\Sc_j$, it can be shown that there exists a minimizer $\Lbmm=[\ellbb_1^\text{mm},\cdots,\ellbb_m^\text{mm}]$ such that $\ellbb_j^\text{mm}=\ellbb_{j'}^\text{mm}$ whenever $\Sc_j=\Sc_{j'}$. This implies that the corresponding columns $j,j'$ of $\Vb^\top$ are identical. As a result, by Cor.~\ref{cor:WH}, $\lim_{\laz}\left\|\overline{\hb_j}-\overline{\hb_{j'}}\right\|=0$. \looseness=-1
% }

It is important to reiterate that Prop.~\ref{propo:collapse} does not prevent the minimizer $\Lb_\la=\W_\la\Hb_\la$ from accurately recovering the distinct soft-label values $\pbh_j,\pbh_{j'}$ and achieving the empirical entropy lower-bound $\Hc$ as $\laz$: The collapse occurs on subspace $\Ts_\perp$, orthogonal to the data subspace $\Ts$ on which the logit matrix satisfies the log-odds constraints indicated in Thm.~\ref{thm:reg path finite}; see Secs.~\ref{sec:ufmexp} and \ref{sec:exp_main} for numerical verifications. \looseness=-1

\subsection{On the solution of \texorpdfstring{\ref{eq:ufm-svm}}{NTP-SVM*}}\label{sec:ntpsvm_sol_main}
% \tina{move here?}
% \subsection{On the solution of \texorpdfstring{\ref{eq:ufm-svm}}{NTP-SVM*}}
% \ct{change framing and call it something like: on the structure of $\Lbmm$}
% \vspace{-0.05in}
% 
To better understand how the SVD factorization informs the geometry, we investigate further how the structure of $\Lbmm$ (solution to \ref{eq:ufm-svm}) depends on the support matrix $\Sb$.
% \vspace{-0.1in}

{We start in Sec. \ref{sec:special_case} with an idealized, perfectly symmetric setting where each of the 
$V \choose k$ contexts is followed by a unique set of 
$k$ words. While this configuration is admittedly too symmetric to closely mimic natural language, it fulfills two functions. First, it allows for explicit calculation of the implicit geometry—specifically, the angles and norms of embeddings. Second, and more importantly for our analysis, it serves as a basis in Sec. \ref{sec:candidate_main} for approximating $\Lbmm$ in more general scenarios, thereby motivating the heuristic proxy \ref{P proxy} in Sec. \ref{sec:proxy_main}.} 

\subsubsection{Special case: support sets of equal sizes}\label{sec:special_case}
% \vspace{-0.05in}

\begin{proposition}\label{prop:all symmetric geometry}
    Fix $k\in[V-1]$ and suppose $\smat$ contains all $m={V\choose k}$ support sets of size $k$.  Then,
     \begin{enumerate}
     [leftmargin=*,noitemsep]
       % \setlength{\itemindent}{-2em}
       % \vspace{-0.15in}
        \item[(i)] The logit matrix takes the form $\Lbmm=(\Id_V-\frac{1}{V}\ones\ones^\top)\Sb \triangleq \smatbar$.
        \item[(ii)] Word embeddings form equiangular tight frame (ETF) being equinorm and maximally separated. 
        \item[(iii)] Context embeddings are equinorm and the embedding $\hb_j$ of context $j$ is co-linear to $\sum_{z\in\Sc_j}\w_z$.
    \end{enumerate}
\end{proposition}

In the symmetric setting of the proposition, we can analytically solve \ref{eq:ufm-svm} giving $\Lbmm$ as expressed in the statement (i). In fact, it is possible to obtain an explicit characterization of the SVD of $\Lbmm$, which enables the precise definition of embeddings' geometries in statements (ii) and (iii); see also Prop.~\ref{prop:sym geo app} in the appendix. 
% 
% \begin{proposition}\label{prop:sym geo app}
%        Assume the setting of Proposition \ref{prop:all symmetric geometry}. The geometry of context and word embeddings  is described by the following relations:
%     \begin{subequations}
%     \begin{align}
%         \forall v\neq v'\in[V]\,&:\,\cosa{\w_v}{\w_{v'}} = \frac{-1}{V-1} \quad\text{and}\quad \|\w_v\|=\|\w_{v'}\|
%         \\
%         \forall j\neq j'\in[m]\,&:\,\cosa{\hb_j}{\hb_{j'}}=\frac{|\Sc_j\cap\Sc_{j'}|-\frac{k^2}{V}}{k-\frac{k^2}{V}}\quad\text{and}\quad \|\hb_j\|=\|\hb_{j'}\|
%         \\
%         \forall j\in[m], v\in[V]\,&:\
%         \cosa{\w_v}{\hb_j} = \begin{cases}
%     \sqrt{\frac{V-1}{k(V-k)}} & v\in\Sc_j\\
%      \frac{-1}{\sqrt{k(V-k)(V-1)}} & v\notin\Sc_j
%     \end{cases}\,
%     \\
%     \forall j\in[m], v\in[V]\,&:\ \frac{\|\w_v\|^2}{\|\hb_j\|^2} = \frac{(V-1)\,\binom{V-2}{k-1}}{k(V-k)}\,.
%     \end{align}
%     \end{subequations}
% \end{proposition}
% 
% 
% \begin{figure}[t]
%     \centering 
%     \includegraphics[width=0.7\linewidth]{figures/symmetric_geo_w_h.eps}  
% \captionsetup{width=\textwidth}
%     \caption{Geometry properties illustration for Proposition \ref{prop:all symmetric geometry}. Shown two values of $V$ for varying values of $k\in[V-1]$. (Left) angles between context and word embeddings. (Right) normalized norm ratio of word to context embeddings.}
%     \label{fig:symmetric_geo}
% \end{figure}
% 
% 
For $k=1$ the embeddings recover the ETF geometry that has been previously shown for the setting of one-hot classification \citep{NC}. For general $k>1$, word embeddings continue forming an ETF, but the geometry of context embeddings changes although they remain equinorm. Embedding $\hb_j$ forms the same \emph{acute} angle with all in-support word vectors $\w_v, v\in\Sc_j$ and  the same \emph{obtuse} angle with all out-of-support word vectors $\w_v, v\notin\Sc_j$.  Additionally, for $k>1$, norms of word embedddings become larger than norms of context embeddings. See also Fig.~\ref{fig:symmetric_geo} for a visualization of these properties. \looseness=-1

% \vspace{-0.05in}
\subsubsection{A simple candidate solution}\label{sec:candidate_main}
 Achieving an analytic expression for $\Lbmm$, as was possible in Prop.~\ref{prop:all symmetric geometry}, might not always be possible. This challenge primarily arises from the combinatorial complexity of the constraints that adhere to the sparsity pattern of $\Sb$. To mitigate this issue and potentially circumvent the need to solve the semidefinite program in \ref{eq:ufm-svm}—which becomes computationally intensive for large $V$ and $m$—we introduce a strategy that allows for a numerical verification of whether $\smatbar$ of Prop. \ref{prop:all symmetric geometry} solves \ref{eq:ufm-svm}.
% The matrix we propose, inspired by Prop.~\ref{prop:all symmetric geometry}, is defined as follows: %$\Lbtina[z,j]=1-|\Sc_j|/V$ if $z\in\Sc_j$, and $\Lbtina[z,j]=-|\Sc_j|/V$ if $z\notin\Sc_j$.
% % $
% %         \Lbmm = \big(\Id_V-\frac{1}{V}\ones\ones^\top\big)\smat\,.
% %     $
% % Consequently, the logits are determined by 
% \begin{align}\label{eq:Lmm explicit}
%     \Lbtina[z,j] = \begin{cases}
%         1-|\Sc_j|/V, & z\in\Sc_j\,\\
%             -{|\Sc_j|}/{V}, & z\not\in\Sc_j\,
%         \end{cases}
%         \,.
%     \end{align}
%
Note that $\Lbtina$ readily satisfies the feasibility conditions for \ref{eq:ufm-svm}, meeting all inequality constraints with equality. \looseness=-1%\tina{-}Our empirical evaluations indicate that it successfully solves \ref{eq:ufm-svm} in numerous  realizations of $\Sb$ we have examined. 
 Below, we further introduce a dual-certificate condition that, when fulfilled by $\Lbtina$, ensures optimality. \looseness=-1 %The details of this condition are articulated in Proposition \ref{propo:svm sufficient condition}.
\begin{proposition}\label{propo:svm sufficient condition}
    Let $\Lbtina=\Ub\Sigmab\Vb^\top$ denote the SVD of $\Lbtina=(\Id_V-\frac{1}{V}\ones\ones^\top)\smat$. Define, $\Ab:=\Ub\Vb^\top$. If $\forall j \in [m]$ and all $v\not\in\Sc_j$, it holds that  $\Ab[v,j]< 0$, then $\Lbtina$ solves \ref{eq:ufm-svm}.
\end{proposition}
%The proposed certificate is applicable to any potential solution matrix. However, we particularly recommend the specific form in \eqref{eq:Lmm explicit} as a straightforward and frequently effective option. Additionally, 
Verifying the certificate's conditions only requires an SVD of $\Lbtina$. In situations where these SVD factors can be analytically determined, the certificate also enables a formal proof of optimality (e.g., Prop.~\ref{prop:all symmetric geometry}).

In more complex settings than that of Prop.~\ref{prop:all symmetric geometry}, whether $\Lbtina$ satisfies the dual certificate in Prop.~\ref{propo:svm sufficient condition} or not, and thus its optimality, depends on the sparsity pattern of $\Sb$. App.~\ref{sec:candidate_app} compares $\Lbtina$ to the optimal solution on a few small-scale examples. Although not necessarily optimal for all configurations of the support set, we conjecture that $\Lbtina$, gives an approximation with almost similar geometric properties of the optimal solution $\Lbmm$. This conjecture motivates a more efficient proxy for predicting the implicit geometry in the next section.
\subsubsection{A proxy for \texorpdfstring{\ref{eq:ufm-svm}}{NTP-SVM*} solutions}\label{sec:proxy_main}
% \input{Arxiv/sections/Lmm_conj_main}
% \new{In Sec\ref{}, we used $\smat^\top\smat$ and $\smat\smat^\top$ as proxies for the geometric patterns of the optimal context and word embeddings, respectively. We found a high correlation between the Gram matrix of $\smat^\top$ (resp. $\smat$) with that of $\W$ (resp. $\Hb$) empirically in the synthetic setup \ref{}. We then used these heuristic metrics as a reference for the optimal word and context embeddings in our larger-scale experiments of \ref{??}. We relied on this approach because, in our large-scale experiment, calculating the optimal $\Lbmm$ is not feasible, and even if it were, performing the SVD on the optimal logit for recovering the optimal $\W$ and $\Hb$ would be computationally expensive. In contrast, the heuristic metrics $\smat\smat^\top$ and $\smat^\top\smat$ are easier to compute and track. }
% 
{Considering $\Lbtina=(\Id_V-\frac{1}{V}\ones\ones^\top)\Sb$, i.e., the column-wise centered support set, as a close proxy of the optimal solution of \ref{eq:ufm-svm}, we provide a connection between the support sets and embeddings geometry. Suppose $\Lbmm\approx\Lbtina$. Then, by Thm.~\ref{thm:reg path},}
% 
% \citet{seli} verifies Claim \ref{claim:conj} in the special case of one-hot case with STEP-imbalanced training distribution. In other words, in this special case, $\Lbmm = (\Id_V - \ones_V\ones_V^\top)\smat$. Thus, 
% \begin{align}\label{eq:proxy1}
%      \Wmm{\Wmm}^\top &= (\Lbmm{\Lbmm}^\top)^{\frac{1}{2}} \approx \Big((\Id_V-\frac1V\ones_V\ones_V^\top)\smat\smat^\top(\Id_V-\frac1V\ones_V\ones_V^\top)\Big)^{\frac{1}{2}},\nn\\
%     {\Hmm}^\top\Hmm &= ({\Lbmm}^\top\Lbmm)^{\frac{1}{2}} \approx \Big(\smat^\top(\Id_V-\frac1V\ones_V\ones_V^\top)\smat\Big)^{\frac{1}{2}}.
% \end{align}
\begin{align}\label{eq:proxy1}
     \Wmm{\Wmm}^\top &= (\Lbmm{\Lbmm}^\top)^{\frac{1}{2}} \approx (\smatbar{\smatbar}^\top)^{\frac{1}{2}},\nn\\
    {\Hmm}^\top\Hmm &= ({\Lbmm}^\top\Lbmm)^{\frac{1}{2}} \approx ({\smatbar}^\top\smatbar)^{\frac{1}{2}}.
\end{align}
% 
% \tina{what is the interpretation of $\Wmm{\Wmm}^\top\Wmm{\Wmm}^\top$?}
% In other words,
% \begin{align*}
%      \Wmm{\Wmm}^\top\Wmm{\Wmm}^\top &=  (\Id_V-\ones_V\ones_V^\top)\smat\smat^\top(\Id_V-\ones_V\ones_V^\top)\\
%     {\Hmm}^\top\Hmm{\Hmm}^\top\Hmm &= \smat^\top(\Id_V-\ones_V\ones_V^\top)\smat
% \end{align*}
% 
% In other words, the Gram matrix of $\Wmm$ (resp. $\Hmm$) is close to the Gram matrix of $\smat$ projected on the subspace orthogonal to $\ones_V$. 
% This suggests a heuristic justification for why we observe the high correlations between the embeddings geometry and the sparsity pattern of language in our experiments in Figs.\ref{}.
Computing the matrix square roots in \Eqref{eq:proxy1} is as expensive as an SVD decomposition which can be prohibitive for large matrices. Instead, we suggest the following more computationally efficient proxies for estimating the directional component of the implicit geometries:\looseness=-1
\begin{align}\label{eq:proxy2}
   \Wmm{\Wmm}^\top&\approx\smatbar\smatbar^\top=(\Id_V-\frac1V\ones_V\ones_V^\top)\smat\smat^\top(\Id_V-\frac1V\ones_V\ones_V^\top),\nn\\
   %\quad \text{and}\quad
     {\Hmm}^\top\Hmm&\approx\smatbar^\top\smatbar = \smat^\top(\Id_V-\frac1V\ones_V\ones_V^\top)\smat,
\end{align}
% 
% where we define $\smatbar=(\Id_V-\frac1V\ones_V\ones_V^\top)\smat$ to be the (column-wise) centered support set.
This leads to Proxy \ref{P proxy} introduced in Sec.~\ref{sec:summary}.

Ignoring for simplicity the projection to the $\Id_V-\frac1V\ones_V\ones_V^\top$ subspace, \Eqref{eq:proxy2} gives a simple explanation for similarity between embeddings: For given context $j\in[m]$, the corresponding embedding has higher correlation with contexts $j'$ whose support set intersection $|\Sc_j\cap \Sc_{j'}|$ is larger. Similarly, word embeddings are closer if the respective words appear together in more support sets: word embedding $\w_z$ is more similar to $\w_{z'}$ when $|\{j\,:\, \w_z, \w_{z'}\in\Sc_j,\,j\in[m]\}|$ is larger. The projection then fixes the bias of the parameters: since ridge-regularization is strongly convex and softmax is invariant to shift, any minimizer of \ref{eq:ufm} satisfies $\ones_V^\top\W=0$ ( see Lemma \ref{lem:W centered}).

\subsection{Numerical simulation: Verifying the theory}

\begin{figure*}[t]
	\vspace{-25pt}
	\centering
	\hspace{-70pt} 
    \begin{subfigure}{0.9\textwidth}
		\centering
		\begin{tikzpicture}[remember picture]
			\node at (-0,-0) 
			% {\includegraphics[scale=0.13]{figures/UFM_Lheatmap.png}};
            % {\includegraphics[scale=0.152]{figures/UFM_Lheatmap_centered.png}};
            {\includegraphics[scale=0.13]{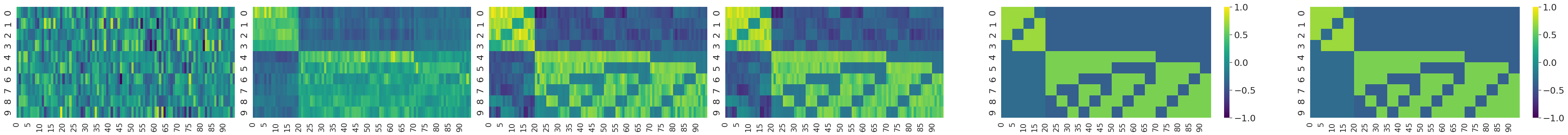}};
			\node at (-6.1,0.7) [scale=0.6]{Epoch 0};
			\node at (-3.9,0.7) [scale=0.6]{Epoch 100};
            \node at (-1.6,0.7) [scale=0.6]{Epoch 1000};
            \node at (0.5,0.7) [scale=0.6]{Epoch 3000};
			\node at (-3.3,1.2) [scale=0.8]{\textbf{(a) Training checkpoints}};
			\node at (3.,1.2) [scale=0.8]{\textbf{(b) Theory}};
			\node at (5.9,1.2) [scale=0.8]{\textbf{(c) Proxy}};

			\node at (3.,0.7) [scale=0.6]{$\Lbmm$};
			\node at (5.9,0.7) [scale=0.6]{$\smatbar$};
			% \node at (-9.5,0.0) [scale=0.6, rotate=90]{$\Lb$};
			\node at (-7.4,0.0) [scale=0.6, rotate=90]{$\Lb_k$};
   % \draw[black, dashed, line width=1pt] (1.5,1.2) -- (1.5,0);
		\end{tikzpicture}
    \end{subfigure}\vspace{-10pt}
	%%%%%%%%%%%%%%%%%%%%%%%%%%%%%%%%%%%%%%%
 
	\vspace{6pt}
	\centering
	\hspace{5pt} 
	\hspace{-80pt} 
    \begin{subfigure}{0.9\textwidth}
		\centering
		\begin{tikzpicture}[remember picture]
			\node at (-0,-0) 
			% {\includegraphics[scale=0.16]{figures/UFM_Wheatmap.png}};
			% {\includegraphics[scale=0.16]{figures/UFM_Wheatmap_norm.png}};
            {\includegraphics[scale=0.13]{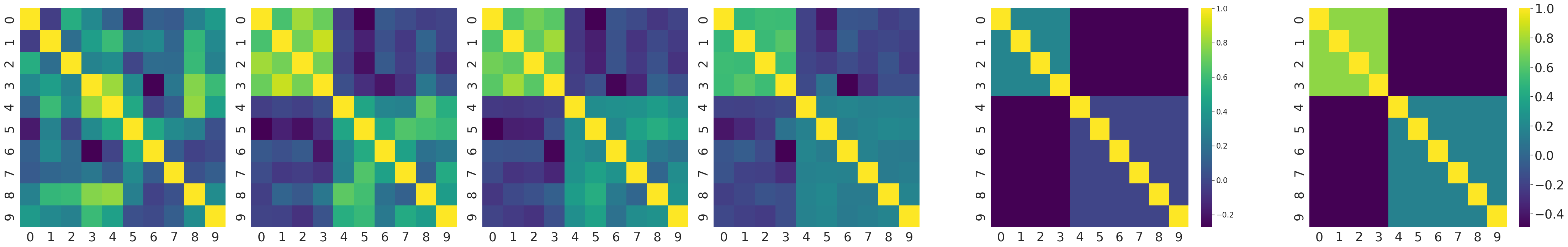}};
			% \node at (-0,-.3) [scale=0.8]{$\mathbf{P}$};
			\node at (-6.1,1.25) [scale=0.6]{Epoch 0};
			\node at (-3.9,1.25) [scale=0.6]{Epoch 100};
            \node at (-1.6,1.25) [scale=0.6]{Epoch 1000};
            \node at (0.4,1.25) [scale=0.6]{Epoch 3000};
			\node at (3.,1.25) [scale=0.6]{$\corr{{\Wmm}^\top}$};
			\node at (5.9,1.25) [scale=0.6]{$\corr{\smatbar^\top}$};
			\node at (-7.45,0.0) [scale=0.6, rotate=90]{$\corr{\W_k^\top}$};
		\end{tikzpicture}
    \end{subfigure}\vspace{-10pt}

    \vspace{7pt}
	\centering
	\hspace{15pt} 
	\hspace{-80pt} 
    \begin{subfigure}{0.9\textwidth}
		\centering
		\begin{tikzpicture}[remember picture]
			\node at (-0,-0) 
			% {\includegraphics[scale=0.16]{figures/UFM_Hheatmap.png}};
			% {\includegraphics[scale=0.16]{figures/UFM_Hheatmap_norm.png}};
            {\includegraphics[scale=0.13]{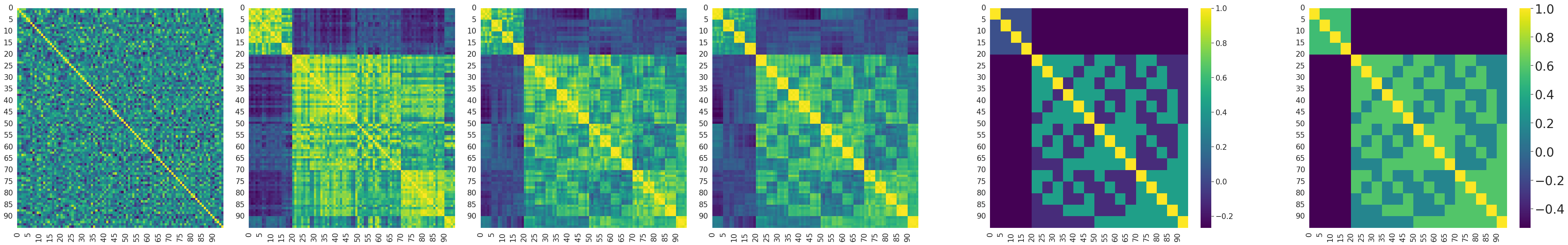}};
			% \node at (-0,-.3) [scale=0.8]{$\mathbf{P}$};
			\node at (-6.1,1.25) [scale=0.6]{Epoch 0};
			\node at (-3.9,1.25) [scale=0.6]{Epoch 100};
            \node at (-1.6,1.25) [scale=0.6]{Epoch 1000};
            \node at (0.4,1.25) [scale=0.6]{Epoch 3000};
			\node at (3.,1.25) [scale=0.6]{$\corr{{\Hmm}}$};
			\node at (5.9,1.25) [scale=0.6]{$\corr{{\smatbar}}$};
			\node at (-7.45,0.0) [scale=0.6, rotate=90]{$\corr{\Hb_k}$};
            % \node at (-3.3,-1.6) [scale=0.8]{\textbf{(a)}};
            % \node at (2.4,-1.6) [scale=0.8]{\textbf{(b)}};
            % \node at (5.6,-1.6) [scale=0.8]{\textbf{(c)}};
		\end{tikzpicture}
    \end{subfigure}
    \begin{tikzpicture}[remember picture,overlay]
        % \draw[red, thick] (current page.north) -- (current page.south);
        
   \draw[black, line width=2pt] (-3.5,7.5) -- (-3.5,0.15);
      \draw[black, line width=2pt] (-0.7,7.5) -- (-0.7,0.15);
   
    \end{tikzpicture}

  %   \vspace{15pt}
  %   \hspace{45pt}
  %   \begin{subfigure}{0.45\textwidth}
		% \centering
		% \begin{tikzpicture}[remember picture]
		% 	\node at (-0,-0) 
		% 	% {\includegraphics[scale=0.13]{figures/UFM_Lheatmap.png}};
  %           {\includegraphics[scale=0.45]{figures/UFM_loss.pdf}};
			
		% 	\node at (0.2,1.6) [scale=0.8]{Loss convergence};
  %          \node at (-2.8,0.2) [scale=0.8, rotate=90]{$\text{CE}(\Lb_k)-\Hc$};
  %            \node at (0.2,-1.6) [scale=0.6]{epoch $(k)$};
  %           \node at (0.2,-2.1) [scale=0.8]{\textbf{(d)}};
		% \end{tikzpicture}
  %   \end{subfigure}
  %   % 
  %   \hspace{-20pt}
  %   \begin{subfigure}{0.45\textwidth}
		% \centering
		% \begin{tikzpicture}[remember picture]
		% 	\node at (-0,-0) 
		% 	% {\includegraphics[scale=0.13]{figures/UFM_Lheatmap.png}};
  %           {\includegraphics[scale=0.45]{figures/UFM_SSIM.pdf}};
			
		% 	\node at (0.3,1.6) [scale=0.75]{Structural correlation with $\Sb$};
  %            \node at (0.2,-1.6) [scale=0.6]{epoch ($k$)};
  %           \node at (0.2,-2.1) [scale=0.8]{\textbf{(e)}};
		% \end{tikzpicture}
  %   \end{subfigure}
    \vspace{-0pt}
	\vspace{-0pt}
 \captionsetup{width=\linewidth}
 \vspace{0pt}
    \caption{Evolution of \ref{eq:ufm} parameters $\Lb_k$, $\corr{\W_k^\top}$ and $\corr{\Hb_k}$ when training close to convergence to the empirical entropy $\Hc$ (See Fig.~\ref{fig:UFM_curves}). At the end of the training, the parameters align with the prediction of Thm.~\ref{thm:reg path} ((a) vs (b)). Additionally, the correlation patterns between the embeddings closely follow the similarities between the support sets ((a) vs (c)). See Sec.~\ref{sec:ufmexp} for details. %We verify the similarity between the optimal embedding geometry and the support sets by measuring $\ssim{\corr{\Hb_k}}{\corr{\Sb}}$ (red) and $\ssim{\corr{\W_k^\top}}{\corr{\Sb^\top}}$ (blue) in plot (e).
    %Evolution of UFM parameters $\Lb_k$, $\corr{\W_k^\top}$ and $\corr{\Hb_k}$ when training close to convergence to the empirical entropy $\Hc$ (plot (d)). At the end of the training, the parameters align with the prediction of Thm.~\ref{thm:reg path} ((a) vs (b)). Additionally, the correlation patterns between the embeddings closely follows the similarities between the support sets ((a) vs (c)). We verify the similarity between the optimal embedding geometry and the support sets by measuring $\ssim{\corr{\Hb_k}}{\corr{\Sb}}$ (red) and $\ssim{\corr{\W_k^\top}}{\corr{\Sb^\top}}$ (blue) in plot (e).
    % Evolution of heatmaps of logits $\Lb$, cosine similarity of word embeddings $\W$ and context embeddings $\Hb$ of UFM during training: 1) at the end of training, heatmaps align with the prediction of Thm.~\ref{thm:reg path}, 2) word/context embeddings have higher correlations when their support set structures are more similar. See text for details. 
    % Heatmaps of $\Lb$, $\W\W^\top$ and $\Hb^\top\Hb$ of UFM during training (Sec.~\ref{sec:ufmexp}): 1) UFM parameters align with the prediction of Thm.~\ref{thm:reg path}, 2) The word and context embeddings have higher correlations when they have more similar support set structures.\tina{replace cosine similarity if time+bit vague}
    }
	\label{fig:UFM_heatmap}
	\vspace{20pt}
% \end{figure*}
% \begin{figure*}[h]
	% \vspace{-35pt}
	\centering
    \hspace{-20pt}
    \begin{subfigure}{0.22\textwidth}
		\centering
		\begin{tikzpicture}[remember picture]
			\node at (-0,-0) 
			% {\includegraphics[scale=0.13]{figures/UFM_Lheatmap.png}};
            {\includegraphics[scale=0.27]{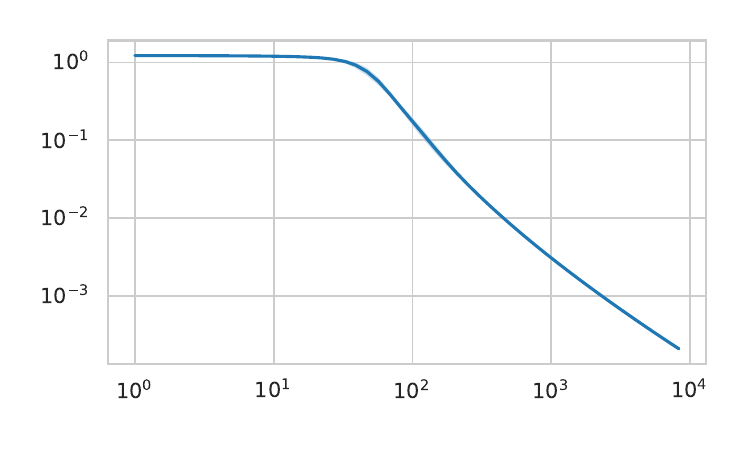}};
			
			\node at (0.1,1.2) [scale=0.8]{\textbf{(a)} Loss convergence};
           \node at (-1.9,0.2) [scale=0.7, rotate=90]{$\text{CE}(\Lb_k)-\Hc$};
             \node at (0.2,-1.2) [scale=0.6]{Epoch $(k)$};
            % \node at (0.2,-1.7) [scale=0.8]{\textbf{(a)}};
		\end{tikzpicture}
    \end{subfigure}
    \hspace{10pt}
    \begin{subfigure}{0.22\textwidth}
		\centering
		\begin{tikzpicture}[remember picture]
			\node at (-0,-0) 
			% {\includegraphics[scale=0.13]{figures/UFM_Lheatmap.png}};
            {\includegraphics[scale=0.27]{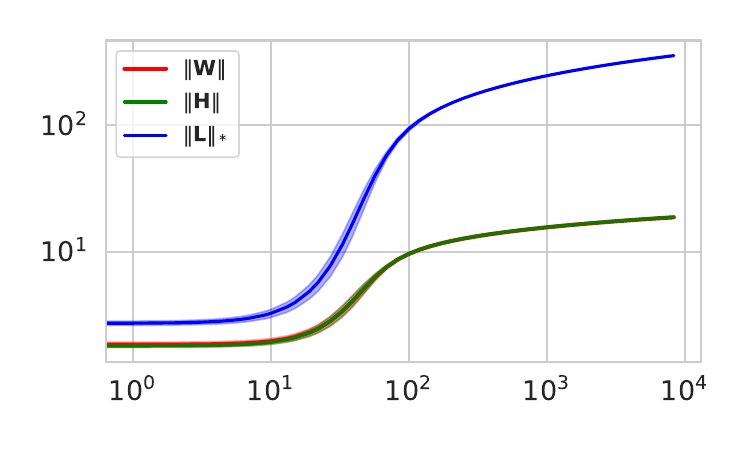}};
			
			\node at (0.1,1.2) [scale=0.75]{\textbf{(b)} Norm growth};
             \node at (0.2,-1.2) [scale=0.6]{Epoch ($k$)};
            % \node at (0.2,-1.7) [scale=0.8]{\textbf{(b)}};
		\end{tikzpicture}
    \end{subfigure}
    \hspace{0pt}
    \begin{subfigure}{0.22\textwidth}
		\centering
		\begin{tikzpicture}[remember picture]
			\node at (-0,-0) 
			% {\includegraphics[scale=0.13]{figures/UFM_Lheatmap.png}};
            {\includegraphics[scale=0.27]{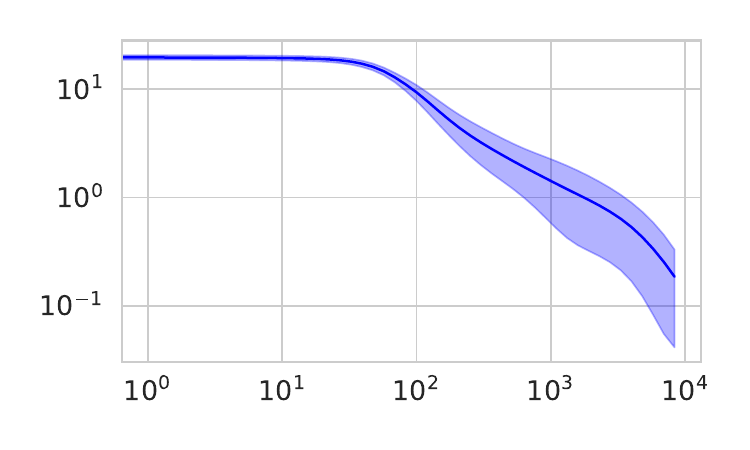}};
			
			\node at (0.1,1.2) [scale=0.75]{\textbf{(c)} $\norm{\Qcs(\Lb_k) - \Lbin}$};
             \node at (0.2,-1.2) [scale=0.6]{Epoch ($k$)};
            % \node at (0.2,-1.7) [scale=0.8]{\textbf{(c)}};
		\end{tikzpicture}
    \end{subfigure}
    \hspace{10pt}
    \begin{subfigure}{0.22\textwidth}
		\centering
		\begin{tikzpicture}[remember picture]
			\node at (-0,-0) 
			% {\includegraphics[scale=0.13]{figures/UFM_Lheatmap.png}};
            % {\includegraphics[scale=0.3]{figures/UFM_SSIM.pdf}};
            {\includegraphics[scale=0.27]{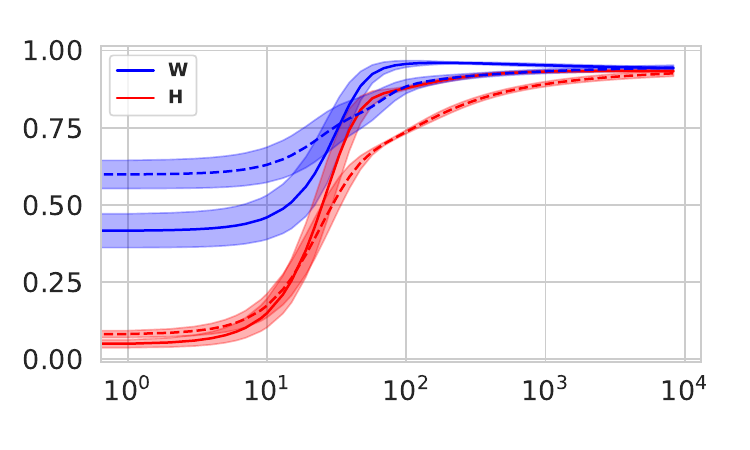}};
			
			\node at (0.1,1.2) [scale=0.75]{\textbf{(d)} Correlation with proxies};
             \node at (0.2,-1.2) [scale=0.6]{Epoch ($k$)};
            % \node at (0.2,-1.7) [scale=0.8]{\textbf{(d)}};
		\end{tikzpicture}
    \end{subfigure}
    \captionsetup{width=\linewidth}
    \vspace{0pt}
    \caption{Numerical experiments on \ref{eq:ufm}. \textbf{(a):} CE converges closely to the empirical entropy $\Hc$, \textbf{(b):} Norms of the parameters grow, \textbf{(c):} While the parameters converge directionally to $\Lbmm$ (Fig.~\ref{fig:UFM_heatmap}), the projection of $\Lb_k$ on $\Fc$ converges to the finite component $\Lbin$ specified by the soft-labels $\pbh_j$, \textbf{(d):} Structural correlation between the learned embeddings and the proxy, i.e., $\ssim{\Hb}{\smatbar}, \ssim{\W^\top}{\smatbar^\top}$. The correlation with $\Hmm$ and ${\Wmm}^\top$ (instead of $\smatbar, \, \smatbar^\top$) is displayed with dashed lines for reference.\looseness=-1
    % Correlation between the Gram matrix of $\Hb$ (resp. $\W^\top$) with $\Sb$ (resp. $\Sb^\top$). At the final stage of training the correlation between the parameters and support set proxies is high. See Sec.~\ref{sec:metrics} for the correlation metric. \tina{no convergence plot to Lmm included} \looseness=-1
    }
	\label{fig:UFM_curves}
	\vspace{0pt}
\end{figure*}

Through numerical experiments training on purely synthetic data, we confirm that the solution derived from (stochastic) gradient descent optimization of \ref{eq:ufm} is consistent with the analysis presented in Sec.~\ref{sec:analysis}. 

Specifically, we choose vocabulary size $V=10$, embedding dimension $d=10$, and $m=95$ training samples and generate (centered) support matrix $\smatbar$ as shown in Fig.~\ref{fig:UFM_heatmap}-(c). We also set $\pih_j = {1}/{m}$ and for this fixed support matrix $\smat$, we generate random soft labels $\pbh_{j}\in\Delta^{V-1}$ such that $\ph_{j,v}=0, v\not\in\Sc_j$.

We display the training evolution in Fig.~\ref{fig:UFM_curves}. We train \ref{eq:ufm} with SGD and small weight decay $\lambda=10^{-5}$ until reaching the empirical entropy lower-bound $\Hc\approx1.04$ within an order of $10^{-4}$ as shown in Fig.~\ref{fig:UFM_curves}-(a). This decrease in the loss is accompanied by a consistent increase in the parameter norms, as displayed in Panel (b). Despite the increase in norms, the projection of $\Lb_k$ onto the subspace $\Fc$ remains close to the finite component $\Lbin$ specified by Thm.~\ref{thm:reg path finite}. This is illustrated in Panel (c) and ensures the recovery of the soft-labels on the in-support tokens.  \looseness=-1

% Defining $\Lb_k, \Hb_k, \W_k$ to be the logits, embeddings and classifiers at iteration $k$, we verify the predictions of Thms.~\ref{thm:reg path} and \ref{thm:reg path finite} in Fig.~\ref{fig:UFM_conv}:  1) The loss converges to its lower-bound, $\Hc$, which we compute using the true soft labels. 2) The norms $\norm{\Lb_k}_*$, $\norm{\W_k}$, and $\norm{\Hb_k}$ keep growing as training progresses. 3) While $\Lb_k$ grows in norm, it converges in direction to $\Lbmm$ that only depends on the support matrix. However, its projection to the subspace $\Ts$ will align with $\Lbin$ that ensures the soft-labels are learned by the model. 4) As specified in Cor.~\ref{cor:WH}, since the relaxation in \eqref{eq:ufm relax} is tight, $\W_k\W_k^\top$ and $\Hb_k\Hb^\top_k$ converge in direction to $\GWmm:=\Ub\Sigmab\Ub^\top$ and $\GHmm:=\Vb\Sigmab\Vb^\top$. \new{Here, we note that, the convergence is slow which is consistent with the implicit bias results that generally exhibit a convergence rate of $\mathcal{O}(1/\log(k))$ \citep{soudry2018implicit}.}

In Fig.~\ref{fig:UFM_heatmap}, we visualize the logits $\Lb_k=\W_k\Hb_k$, and cosine similarities $\corr{\W_k^\top}$ and $\corr{\Hb_k}$ of word and context embeddings, respectively. 
% word embeddings $\corr{\W_k^\top}$ and context embeddings $\corr{\Hb_k}$. 
As shown in Panel (a) of Fig.~\ref{fig:UFM_heatmap}, as training continues the parameters recover the implicit geometry in Panel (b), which visualizes the prediction made by Thm.~\ref{thm:reg path} and Cor.~\ref{cor:WH}. In Panel (c), we also visualize the cosine similarity of $\smatbar$ and $\smatbar^\top$, which we introduced as a less expensive proxy of $\Wmm$ and $\Hmm$ (see Sec.~\ref{sec:proxy_main}). While the theory and proxy are not exactly the same, their heatmaps display similar structure. To verify the close relationship between the embeddings and the proxies quantitatively, we measure $\ssim{{\Hb}}{{\smatbar}}$ and $\ssim{{\W^\top}}{{\smatbar^\top}}$, where $\ssim{\Xb}{\Yb}$ measures the structural correlation between the two matrices (see Sec.~\ref{sec:metric}). Fig.~\ref{fig:UFM_curves}-(d) confirms the high correlation between the proxy \ref{P proxy} and the embeddings' implicit geometry. 
We also note that in the special case where the support sets of two contexts are identical, the heatmaps also verify the subspace collapse Claim \ref{P  collapse}: Along the diagonal blocks of $\corr{\Hb}$, where the contexts with similar support sets lie, the samples have maximum correlation and are aligned. \looseness=-1

% We train the UFM until reaching the empirical entropy lower-bound $\Hc=1.04\pm0.03$ within an order of $10^{-4}$ (Fig.~\ref{fig:UFM_curves}-(a)). This decrease in the loss is accompanied by a consistent increase in the parameters norms, as displayed in Panel (b). However, despite the increase in norms, the projection of $\Lb_k$ into the subspace $\Fc$ remains close to the finite component $\Lbin$ specified by Thm.~\ref{thm:reg path finite}. This is illustrated in panel (c) and ensures the recovery of the soft-labels on the tokens in the support set. Finally, to verify the close relationship between the embeddings and our proxies introduced in Sec.~\ref{sec:proxy} more quantitatively, we measure $\ssim{{\Hb}}{{\smatbar}}$ and $\ssim{{\W^\top}}{{\smatbar^\top}}$, where $\ssim{\Xb}{\Yb}$ measures the structural correlation between the two matrices (see Sec.~\ref{sec:metric}). Panel (d) confirms the high correlation between the proxy and the embeddings' implicit geometry. 
% These observations are consistent with the intuition of the distributional assumption that the meaning of a word/context is determined by the ``company it keeps'' \citep{harris1954distributional}.
% This behavior, which is reflected on embeddings, emerges from our model and analysis as a consequence of the implicit bias of the NTP training to promote low-rank logits subject to equality/inequality margin constraints for in/out-support tokens. 

We finally note that in this experiment $\Lbmm=\Lbtina$ (we verified that the dual-certificate condition in Prop.~\ref{propo:svm sufficient condition} holds). %\tina{fix the reference} \looseness=-1

\label{sec:ufmexp}

\section{Experiments}\label{sec:exp_main}
% \vspace{-5pt}
% In this section, we empirically validate our theoretical findings. We start  by tracking the GD path when training UFM and confirming its convergence to the theoretical claims outlined in Thms.~\ref{thm:reg path}, \ref{thm:reg path finite} and Cor.~\ref{cor:WH}. We then validate our predictions using transformers (TFs) trained on our synthetic data and on a subset of Tiny Stories \citep{eldan2023tinystories}. 
% \vspace{-0.1in}
% 
% In this section, we empirically validate our theoretical findings. We start by UFM, confirming that the solution found by GD converges to the claims outlined in Thms.~\ref{thm:reg path} and \ref{thm:reg path finite}. We then move on to experiments on deep-nets. We start by a synthetic setup where it is computationally feasible to find the optimal solution predicted by our theoretical analysis for $\W$ and $\Hb$. In these two setups, we empirically find a close proxy for the geometrical structure of the embeddings that does not rely on SVD calculations. We then use this proxy in a larger-scale experiment, where we train transformer (TF) based model on a portion of the TinyStories dataset.

\vspace{-5pt}
We now empirically validate our analysis on text data. 
% We begin with UFM, confirming that the solution found by gradient descent aligns with the regularization path characterized in Thms.~\ref{thm:reg path} and \ref{thm:reg path finite}. Next, 
% Unlike Sec.~\ref{sec:ufmexp}, the probability and sparsity matrices $\Pbb$ and $\smat$ abide by the structure of underlying language data. Also, embeddings are generated by training deep networks. 
We start with two small-scale synthetic datasets to examine Claims \ref{P logits}-\ref{P soft-labels}. 
We then experiment with a larger-scale dataset, where SVD calculation are computationaly expensive, we examine the Proxy \ref{P proxy}. \looseness=-1

\vspace{-5pt}
\subsection{Datasets}
% We use a total of three datasets. 
\vspace{-5pt}
For a detailed description of the datasets, refer to App.~\ref{sec:exp_details}.\looseness=-1 %We first employ two smaller-scale datasets, which are well-suited for verifying our theoretical solutions. Then, we use one larger-scale dataset designed to investigate the geometric properties of text in conditions that approximate real-world text scenarios. \looseness=-1 

% \begin{itemize}
    % \item 
    \vspace{-2pt}
    \noindent{\textbf{\synth.}}~We manually create simple (context, next-token) pairs, with context length of size $T-1=5$. The dataset consists of $n=116$ samples, containing $m=16$ distinct contexts, and a vocabulary size of $V=30$. Each context has a fixed support set length of $S_j=3$ and the empirical entropy lower-bound is $\Hc= 1.6597$. 
    % We generate each context $\xd_j$ as follows: We select $T-1=5$ words (tokens), and manually come up with the $T=6$-th tokens that are consistent with the given context. 
    % To model the probabilities $\pih_j$ and $\pbh_j$, we sample each support independently to a number of repeats to emulate the behavior of repeated context in natural language. For instance, the context $\xd_j=$ ``\texttt{Lily wants to try the}'' is followed by $\Sc=\{\texttt{soup}, \texttt{game}, \texttt{movie}\}$ with respective probabilities $ [0.5, 0.25, 0.25]$. 
    % {The dataset consists of $n=116$ samples, containing $m=16$ distinct contexts, and a vocabulary size of $V=30$. Each context has a fixed support set length of $|\mathcal{S}_j|=3$.} 
    % In Fig.~\ref{fig:verysmalldataset}, we show the $m=16$ unique contexts and the soft labels on their next token. \looseness=-1
    % In this dataset, the tokenization is done at the word level and the empirical entropy lower-bound is $\Hc= 1.6597$.% In Fig.~\ref{fig:verysmall_heatmaps}, we provide analogous plots to Fig.~\ref{fig:syn_heatmaps_w} for this setting.
    
    %\item 
    \vspace{-2pt}
    \noindent{\textbf{\simTS.}}~For a more realistic but still controlled setup, we curate the dataset from the \TS~corpus: 
    We derive contexts $\xd_j$ and support sets $\Sc_j$ by choosing the most frequent word-level contexts with length $T-1=5$. Here, $V=104$, $n\approx3050$ and $m\approx400$ with the empirical entropy lower-bound $\Hc= 1.0166$.\looseness=-1
    % For the support sets $\Sc_j$, we record all next tokens of $\xd_j$ in the original dataset. Then, we replace the words in the contexts with their synonymous to generate new contexts with identical support sets. This allows us to have multiple contexts sharing common support sets, while controlling the vocabulary size, which is set to $V=104$. 
    % {The final dataset consists of $n\approx3050$ samples with $m\approx400$ unique contexts and $V=104$.} \looseness=-1
    % The frequencies $\pbh_j$ are determined by  independently sampling each next token from $\Sc_j$ several times. Here, $\Hc= 1.0166$.
    
    %\item 
    \vspace{-2pt}
    \noindent{\textbf{\TS.}}~For more standard data, we use 100 stories sampled from \TS. Unlike the other two datasets, we do not sample over the contexts and support sets. We use a fixed context length $T-1=6$ for training. In the final dataset, $V=128$,  $m\sim 10^5$, and the empirical entropy lower-bound is $\Hc=0.3112$. We choose a small context window and vocabulary size to make tracking the distinct contexts and their support sets computationally manageable.\looseness=-1
    %We use a tokenizer with vocabulary size $V=128$, and fixed context length $T=6$. We choose a small context window and vocabulary size to make tracking the distinct contexts and support sets manageable. In this setup, the number of distinct contexts $m\sim 10^5$.% and $\Hc=0.3112$. \looseness=-1
% \end{itemize}

% \subsection{Models}
% \noindent\textbf{Models.}
\vspace{-5pt}
\subsection{Models}
\vspace{-5pt}
We train decoder-only TF with 4 layers and $d=128$ for \synth~and \simTS, and 12 layers and $d=256$ for \TS. For training details, see App.~\ref{sec:exp_details}.\looseness=-1
% on the three datasets described above. For \synth~and \simTS, we use a TF model with 4 layers, 6 attention heads and embedding dimension $d=128$. For \TS, we use a 12-layer TF with $d=256$. 

\vspace{-2pt}
We verify two claims made in Sec.~\ref{sec:analysis}: 1) We assess whether numerically optimizing both the TF and \ref{eq:ufm} in each setup yields the same embedding geometries, 2) We evaluate the global solution of \ref{eq:ufm}, as specified by Thm.~\ref{thm:reg path} and Cor.~\ref{cor:WH}, and check whether it respects the same implicit geometry as the TF and the \ref{eq:ufm}.
For the first one, we train \ref{eq:ufm} directly on the label distribution of each dataset; this is computationally feasible only for the first two datasets. For the second claim, we find the global solution $\Lbmm$ of \ref{eq:ufm-svm} on each dataset using CVXPY \citep{diamond2016cvxpy} and $\Hmm$, $\Wmm$ using Cor.~\ref{cor:WH}. If not computationally possible, we only use the Proxy \ref{P proxy} as reference.\looseness=-1

\vspace{-5pt}
\subsection{Metrics.}\label{sec:metric}
\vspace{-5pt}
% \noindent\textbf{Metrics.}\label{sec:metric}
We define the following metrics for verifying our theoretical results:
% To verify our theoretical analysis, We highlight the following key metrics:
% We track the CE loss convergence to the empirical entropy $\Hc$ of each dataset and verify the norm growth of the parameters as specified by Thm.~\ref{thm:reg path}. We also define the following metrics:

% \vspace{-5pt}
\noindent\textbf{Visualization of embeddings' geometry.}~To visualize the geometry of the word and context embeddings, we plot the heatmap of the normalized Gram matrix $\corr{\W^\top}$ and $\corr{\Hb}$, where for matrix $\Xb$,  $\corr{\Xb}$ is a matrix whose $(i,j)$-th entry is the cosine similarity between the respective columns of $\Xb$, i.e., \looseness=-1
$$\big[\corr{\Xb}\big]_{i,j}={\x_i^\top\x_j}/{(\|\x_i\|\|\x_j\|)}.$$
To check Proxy \ref{P proxy}, we compare the context and word embeddings with $\corr{\smatbar}$ and $\corr{\smatbar^\top}$. \looseness=-1

% \vspace{2pt}
\noindent\textbf{Quantifying geometric similarity.}~To measure the structural similarity between $\W$, $\Hb$ and $\smatbar^\top$, $\smatbar$, we use
%\begin{align}\label{eq:ssim}
\vspace{-2pt}
 $$   \ssim{\Xb}{\Yb} = ({\sigma_{\Xb\Yb} + \epsilon})/{(\sigma_\Xb \sigma_\Yb + \epsilon)},$$ as the correlation matrix between two matrices.\footnote{This is the Structural Similarity Index Measure (SSIM) used to measure the similarity between two images \citep{wang2004image}.}
%\end{align}
Here, \(\sigma_{\Xb\Yb}\) is the covariance, and \(\sigma_\Xb\) and \(\sigma_\Yb\) are the standard deviations of \(\Xb\) and \(\Yb\), respectively, and $\epsilon$ is a small constant for stable division. A value of \(\ssim{\Xb}{\Yb} = 1\) indicates perfect structural correlation. 
To simplify the presentation, and with some abuse of notation, we denote the similarity metric for context and word embeddings as follows,
% \footnote{To achieve \eqref{eq:proxy2}, we approximated the optimal solution and reduced computational complexity by omitting the square root. Therefore, we do not expect \eqref{eq:proxy2} to exactly match the Gram matrix of the embeddings. However, we want to determine if this proxy accurately predicts the structural properties of the embeddings. Thus, instead of directly computing the correlation between the embeddings and the centered support sets, we use the SSIM. \tina{keep or remove??}} 
\vspace{-2pt}
\begin{align*}
    \ssimstar{\Hb}{\smatbar}:=\ssim{\corr{\Hb}}{\corr{\smat}}, \quad \ssimstar{\W^\top}{\smatbar^\top}:=\ssim{\corr{\W^\top}}{\corr{\smat^\top}}.
\end{align*}

\vspace{-5pt}
\noindent\textbf{Recovery of soft-labels.}~To verify   Thm.~\ref{thm:reg path finite}, we measure the distance $\|\Qcs(\Lb_k) - \Lbin\|$ of $\Lbin$ from the projection of $\Lb_k$ onto the subspace $\Fc$. 
% \begin{align*}
%     \distF{\Lb_k}{\Lbin} = \|\Qcs(\Lb_k) - \Lbin\|
% \end{align*}

% \begin{itemize}
    % \item \textbf{Loss Convergence to Entropy:} We measure how closely the loss values converge to the theoretical entropy limits, indicating the efficiency of the model in capturing the underlying data distribution.

    % \item \textbf{Similarity between Optimal Embedding Geometry and Support Sets:} We assess the geometric similarity between the calculated optimal embeddings and the actual support sets. This evaluation includes comparisons of heatmap representations and the use of the Structural Similarity Index (SSIM) to measure similarity more quantitatively.

    % \item \textbf{Projected Logits Convergence to $\Lbin$:} \yz{Yize fix}
    
    % \item \textbf{Convergence of Embedding Gram Matrices:} We examine the convergence of the embedding gram matrices to the theoretical constructs described in Corollary \ref{cor:WH}. This metric checks for the structural and mathematical alignment of the embeddings with the theoretical predictions.
% \end{itemize}

% \vspace{-5pt}
% \noindent\textbf{Training Data.}~
% \noindent\textbf{Models.}~
% \input{Arxiv_NewTempl/sections/expyize}
% \tina{@Yize, check if any details are missing} 
% \vspace{-5pt}
% \subsection{Results}

\hspace{-20pt}
\begin{figure*}[t]% [htbp]
    \vspace{-30pt}
    \centering
    \resizebox{0.85\textwidth}{!}{
    \begin{tikzpicture}
        % First row
        \node (img1) at (0,0) {\includegraphics[width=0.33\textwidth]{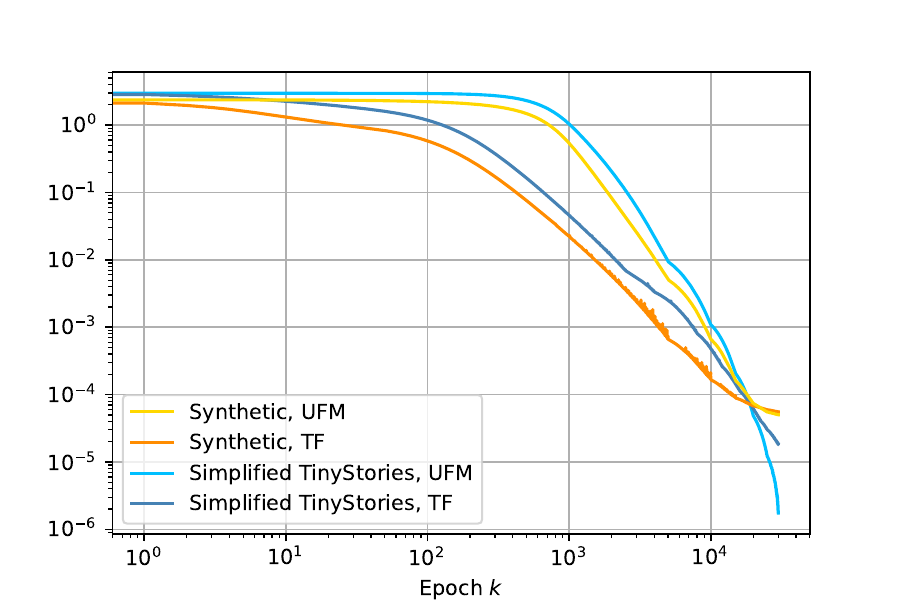}};
        \node (img2) at (5,0) {\includegraphics[width=0.33\textwidth]{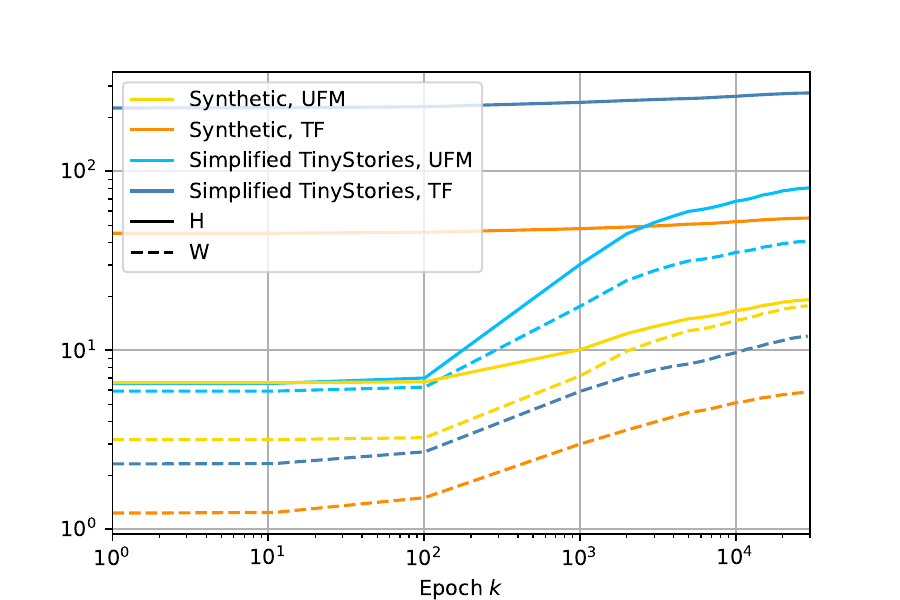}};
        \node (img3) at (10,0) {\includegraphics[width=0.33\textwidth]{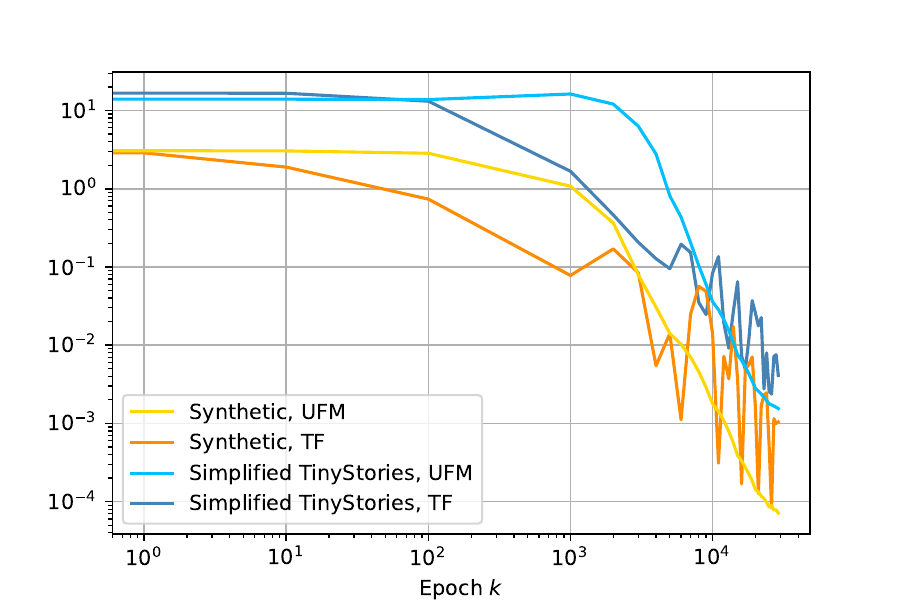}};
        
        % Second row
        \node (img4) at (2.5,-3.7)
        {\includegraphics[width=0.33\textwidth]{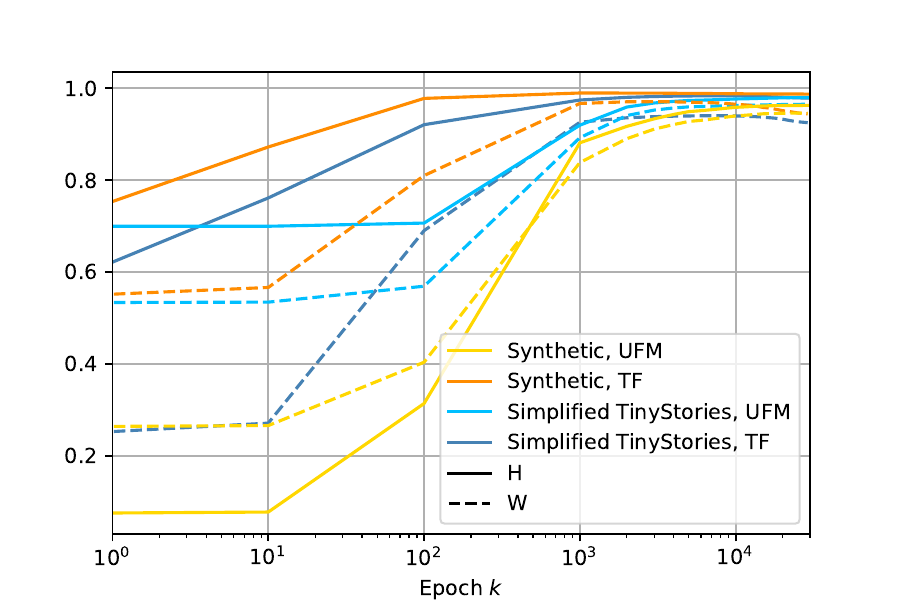}};
        \node (img5) at (7.5,-3.7) 
        {\includegraphics[width=0.33\textwidth]{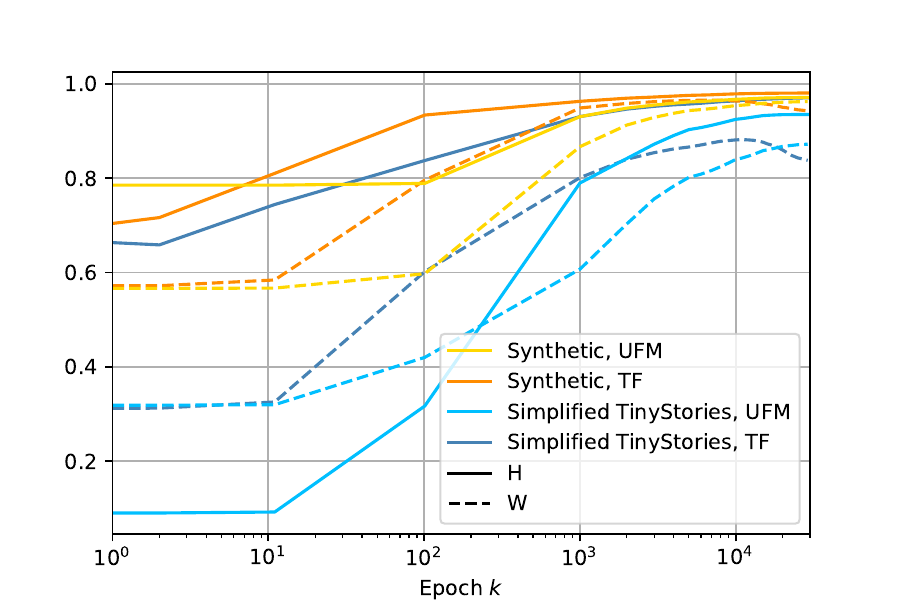}};
        
        % Center-top aligned text for each figure
        \node[anchor=south, text width=0.3\textwidth, align=center, fill=white, fill opacity=1] at (0,-1.85) {\tiny Epoch$(k)$};
        \node[anchor=south, text width=3cm, align=center, fill=white, fill opacity=1] at (5,-1.85) {\tiny Epoch$(k)$};;
        \node[anchor=south, text width=3cm, align=center, fill=white, fill opacity=1] at (10,-1.85) {\tiny Epoch$(k)$};;
        \node[anchor=south, text width=0.5\textwidth, align=center, fill=white, fill opacity=1] at (2.5,-5.6) {\tiny Epoch$(k)$};;
        \node[anchor=south, text width=0.5\textwidth, align=center, fill=white, fill opacity=1] at (7.5,-5.6) {\tiny Epoch$(k)$};

        \node[anchor=south, text width=0.3\textwidth, align=center] at (0,1.2) {\textbf{(a)} Loss convergence};
        \node[anchor=south, text width=3cm, align=center] at (5,1.2) {\textbf{(b)} Norm growth};
        \node[anchor=south, text width=3cm, align=center] at (10,1.2) {\textbf{(c)} $\norm{\Qcs(\Lb_k) - \Lbin}$};
        \node[anchor=south, text width=0.5\textwidth, align=center] at (2.5,-2.6) {\textbf{(d)} Correlation with proxies};
        \node[anchor=south, text width=0.5\textwidth, align=center] at (7.5,-2.6) {\textbf{(e)} Correlation with theory};

        % Y-axis labels, rotated and centered vertically
        \node[anchor=east, rotate=90] at (-2.1,0.8) {\tiny $\text{CE}(\Lb_k)-\Hc$};
        % \node[anchor=east, rotate=90] at (3.2,0.8) {\tiny Euclidean norm};
        % \node[anchor=east, rotate=90] at (8.2,0.8) {\tiny Y-axis Label 3};
        % \node[anchor=east, rotate=90] at (1.7,-4) {\tinyY-axis Label 4};
        % \node[anchor=east, rotate=90] at (6.7,-4) {\tinyY-axis Label 5};
    \end{tikzpicture}
    }
    \vspace{-7pt}
    \captionsetup{width=\textwidth}
    % \caption{\yz{whats wrong with this figure}}
    \caption{Experiments on \synth~and \simTS~datasets. 
    \textbf{(a):} CE approaches $\Hc$, \textbf{(b):} Parameters' norms grow during training. \textbf{(c):} $\Lb_k$'s projection on the data subspace $\Fc$ converges to the sparse component $\Lbin$ specified by the soft-labels $\pbh_j$. \textbf{(d):} $\ssimstar{\Hb_k}{\smatbar}$ (solid) and $\ssimstar{\W_k^\top}{\smatbar^\top}$ (dashed). At the final stage the learned embeddings exhibit high similarity with proxy \ref{P proxy}. \textbf{(e):} Same as (d), this time comparing with theory, i.e., $\ssimstar{\Hb_k}{\Hmm}$ and $\ssimstar{\W_k^\top}{{\Wmm}^\top}$. Details in Sec.~\ref{sec:exp_main}.\looseness=-1
    % \tina{@ Yize: xlabel's are not in latex so i can't change them: change xlabel to ``Epoch $k$''. } 
    }
    % \vspace{-15pt}
    \label{fig:synthetic_curves}
% \end{figure*} 

% \begin{figure*}[h]%[htbp]
% \vspace{-30pt}
    % \hspace{0pt}
    \vspace{10pt}
    \hspace{-30pt}
    \resizebox{0.95\textwidth}{!}{
    \begin{subfigure}{0.24\textwidth}
		\centering
		\begin{tikzpicture}
			\node at (-0,-1.4) 
			{\includegraphics[scale=0.2]{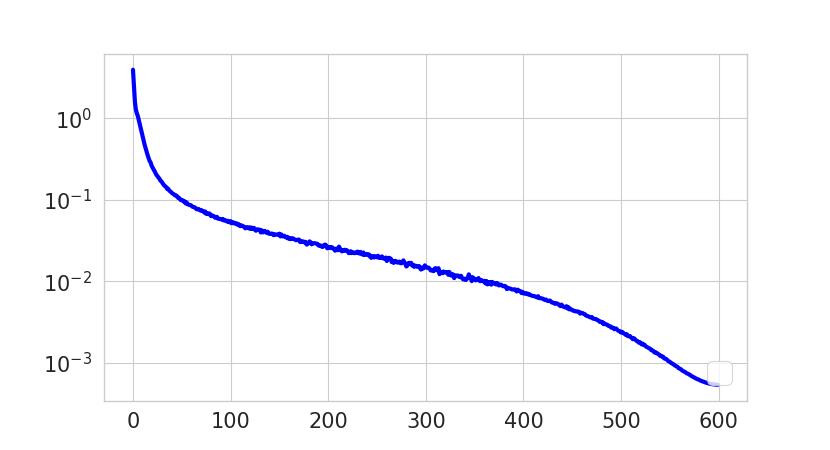}};
            \node at (0.1,-0.3) [scale=0.8]{\textbf{(a)} Loss convergence};
            \node at (-2.1,-1.5) [scale=0.8, rotate=90]{$\text{CE}(\Lb_k)-\Hc$};
            \node at (0.2,-2.7) [scale=0.6]{Epoch ($k$)};
		\end{tikzpicture}
    \end{subfigure}
    \hspace{-0pt} \begin{subfigure}{0.24\textwidth}
		\centering
		\begin{tikzpicture}
			\node at (-0,-1.4) 
			{\includegraphics[scale=0.2]{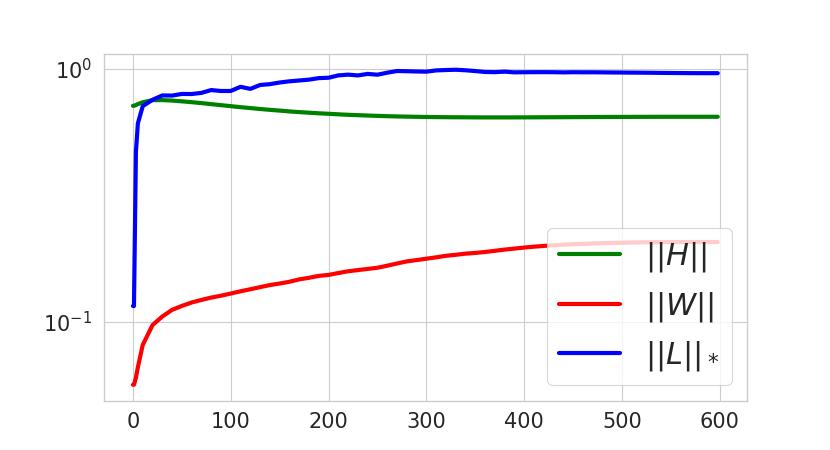}};
            \node at (0.2,-2.7) [scale=0.6]{Epoch ($k$)};
            \node at (0.1,-0.3) [scale=0.75]{\textbf{(b)} Norm growth};
		\end{tikzpicture}
    \end{subfigure}
    \hspace{-0pt} \begin{subfigure}{0.24\textwidth}
		\centering
		\begin{tikzpicture}
			\node at (-0,-1.4) 
			{\includegraphics[scale=0.2]{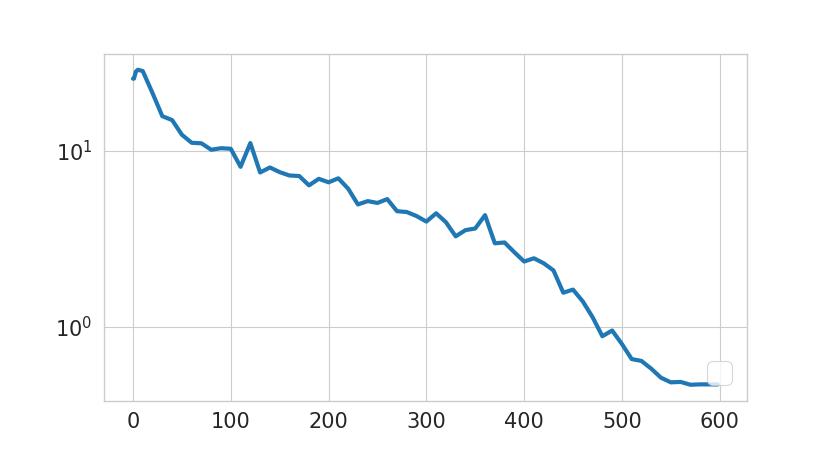}};
            \node at (0.1,-0.3) [scale=0.75]{\textbf{(c)} $\norm{\Qcs(\Lb_k) - \Lbin}$};
            \node at (0.2,-2.7) [scale=0.6]{Epoch ($k$)};
		\end{tikzpicture}
    \end{subfigure}
    \hspace{-0pt} \begin{subfigure}{0.24\textwidth}
		\centering
		\begin{tikzpicture}
			\node at (-0,-1.4) 
			{\includegraphics[scale=0.2]{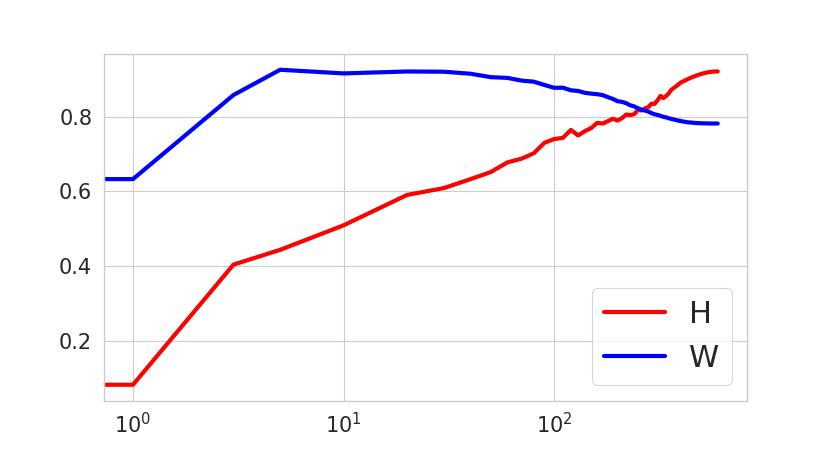}};
            \node at (0.1,-0.3) [scale=0.75]{\textbf{(d)} Correlation with proxies};
             \node at (0.2,-2.7) [scale=0.6]{Epoch ($k$)};
		\end{tikzpicture}
    \end{subfigure}
    }
	% %%%%%%%%%%%%%%%%%%%%%%%%%%%%%%%%%%%%%%%
    \vspace{-10pt}
    \captionsetup{width=\textwidth}
    \caption{
    Similar to Fig.~\ref{fig:UFM_curves} and \ref{fig:synthetic_curves}, this time on a 12-layer TF trained on a subset of $100$ stories from the \TS~dataset. \looseness=-1 %Here, computing the theoretical prediction $\Hmm$ is computationally expensive. Thus, we compare the embeddings geometry with the proxy. See Sec.~\ref{sec:exp_tiny} for details.  
    % 
    % 
     % Embeddings' geometry correlates with the structural properties of text data as captured by the sparsity pattern of the support set of distinct contexts within the dataset: $12$-layer transformer (TF) trained on a subset of $1000$ stories from the Tiny Stories \citep{eldan2023tinystories}. \textbf{(Left)} Loss approaches the empirical entropy $\Hc$ of the dataset. \textbf{(Middle)} The intersection between the support sets of $100$ contexts in the dataset. Entry $(i,j)$ equals $|\Sc_j\cap\Sc_i|$, where $\Sc_i$ is the set of possible next tokens for the $i$-th context. \textbf{(Right)} Gram matrix of the corresponding $100$ context embeddings $\Hb$ learned by the TF. Contexts that share a larger common set of next tokens, are arranged closer to each other in the embedding space. Contexts that share the exact same support set, i.e., $\Sc_j=\Sc_i$, approximately align with each other (observed on the diagonal blocks). See text for details. \tina{@Vala: Christos: ``in app, do you think is good idea to add a visualization of $\smat$ (or at least part of it) to give sense of the sparsity patterns?''}
    }
	\label{fig:TinyStories_curves} 
	% \vspace{-15pt}
\end{figure*}

\begin{figure}[h!]
    \centering
    \vspace{-30pt}
    \begin{tikzpicture}
        % Adding a text node as title above the image
        % Include the image in the figure
        \node[inner sep=0] (image) at (0,0) {\includegraphics[width=\linewidth]{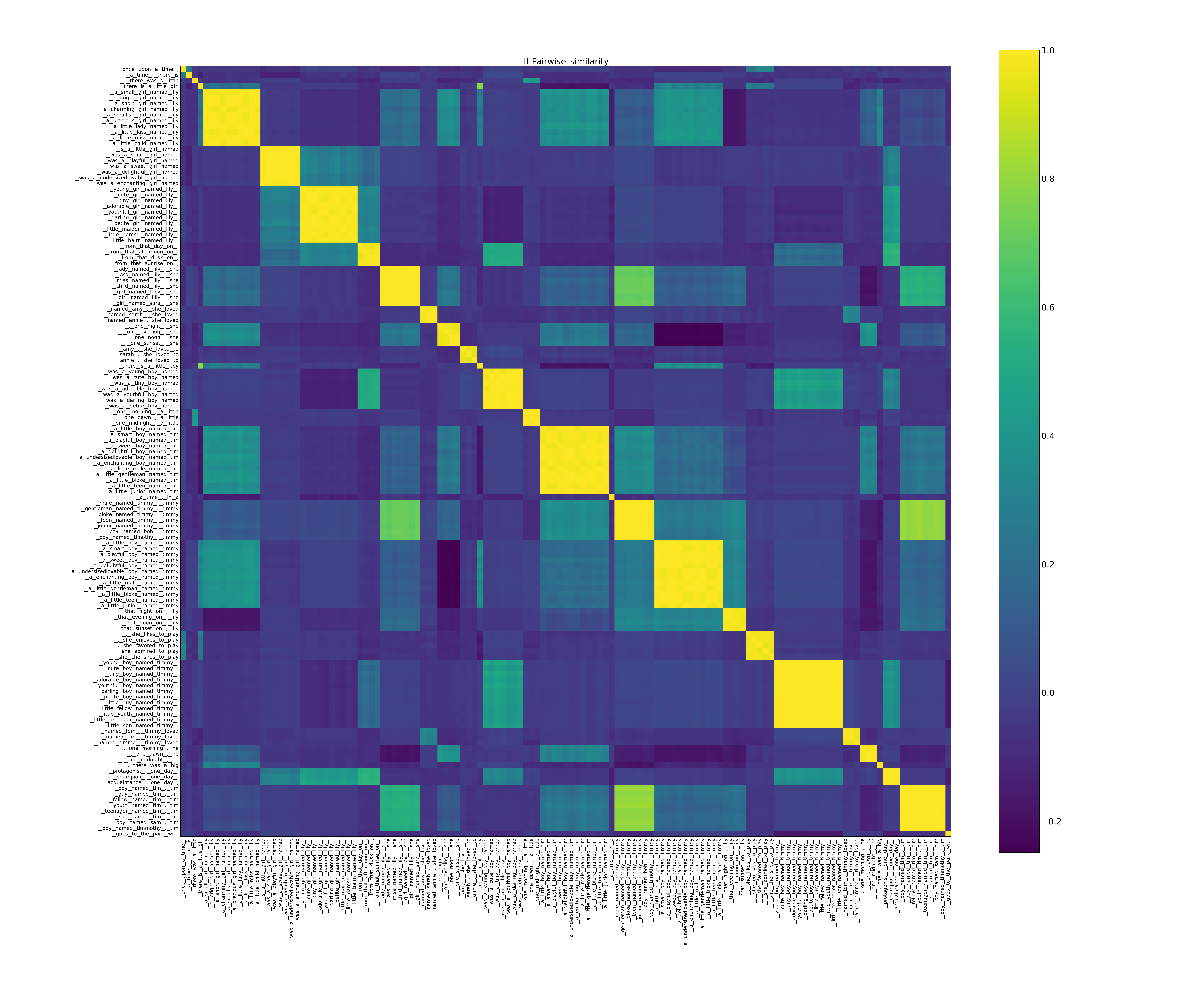}};
        \node[align=center, fill=white, fill opacity=1, text width=0.5\linewidth] at (-0.3,5.355) {$\corr{\Hb}$}; 
    \end{tikzpicture}
    \vspace{-30pt}
    \captionsetup{width=\textwidth}
    % \caption{Geometry of context embeddings for \simTS~dataset. A lighter color indicates higher similarity in the embedding space.}
    % \label{fig:sim_H_lable_404}
% \end{figure}
% 
% 
% \begin{figure}[h]

    \centering
    \vspace{-0pt}
    \begin{tikzpicture}
        % Adding a text node as title above the image
        % Include the image in the figure
        \node[inner sep=0] (image) at (0,0) {\includegraphics[width=\linewidth]{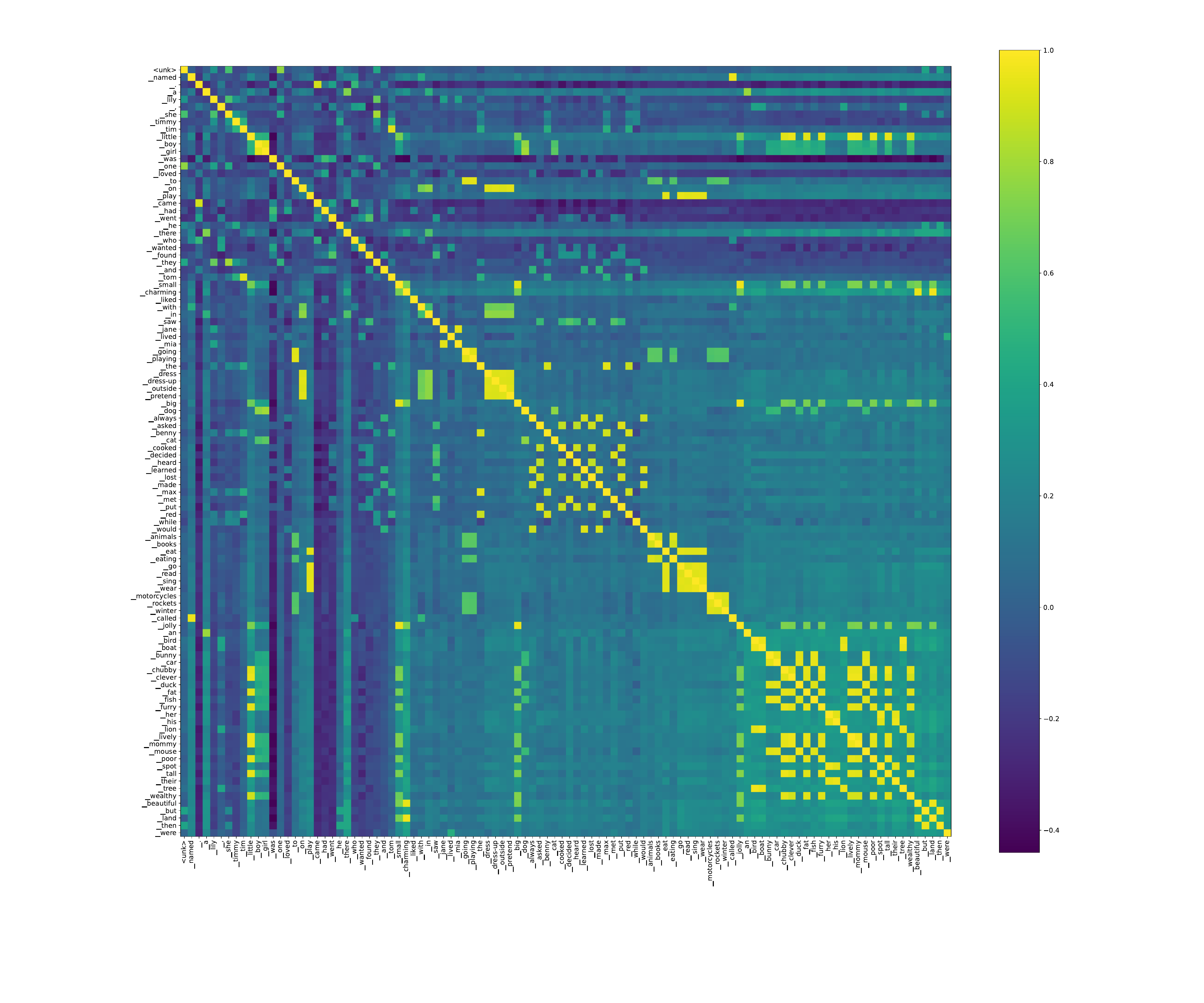}};
        \node[align=center, fill=white, fill opacity=1, text width=0.5\linewidth] at (-0.3,5.52) {$\corr{\W^\top}$}; 
    \end{tikzpicture}
    \vspace{-50pt}
    \captionsetup{width=\textwidth}
    \caption{Implicit geometry of context (Top) and word (Bottom) embeddings {and their associated text values} for \simTS~dataset. Lighter color indicates higher similarity in the embedding space.\looseness=-1}
    \label{fig:sim_W_lable_404}
\end{figure}

\vspace{-15pt}
\subsection{Results}
% \noindent\textbf{Results.}~
\vspace{-5pt}
In Sec.~\ref{sec:intro_methodology}, we discussed Fig.~\ref{fig:simplified_intro}, where we compared the logits, word and context embeddings trained on \simTS~by (a) TF, (b) \ref{eq:ufm}, and those predicted by (c) Claims \ref{P logits}-\ref{P  collapse}, and (d) Proxy \ref{P proxy}. Similar observations for the \TS~and \synth~datasets are shown in Figs.~\ref{fig:intro_TinyStories} and \ref{fig:synthetic_heatmap_small}. For visualization details, see App.~\ref{sec:exp_details}. \looseness=-1

Next, we discuss Fig.~\ref{fig:synthetic_curves} that tracks our metrics during training for the \synth~and \simTS~datasets. Fig.~\ref{fig:synthetic_curves}-(a) confirms that the loss is close to the empirical entropy $\Hc$. Fig.~\ref{fig:synthetic_curves}-(b) confirms the norm growth of the parameters. Note the slower growth of the context embeddings $\Hb$ in the TF model, which we suspect is due to the layer normalization at the final layer. We also confirm Claim \ref{P soft-labels}, the recovery of the soft labels,  in Fig.~\ref{fig:synthetic_curves}-(c). At last, Fig.~\ref{fig:synthetic_curves}-(d) and (e) track the correlation of the learned geometries with the proxy \ref{P proxy} and with the theoretical predictions $(\Hmm,\Wmm)$. Note that all similarity metrics increase towards $1$ as training progresses. In Figs.~\ref{fig:m404steps} and \ref{fig:verysmallsteps} in the appendix, we display the evolution of the embeddings by capturing their heatmaps at various training checkpoints. Interestingly, the final patterns emerge early in the training process. Fig.~\ref{fig:TinyStories_curves} displays analogous plots  for \TS.\looseness=-1

{Finally, in Figs.~\ref{fig:sim_W_lable_404} and \ref{fig:sim_W_label_verysmall}, we visualize the text values of the contexts and words along with their geometry heatmaps. Observe that the embeddings' similarities are consistent with the linguistic properties of words/contexts: Embeddings of contexts such as \texttt{``boy named Timmy. Timmy''} and \texttt{``kid called Lilly. She''} align closely as they are likely to be followed by similar set of words {(i.e., they have similar sparsity patterns)}. Also, word embeddings tend to cluster by their grammatical categories, with verbs and adjectives often showing high similarity. For instance, the verbs \texttt{go}, \texttt{eat}, \texttt{sing}, \texttt{wear}, \texttt{play}, and the adjectives \texttt{clever}, \texttt{fat}, \texttt{furry}, \texttt{lively}, \texttt{little} have high similarity. {This contrasts with the maximally separated word embeddings in the unrealistic, perfectly symmetric setting analyzed in Section \ref{sec:special_case}, for which word embeddings where maximally separated. Here, the sparsity patterns of natural language induces  an intricate geometry of word embeddings that accurately reflects linguistic structures such as grammatical categories.}\looseness=-1

\vspace{-10pt}
% \subsection{Deep-nets}

% \input{Arxiv_NewTempl/sections/exps}
% \section{Auto-regressive training \tina{??}}
% \input{Arxiv_NewTempl/sections/AR-2}
% \input{Arxiv_NewTempl/sections/autoregressive}

% \input{Arxiv_NewTempl/sections/rel}
\vspace{5pt}
\section{Concluding remarks: Limitations and future work}\label{sec:conclusion}
\vspace{-5pt}
% \vspace{-0.1in}
% \vspace{-0.1in}
We have merged insights from the study of deep-net embeddings' geometry in one-hot classification, implicit regularization theories, and a sparse soft-label classification framing of NTP, to arrive at a framework for examining word and context embeddings under the NTP  paradigm. Key to our approach is the simplification to a problem of rank-constrained nuclear-norm-regularized CE minimization across logits (Eq. \eqref{eq:ufm relax}). When the embedding dimension $d$ is large relative to the vocabulary size $V$, we reveal how the similarities between word/context embeddings are heavily influenced by the sparsity patterns of the support sets. We do so by characterizing the implicit bias of NTP as it approaches the empirical entropy lower-bound leading to Claims \ref{P logits}-\ref{P soft-labels}.%: 1) a low-rank component identified by \ref{eq:ufm-svm} that maximizes the margin between in- and out-support tokens of each context, similar to SVM optimization in the one-hot setup, and 2) a sparse orthogonal component that recovers the empirical soft-labels over the in-support tokens. 
% This enables us to reveal how the similarities between word/context embeddings are heavily influenced by the sparsity patterns of the support sets.

% this enables a detailed solution representation \ref{eq:ufm-svm}, akin to SVM optimization.

Looking forward, we see several promising avenues for expanding upon our framework and addressing its present limitations:

First, we believe it is feasible to broaden the regularization-path analysis to the non-convex realm of problem \eqref{eq:ufm relax}, potentially revealing a rank-constrained variant of \ref{eq:ufm-svm}. We hypothesize that the support matrices typical of natural language data might accommodate low-rank solutions that adhere to the constraints specified in Eqns.~\eqref{eq:svm equalities}-\eqref{eq:svm inequalities}. This would allow us to relax the $d\geq V$ constraint in our analysis; see App.~\ref{sec:d<V} for more discussions. {But even under the condition $d\geq V$ several exciting directions remain open. 
 For example, investigating the implicit bias of GD—and potentially other optimizers like Adam, which are more commonly used in language model training but might present additional challenges—is important. In optimizing \eqref{eq:ufm relax}, does GD follow the regularization path? If not, to what extent does the solution learned by GD deviate from the sparse plus low-rank structure?}\looseness=-1

Second, while our initial focus was on establishing the framework, conducting preliminary analyses and proof-of-concept experiments, we encourage more extensive experimental exploration, especially with larger datasets and more complex model architectures with larger context window. Experiments in this paper focused on TF models due to their recent success in language. However, our theoretical results are independent of network design. Instead, they require the model to be over-parameterized enough to reach the entropy lower bound and ensure expressiveness in the embedding space. For an example, see App.~\ref{sec:mlp}. 

Third, it is intriguing to investigate how the distributional characteristics of word/context pairs in natural language, represented by the probability matrix $\Pbb$—and primarily by its support $\smat$—influence key properties like low-rankness and symmetry in the SVD factors of $\Lbmm$. {Relatedly, we aim to leverage these insights to more effectively explore the empirical connection between $\Lbmm$ and the centered support matrix $\smatbar$} complementing the heuristic arguments in Sec. \ref{sec:proxy_main} in support of proxy \ref{P proxy}. {Considering this proxy, a low-rank structure of $\smatbar$ would also imply a low-rank nature for $\Lbmm$ (currently only indirectly supported by low-rank minimization).}
Fourth, although the framework encompasses full autoregressive training, our experimental study has focused on prediction of the last token prediction, with the exception of preliminary experiments in App.~\ref{sec:ar}. Deeper studies are required with sequences of variable length to empirically evaluate the brittleness of our model's assumption, which enumerates distinct contexts regardless of their length or constituent token composition.

Fifth, our framework characterizes the geometry of word/context representations in a saturating regime where the NTP loss reaches its entropy lower-bound, suggesting that the model is sufficiently overparameterized. Although current state-of-the-art models do not reach this saturation regime and are likely not overparameterized, emerging evidence points to the need for extended training periods to {achieve better generalization \citep{power2022grokking} or to} learn complex behaviors like hierarchical structures \citep{murty2023grokking} {and high-order in-context statistics \citep{edelman2024evolution}}. It's also plausible that as models grow larger and training methods become more efficient, extended training periods may become feasible. Investigating the potential benefits of prolonged training could yield valuable insights, and we hypothesize that the implicit optimization biases of NTP, as discussed in \cite{ntp} and our work, offer a compelling framework for such exploration. Additionally, determining the degree of overparameterization required to achieve the entropy lower-bound presents an intriguing challenge in its own right \cite{madden2024upper}.\looseness=-1

% \ct{I want to add that nuc-norm might not be the correct answer for GD and this open question}
% 
% represents an intriguing direction for future research. \new{In our analysis, we focus on the long training regime with over-parameterized models which allows us to treat context embeddings as independent embeddings, regardless of their length or constituent tokens. Thus, our analysis applies to variable-length sequences within this regime. However, it is valuable to develop theoretical models that better capture the intricacies of training that depend on the context sequence length and tokens. For initial results, see App.~\ref{sec:ar}.
% }
% \new{Our analysis does not explicitly depend on the length of the context sequence and holds for variable-length input as long as the model is expressive and trained for long enough. However,    . For initial results, see App.~\ref{sec:ar}.}
% \new{Our analysis does not explicitly depend on the length of the context sequence and holds for variable-length input as long as the model is expressive enough to learn independent context embeddings, even for embeddings that share part of their context tokens. This also requires the model to be trained long enough to reach the entropy lower-bound. In other cases, accurately capturing the dependence of embeddings learned from auto-regressively chosen contexts seems essential for predicting their implicit geometry. For initial results, see App.~\ref{sec:ar}.} \ct{emphasize ansatz of long training}

Ultimately, our goal is to spur further exploration into the geometries of context and word embeddings. 
Akin to recent promising studies in image-classification that leverage the neural-collapse geometry, insights gained from such studies within NTP could  inform the development of more effective loss functions, accelerate training processes, provide defenses against adversarial attacks, and guide the architectural choices for downstream tasks.
More broadly, we aim to foster interest in deeper investigations into how the inductive bias of NTP influences the learning of linguistic patterns and regularities.

\vspace{5pt}
\subsubsection*{Acknowledgments}
This work was funded by the NSERC Discovery Grant No. 2021-03677, the Alliance Grant ALLRP 581098-22, and an Alliance Mission Grant. TB gratefully acknowledges the support from UBC’s Four Year Doctoral Fellowship and a scholarship from the Advanced Machine Learning Training Network at UBC. VV gratefully acknowledges support from BC Graduate Scholarship (BCGS) and Canadian Graduate Scholarship for Master's (CGSM).

\vspace{20pt}
% \clearpage
% \newpage

% \bibliography{colm2024_conference}
\bibliography{refs,transformers,refs_NC,refs_generalization,bib_extra}

\begin{thebibliography}{70}
\providecommand{\natexlab}[1]{#1}
\providecommand{\url}[1]{\texttt{#1}}
\expandafter\ifx\csname urlstyle\endcsname\relax
  \providecommand{\doi}[1]{doi: #1}\else
  \providecommand{\doi}{doi: \begingroup \urlstyle{rm}\Url}\fi

\bibitem[Allen \& Hospedales(2019)Allen and Hospedales]{allen2019analogies}
Carl Allen and Timothy Hospedales.
\newblock Analogies explained: Towards understanding word embeddings.
\newblock In \emph{International Conference on Machine Learning}, pp.\  223--231. PMLR, 2019.

\bibitem[Arora et~al.(2016)Arora, Li, Liang, Ma, and Risteski]{arora2016latent}
Sanjeev Arora, Yuanzhi Li, Yingyu Liang, Tengyu Ma, and Andrej Risteski.
\newblock A latent variable model approach to pmi-based word embeddings.
\newblock \emph{Transactions of the Association for Computational Linguistics}, 4:\penalty0 385--399, 2016.

\bibitem[Azizan et~al.(2021)Azizan, Lale, and Hassibi]{azizan2021stochastic}
Navid Azizan, Sahin Lale, and Babak Hassibi.
\newblock Stochastic mirror descent on overparameterized nonlinear models.
\newblock \emph{IEEE Transactions on Neural Networks and Learning Systems}, 33\penalty0 (12):\penalty0 7717--7727, 2021.

\bibitem[Baroni et~al.(2014)Baroni, Dinu, and Kruszewski]{neural_language_models_3}
Marco Baroni, Georgiana Dinu, and Germ{\'a}n Kruszewski.
\newblock Don’t count, predict! a systematic comparison of context-counting vs. context-predicting semantic vectors.
\newblock In \emph{Proceedings of the 52nd Annual Meeting of the Association for Computational Linguistics (Volume 1: Long Papers)}, pp.\  238--247, 2014.

\bibitem[Behnia et~al.(2023)Behnia, Kini, Vakilian, and Thrampoulidis]{behnia2023implicit}
Tina Behnia, Ganesh~Ramachandra Kini, Vala Vakilian, and Christos Thrampoulidis.
\newblock On the implicit geometry of cross-entropy parameterizations for label-imbalanced data.
\newblock In \emph{International Conference on Artificial Intelligence and Statistics}, pp.\  10815--10838. PMLR, 2023.

\bibitem[Bengio \& Bengio(2000)Bengio and Bengio]{bengio2000taking}
Samy Bengio and Yoshua Bengio.
\newblock Taking on the curse of dimensionality in joint distributions using neural networks.
\newblock \emph{IEEE Transactions on Neural Networks}, 11\penalty0 (3):\penalty0 550--557, 2000.

\bibitem[Bengio et~al.(2000)Bengio, Ducharme, and Vincent]{neural_language_models_2}
Yoshua Bengio, R{\'e}jean Ducharme, and Pascal Vincent.
\newblock A neural probabilistic language model.
\newblock \emph{Advances in neural information processing systems}, 13, 2000.

\bibitem[Bi{\'s} et~al.(2021)Bi{\'s}, Podkorytov, and Liu]{bis2021too}
Daniel Bi{\'s}, Maksim Podkorytov, and Xiuwen Liu.
\newblock Too much in common: Shifting of embeddings in transformer language models and its implications.
\newblock In \emph{Proceedings of the 2021 conference of the North American chapter of the Association for Computational Linguistics: Human Language Technologies}, pp.\  5117--5130, 2021.

\bibitem[Brown et~al.(2020)Brown, Mann, Ryder, Subbiah, Kaplan, Dhariwal, Neelakantan, Shyam, Sastry, Askell, Agarwal, Herbert-Voss, Krueger, Henighan, Child, Ramesh, Ziegler, Wu, Winter, Hesse, Chen, Sigler, Litwin, Gray, Chess, Clark, Berner, McCandlish, Radford, Sutskever, and Amodei]{brown2020language}
Tom~B. Brown, Benjamin Mann, Nick Ryder, Melanie Subbiah, Jared Kaplan, Prafulla Dhariwal, Arvind Neelakantan, Pranav Shyam, Girish Sastry, Amanda Askell, Sandhini Agarwal, Ariel Herbert-Voss, Gretchen Krueger, Tom Henighan, Rewon Child, Aditya Ramesh, Daniel~M. Ziegler, Jeffrey Wu, Clemens Winter, Christopher Hesse, Mark Chen, Eric Sigler, Mateusz Litwin, Scott Gray, Benjamin Chess, Jack Clark, Christopher Berner, Sam McCandlish, Alec Radford, Ilya Sutskever, and Dario Amodei.
\newblock Language models are few-shot learners, 2020.

\bibitem[Cattaneo et~al.(2023)Cattaneo, Klusowski, and Shigida]{cattaneo2023implicit}
Matias~D Cattaneo, Jason~M Klusowski, and Boris Shigida.
\newblock On the implicit bias of adam.
\newblock \emph{arXiv preprint arXiv:2309.00079}, 2023.

\bibitem[Diamond \& Boyd(2016)Diamond and Boyd]{diamond2016cvxpy}
Steven Diamond and Stephen Boyd.
\newblock Cvxpy: A python-embedded modeling language for convex optimization.
\newblock \emph{Journal of Machine Learning Research}, 17\penalty0 (83):\penalty0 1--5, 2016.

\bibitem[Edelman et~al.(2024)Edelman, Edelman, Goel, Malach, and Tsilivis]{edelman2024evolution}
Benjamin~L Edelman, Ezra Edelman, Surbhi Goel, Eran Malach, and Nikolaos Tsilivis.
\newblock The evolution of statistical induction heads: In-context learning markov chains.
\newblock \emph{arXiv preprint arXiv:2402.11004}, 2024.

\bibitem[Eldan \& Li(2023)Eldan and Li]{eldan2023tinystories}
Ronen Eldan and Yuanzhi Li.
\newblock Tinystories: How small can language models be and still speak coherent english?
\newblock \emph{arXiv preprint arXiv:2305.07759}, 2023.

\bibitem[Ethayarajh(2019)]{ethayarajh2019contextual}
Kawin Ethayarajh.
\newblock How contextual are contextualized word representations? comparing the geometry of bert, elmo, and gpt-2 embeddings.
\newblock \emph{arXiv preprint arXiv:1909.00512}, 2019.

\bibitem[Ethayarajh et~al.(2018)Ethayarajh, Duvenaud, and Hirst]{ethayarajh2018towards}
Kawin Ethayarajh, David Duvenaud, and Graeme Hirst.
\newblock Towards understanding linear word analogies.
\newblock \emph{arXiv preprint arXiv:1810.04882}, 2018.

\bibitem[Fang et~al.(2021)Fang, He, Long, and Su]{fang2021exploring}
Cong Fang, Hangfeng He, Qi~Long, and Weijie~J Su.
\newblock Exploring deep neural networks via layer-peeled model: Minority collapse in imbalanced training.
\newblock \emph{Proceedings of the National Academy of Sciences}, 118\penalty0 (43), 2021.

\bibitem[Fazel(2002)]{Fazel}
Maryam Fazel.
\newblock \emph{Matrix rank minimization with applications}.
\newblock PhD thesis, PhD thesis, Stanford University, 2002.

\bibitem[Fisher et~al.(2024)Fisher, Meng, and Papyan]{fisher2024pushing}
Quinn Fisher, Haoming Meng, and Vardan Papyan.
\newblock Pushing boundaries: Mixup's influence on neural collapse.
\newblock \emph{arXiv preprint arXiv:2402.06171}, 2024.

\bibitem[Gao et~al.(2019)Gao, He, Tan, Qin, Wang, and Liu]{gao2019degeneration}
Jun Gao, Di~He, Xu~Tan, Tao Qin, Liwei Wang, and Tie-Yan Liu.
\newblock Representation degeneration problem in training natural language generation models, 2019.

\bibitem[Godey et~al.(2024)Godey, de~la Clergerie, and Sagot]{godey2024anisotropy}
Nathan Godey, {\'E}ric de~la Clergerie, and Beno{\^\i}t Sagot.
\newblock Anisotropy is inherent to self-attention in transformers.
\newblock \emph{arXiv preprint arXiv:2401.12143}, 2024.

\bibitem[Grant \& Boyd(2014)Grant and Boyd]{cvx}
Michael Grant and Stephen Boyd.
\newblock {CVX}: Matlab software for disciplined convex programming, version 2.1.
\newblock \url{http://cvxr.com/cvx}, March 2014.

\bibitem[Gunasekar et~al.(2018{\natexlab{a}})Gunasekar, Lee, Soudry, and Srebro]{gunasekar2018characterizing}
Suriya Gunasekar, Jason Lee, Daniel Soudry, and Nathan Srebro.
\newblock Characterizing implicit bias in terms of optimization geometry.
\newblock In \emph{International Conference on Machine Learning}, pp.\  1832--1841. PMLR, 2018{\natexlab{a}}.

\bibitem[Gunasekar et~al.(2018{\natexlab{b}})Gunasekar, Lee, Soudry, and Srebro]{gunasekar2018implicit}
Suriya Gunasekar, Jason~D Lee, Daniel Soudry, and Nati Srebro.
\newblock Implicit bias of gradient descent on linear convolutional networks.
\newblock \emph{Advances in Neural Information Processing Systems}, 31:\penalty0 9461--9471, 2018{\natexlab{b}}.

\bibitem[Harris(1954)]{harris1954distributional}
Zellig~S Harris.
\newblock Distributional structure.
\newblock \emph{Word}, 10\penalty0 (2-3):\penalty0 146--162, 1954.

\bibitem[Hashimoto et~al.(2016)Hashimoto, Alvarez-Melis, and Jaakkola]{hashimoto2016word}
Tatsunori~B Hashimoto, David Alvarez-Melis, and Tommi~S Jaakkola.
\newblock Word embeddings as metric recovery in semantic spaces.
\newblock \emph{Transactions of the Association for Computational Linguistics}, 4:\penalty0 273--286, 2016.

\bibitem[Hutchinson et~al.(2012)Hutchinson, Ostendorf, and Fazel]{hutchinson2012sparse}
Brian Hutchinson, Mari Ostendorf, and Maryam Fazel.
\newblock A sparse plus low rank maximum entropy language model.
\newblock In \emph{INTERSPEECH}, pp.\  1676--1679, 2012.

\bibitem[Hutchinson et~al.(2013)Hutchinson, Ostendorf, and Fazel]{hutchinson2013exceptions}
Brian Hutchinson, Mari Ostendorf, and Maryam Fazel.
\newblock Exceptions in language as learned by the multi-factor sparse plus low-rank language model.
\newblock In \emph{2013 IEEE International Conference on Acoustics, Speech and Signal Processing}, pp.\  8580--8584. IEEE, 2013.

\bibitem[Hutchinson et~al.(2015)Hutchinson, Ostendorf, and Fazel]{hutchinson2015sparse}
Brian Hutchinson, Mari Ostendorf, and Maryam Fazel.
\newblock A sparse plus low-rank exponential language model for limited resource scenarios.
\newblock \emph{IEEE/ACM Transactions on Audio, Speech, and Language Processing}, 23\penalty0 (3):\penalty0 494--504, 2015.

\bibitem[Ji \& Telgarsky(2018)Ji and Telgarsky]{ji2018risk}
Ziwei Ji and Matus Telgarsky.
\newblock Risk and parameter convergence of logistic regression.
\newblock \emph{arXiv preprint arXiv:1803.07300}, 2018.

\bibitem[Ji \& Telgarsky(2020)Ji and Telgarsky]{ji2020directional}
Ziwei Ji and Matus Telgarsky.
\newblock Directional convergence and alignment in deep learning.
\newblock \emph{Advances in Neural Information Processing Systems}, 33:\penalty0 17176--17186, 2020.

\bibitem[Ji et~al.(2020)Ji, Dud{\'\i}k, Schapire, and Telgarsky]{ji2020gradient}
Ziwei Ji, Miroslav Dud{\'\i}k, Robert~E Schapire, and Matus Telgarsky.
\newblock Gradient descent follows the regularization path for general losses.
\newblock In \emph{Conference on Learning Theory}, pp.\  2109--2136. PMLR, 2020.

\bibitem[Jurafsky \& Martin(2023)Jurafsky and Martin]{modern_nlp_book}
Daniel Jurafsky and James~H. Martin.
\newblock \emph{Speech and Language Processing}.
\newblock Prentice Hall, 3 edition, 2023.
\newblock Draft.

\bibitem[Levy \& Goldberg(2014)Levy and Goldberg]{levy}
Omer Levy and Yoav Goldberg.
\newblock Neural word embedding as implicit matrix factorization.
\newblock \emph{Advances in neural information processing systems}, 27, 2014.

\bibitem[Li et~al.(2020)Li, Zhou, He, Wang, Yang, and Li]{li2020sentence}
Bohan Li, Hao Zhou, Junxian He, Mingxuan Wang, Yiming Yang, and Lei Li.
\newblock On the sentence embeddings from pre-trained language models.
\newblock \emph{arXiv preprint arXiv:2011.05864}, 2020.

\bibitem[Li et~al.(2023)Li, Wang, Li, and Qu]{multilabel_nc}
Pengyu Li, Yutong Wang, Xiao Li, and Qing Qu.
\newblock Neural collapse in multi-label learning with pick-all-label loss.
\newblock \emph{arXiv preprint arXiv:2310.15903}, 2023.

\bibitem[Liu et~al.(2023)Liu, Xie, Li, and Ma]{liu2023same}
Hong Liu, Sang~Michael Xie, Zhiyuan Li, and Tengyu Ma.
\newblock Same pre-training loss, better downstream: Implicit bias matters for language models.
\newblock In \emph{International Conference on Machine Learning}, pp.\  22188--22214. PMLR, 2023.

\bibitem[Lyu \& Li(2020)Lyu and Li]{lyu2020Gradient}
Kaifeng Lyu and Jian Li.
\newblock Gradient descent maximizes the margin of homogeneous neural networks.
\newblock In \emph{International Conference on Learning Representations}, 2020.

\bibitem[Madden et~al.(2024)Madden, Fox, and Thrampoulidis]{madden2024upper}
Liam Madden, Curtis Fox, and Christos Thrampoulidis.
\newblock Upper and lower memory capacity bounds of transformers for next-token prediction.
\newblock \emph{arXiv preprint arXiv:2405.13718}, 2024.

\bibitem[Mikolov et~al.(2013{\natexlab{a}})Mikolov, Chen, Corrado, and Dean]{word2vec_2}
Tomas Mikolov, Kai Chen, Greg Corrado, and Jeffrey Dean.
\newblock Efficient estimation of word representations in vector space.
\newblock \emph{arXiv preprint arXiv:1301.3781}, 2013{\natexlab{a}}.

\bibitem[Mikolov et~al.(2013{\natexlab{b}})Mikolov, Sutskever, Chen, Corrado, and Dean]{word2vec_1}
Tomas Mikolov, Ilya Sutskever, Kai Chen, Greg~S Corrado, and Jeff Dean.
\newblock Distributed representations of words and phrases and their compositionality.
\newblock \emph{Advances in neural information processing systems}, 26, 2013{\natexlab{b}}.

\bibitem[Mikolov et~al.(2013{\natexlab{c}})Mikolov, Yih, and Zweig]{mikolov2013linguistic}
Tom{\'a}{\v{s}} Mikolov, Wen-tau Yih, and Geoffrey Zweig.
\newblock Linguistic regularities in continuous space word representations.
\newblock In \emph{Proceedings of the 2013 conference of the north american chapter of the association for computational linguistics: Human language technologies}, pp.\  746--751, 2013{\natexlab{c}}.

\bibitem[Mixon et~al.(2020)Mixon, Parshall, and Pi]{mixon2020neural}
Dustin~G Mixon, Hans Parshall, and Jianzong Pi.
\newblock Neural collapse with unconstrained features.
\newblock \emph{arXiv preprint arXiv:2011.11619}, 2020.

\bibitem[Mu et~al.(2017)Mu, Bhat, and Viswanath]{mu2017all}
Jiaqi Mu, Suma Bhat, and Pramod Viswanath.
\newblock All-but-the-top: Simple and effective postprocessing for word representations.
\newblock \emph{arXiv preprint arXiv:1702.01417}, 2017.

\bibitem[Murty et~al.(2023)Murty, Sharma, Andreas, and Manning]{murty2023grokking}
Shikhar Murty, Pratyusha Sharma, Jacob Andreas, and Christopher~D Manning.
\newblock Grokking of hierarchical structure in vanilla transformers.
\newblock \emph{arXiv preprint arXiv:2305.18741}, 2023.

\bibitem[Nacson et~al.(2019)Nacson, Lee, Gunasekar, Savarese, Srebro, and Soudry]{nacson2019convergence}
Mor~Shpigel Nacson, Jason Lee, Suriya Gunasekar, Pedro Henrique~Pamplona Savarese, Nathan Srebro, and Daniel Soudry.
\newblock Convergence of gradient descent on separable data.
\newblock In \emph{The 22nd International Conference on Artificial Intelligence and Statistics}, pp.\  3420--3428. PMLR, 2019.

\bibitem[Papyan et~al.(2020)Papyan, Han, and Donoho]{NC}
Vardan Papyan, XY~Han, and David~L Donoho.
\newblock Prevalence of neural collapse during the terminal phase of deep learning training.
\newblock \emph{Proceedings of the National Academy of Sciences}, 117\penalty0 (40):\penalty0 24652--24663, 2020.

\bibitem[Pennington et~al.(2014)Pennington, Socher, and Manning]{glove}
Jeffrey Pennington, Richard Socher, and Christopher~D Manning.
\newblock Glove: Global vectors for word representation.
\newblock In \emph{Proceedings of the 2014 conference on empirical methods in natural language processing (EMNLP)}, pp.\  1532--1543, 2014.

\bibitem[Pesme et~al.(2021)Pesme, Pillaud-Vivien, and Flammarion]{pesme2021implicit}
Scott Pesme, Loucas Pillaud-Vivien, and Nicolas Flammarion.
\newblock Implicit bias of sgd for diagonal linear networks: a provable benefit of stochasticity.
\newblock \emph{Advances in Neural Information Processing Systems}, 34:\penalty0 29218--29230, 2021.

\bibitem[Power et~al.(2022)Power, Burda, Edwards, Babuschkin, and Misra]{power2022grokking}
Alethea Power, Yuri Burda, Harri Edwards, Igor Babuschkin, and Vedant Misra.
\newblock Grokking: Generalization beyond overfitting on small algorithmic datasets.
\newblock \emph{arXiv preprint arXiv:2201.02177}, 2022.

\bibitem[Radford et~al.(2018)Radford, Narasimhan, Salimans, Sutskever, et~al.]{radford2018improving}
Alec Radford, Karthik Narasimhan, Tim Salimans, Ilya Sutskever, et~al.
\newblock Improving language understanding by generative pre-training.
\newblock 2018.

\bibitem[Radford et~al.(2019)Radford, Wu, Child, Luan, Amodei, Sutskever, et~al.]{radford2019language}
Alec Radford, Jeffrey Wu, Rewon Child, David Luan, Dario Amodei, Ilya Sutskever, et~al.
\newblock Language models are unsupervised multitask learners.
\newblock \emph{OpenAI blog}, 1\penalty0 (8):\penalty0 9, 2019.

\bibitem[Recht et~al.(2010)Recht, Fazel, and Parrilo]{RechtFazel}
Benjamin Recht, Maryam Fazel, and Pablo~A Parrilo.
\newblock Guaranteed minimum-rank solutions of linear matrix equations via nuclear norm minimization.
\newblock \emph{SIAM review}, 52\penalty0 (3):\penalty0 471--501, 2010.

\bibitem[Rosset et~al.(2003)Rosset, Zhu, and Hastie]{rosset2003margin}
Saharon Rosset, Ji~Zhu, and Trevor Hastie.
\newblock Margin maximizing loss functions.
\newblock In \emph{NIPS}, pp.\  1237--1244, 2003.

\bibitem[Saunshi et~al.(2020)Saunshi, Malladi, and Arora]{saunshi2020mathematical}
Nikunj Saunshi, Sadhika Malladi, and Sanjeev Arora.
\newblock A mathematical exploration of why language models help solve downstream tasks.
\newblock \emph{arXiv preprint arXiv:2010.03648}, 2020.

\bibitem[Sch{\"u}tze et~al.(2008)Sch{\"u}tze, Manning, and Raghavan]{schutze2008introduction}
Hinrich Sch{\"u}tze, Christopher~D Manning, and Prabhakar Raghavan.
\newblock \emph{Introduction to information retrieval}, volume~39.
\newblock Cambridge University Press Cambridge, 2008.

\bibitem[Shannon(1948)]{shannon1948mathematical}
Claude~Elwood Shannon.
\newblock A mathematical theory of communication.
\newblock \emph{The Bell system technical journal}, 27\penalty0 (3):\penalty0 379--423, 1948.

\bibitem[Soudry et~al.(2018)Soudry, Hoffer, Nacson, Gunasekar, and Srebro]{soudry2018implicit}
Daniel Soudry, Elad Hoffer, Mor~Shpigel Nacson, Suriya Gunasekar, and Nathan Srebro.
\newblock The implicit bias of gradient descent on separable data.
\newblock \emph{The Journal of Machine Learning Research}, 19\penalty0 (1):\penalty0 2822--2878, 2018.

\bibitem[Srebro et~al.(2004)Srebro, Rennie, and Jaakkola]{srebro2004maximum}
Nathan Srebro, Jason Rennie, and Tommi Jaakkola.
\newblock Maximum-margin matrix factorization.
\newblock \emph{Advances in neural information processing systems}, 17, 2004.

\bibitem[Sun et~al.(2022)Sun, Ahn, Thrampoulidis, and Azizan]{sun2022mirror}
Haoyuan Sun, Kwangjun Ahn, Christos Thrampoulidis, and Navid Azizan.
\newblock Mirror descent maximizes generalized margin and can be implemented efficiently.
\newblock \emph{Advances in Neural Information Processing Systems}, 35:\penalty0 31089--31101, 2022.

\bibitem[Tarzanagh et~al.(2023{\natexlab{a}})Tarzanagh, Li, Thrampoulidis, and Oymak]{tarzanagh2023transformers}
Davoud~Ataee Tarzanagh, Yingcong Li, Christos Thrampoulidis, and Samet Oymak.
\newblock Transformers as support vector machines, 2023{\natexlab{a}}.

\bibitem[Tarzanagh et~al.(2023{\natexlab{b}})Tarzanagh, Li, Zhang, and Oymak]{tarzanagh2023maxmargin}
Davoud~Ataee Tarzanagh, Yingcong Li, Xuechen Zhang, and Samet Oymak.
\newblock Max-margin token selection in attention mechanism, 2023{\natexlab{b}}.

\bibitem[Thrampoulidis(2024)]{ntp}
Christos Thrampoulidis.
\newblock Implicit bias of next-token prediction.
\newblock \emph{arXiv preprint arXiv:2402.18551}, 2024.

\bibitem[Thrampoulidis et~al.(2022)Thrampoulidis, Kini, Vakilian, and Behnia]{seli}
Christos Thrampoulidis, Ganesh~R Kini, Vala Vakilian, and Tina Behnia.
\newblock Imbalance trouble: Revisiting neural-collapse geometry.
\newblock \emph{arXiv preprint arXiv:2208.05512}, 2022.

\bibitem[Timkey \& Van~Schijndel(2021)Timkey and Van~Schijndel]{timkey2021all}
William Timkey and Marten Van~Schijndel.
\newblock All bark and no bite: Rogue dimensions in transformer language models obscure representational quality.
\newblock \emph{arXiv preprint arXiv:2109.04404}, 2021.

\bibitem[Turian et~al.(2010)Turian, Ratinov, and Bengio]{neural_language_models_1}
Joseph Turian, Lev Ratinov, and Yoshua Bengio.
\newblock Word representations: a simple and general method for semi-supervised learning.
\newblock In \emph{Proceedings of the 48th annual meeting of the association for computational linguistics}, pp.\  384--394, 2010.

\bibitem[Vardi \& Shamir(2021)Vardi and Shamir]{vardi2021implicit}
Gal Vardi and Ohad Shamir.
\newblock Implicit regularization in relu networks with the square loss.
\newblock In \emph{Conference on Learning Theory}, pp.\  4224--4258. PMLR, 2021.

\bibitem[Wang et~al.(2004)Wang, Bovik, Sheikh, and Simoncelli]{wang2004image}
Zhou Wang, Alan~C Bovik, Hamid~R Sheikh, and Eero~P Simoncelli.
\newblock Image quality assessment: from error visibility to structural similarity.
\newblock \emph{IEEE transactions on image processing}, 13\penalty0 (4):\penalty0 600--612, 2004.

\bibitem[Wu \& Papyan(2024)Wu and Papyan]{wu2024linguistic}
Robert Wu and Vardan Papyan.
\newblock Linguistic collapse: Neural collapse in (large) language models.
\newblock \emph{arXiv preprint arXiv:2405.17767}, 2024.

\bibitem[Yang et~al.(2017)Yang, Dai, Salakhutdinov, and Cohen]{yang2017breaking}
Zhilin Yang, Zihang Dai, Ruslan Salakhutdinov, and William~W Cohen.
\newblock Breaking the softmax bottleneck: A high-rank rnn language model.
\newblock \emph{arXiv preprint arXiv:1711.03953}, 2017.

\bibitem[Zhu et~al.(2021)Zhu, Ding, Zhou, Li, You, Sulam, and Qu]{zhu2021geometric}
Zhihui Zhu, Tianyu Ding, Jinxin Zhou, Xiao Li, Chong You, Jeremias Sulam, and Qing Qu.
\newblock A geometric analysis of neural collapse with unconstrained features.
\newblock \emph{Advances in Neural Information Processing Systems}, 34, 2021.

\end{thebibliography}
\bibliographystyle{colm2024_conference}

%%%%%%%%%%%%%%%%%%%%%%%%%%%%%%%%%%%%%%%%%%%%%%%%%%%%%%%%%%%%%%%%%%%%%%%%%%%%%%%
%%%%%%%%%%%%%%%%%%%%%%%%%%%%%%%%%%%%%%%%%%%%%%%%%%%%%%%%%%%%%%%%%%%%%%%%%%%%%%%
% APPENDIX
%%%%%%%%%%%%%%%%%%%%%%%%%%%%%%%%%%%%%%%%%%%%%%%%%%%%%%%%%%%%%%%%%%%%%%%%%%%%%%%
%%%%%%%%%%%%%%%%%%%%%%%%%%%%%%%%%%%%%%%%%%%%%%%%%%%%%%%%%%%%%%%%%%%%%%%%%%%%%%%

\clearpage % Start appendices on a new page (optional)
\tableofcontents
\appendix
% \addcontentsline{toc}{section}{Appendices} % Optional: Add "Appendices" to the main ToC
% \startcontents[appendices] % Start accumulating content for appendices ToC
% \printcontents[appendices]{l}{1}{\setcounter{tocdepth}{2}} % Print the ToC for appendices

\appendix
    \setcounter{tocdepth}{1} % Adjust this to include sections/subsections as needed
    \setcounter{section}{0} % Reset section numbering
% \section{Appendix}

\onecolumn

% \section{Notations}\label{app:notations}
% \input{Arxiv_NewTempl/sections/notations}
% \section{Notations}\label{sec:notation}
% \input{Arxiv_NewTempl/sections/notations}
\section{Related work: Detailed version}\label{app:related}
This section expands upon the section on related work in the main body.

\subsection{Word embeddings as matrix factorization}\label{sec:lit_mf}
% \noindent\textbf{Word embeddings as matrix factorization.}~
Neural probabilistic language models are designed to predict the probability distribution of a target word $v\in\Vc$ within a given vocabulary 
$\Vc$, based on a sequence of preceding or surrounding words of (variable or fixed) length $t$ known as the context $\z\in\Vc^t$. In these models, word and context representations (aka embeddings) are parameterized by their internal architecture \citep{neural_language_models_2}. The emergence of energy-based models like the Word2Vec family, particularly through the introduction of training methodologies such as the Skip-Gram with Negative Sampling (SGNS) objective by \citet{word2vec_1,word2vec_2}, marked a significant advancement in learning high quality embeddings; e.g., see \citet{modern_nlp_book}. SGNS aims to optimize the alignment between vectors of \emph{target} words and the \emph{context} words surrounding the target, while disassociating those of randomly drawn pairs. This is achieved by optimizing a log-bilinear model that learns unconstrained embeddings for words in two roles: as targets and as contexts. While intuitive, the objective remained somewhat mysterious until the seminal analysis by \citet{levy} offered a formal examination of embeddings generated by the SGNS objective, framing it as a form of weighted matrix factorization. Specifically for embeddings of large dimensionality, Levy and Goldberg demonstrated that SGNS implicitly factorizes a shifted version of the Pointwise Mutual Information (PMI) matrix, denoted as 
$\PMI$, which is a measure of association between words and contexts by having entries  $\PMI[z,j]=\log\frac{\rm{Pr}(z|j)}{\rm{Pr}(z)}=\log\frac{\ph_{j,z}}{\rm{Pr}(z)}$ for target word $z$ and context word $j$. This result establishes a fundamental connection between count-based and predictive models \citep{glove}. However, this approach presumes non-zero occurrences of all context-word pairs in the training set, a condition not always met, leading to undefined factorization for pairs with zero co-occurrence (resulting in 
$\PMI[z,j]=-\infty$). To address the possibility of negative infinity values, an approximation involving a positive, sparse, and non-negative PMI (PPMI) matrix was proposed. While this adaptation offered an intriguing alternative to the original training objectives, the geometric characteristics of the resultant embeddings in comparison to those derived from the original loss remain unclear. 

Our approach mirrors that of \citet{levy} to some extent (see also \citet{glove}), especially in treating NTP as a soft-label classification across distinct contexts and presuming unconstrained context embeddings (although these concepts are not clearly mentioned in those original works). At a very basic level our work distinguishes itself by focusing on the NTP objective, which, although akin to SGNS and NC objectives, employs cross-entropy loss instead of sigmoid functions. Aside from this distinction, we identify the following conceptual divergences, contributing to our novel perspective on this line of inquiry: \textbf{(i)} Firstly, we aim to characterize the geometry of embeddings as learned directly from the loss function, rather than through approximations. 
\textbf{(ii)} Secondly, although \ref{eq:ufm}, as a log-bilinear model, is in spirit to the log-bilinear models trained by Word2Vec, we only use such log-bilinear models as an analytically-tractable proxy to expressive deep neural probabilistic language models such as transformers.
\textbf{(iii)} Thirdly, we confront the sparsity in the probabilistic labels of each context head-on, unlike other matrix-factorization methods that circumvent sparsity in a heuristic manner through smoothing or weighting techniques \citep{levy, glove}. This helps us identify that sparsity leads to diverging embeddings, a scenario where setting the loss gradient to zero is infeasible with finite weights. 
\textbf{(iv)} Fourthly, by examining embeddings through the lens of the regularization path—a surrogate for gradient descent optimization—we unveil that word and context embeddings emerge from the factorization of a matrix 
$\Lbin+R\Lbmm$. Here, 
$\Lbin$ captures the frequency of word-context occurrences akin to the PMI, setting non-occurring pairs to zero while retaining positive values for occurring pairs, much like the PPMI matrix. The directional component 
$\Lbmm$, becoming dominant as weights diverge, reflects the explicit sparsity patterns of the context-word probability matrix. 

The above-mentioned conceptual novelties can also be directly applied to the SGNS objective generalizing the result of \citet{levy} to sparse language patterns. We show this in App.~\ref{sec:word2vec}.
% \new{In other words, our analysis uses UFM to turn a neural language model into a log-bilinear model such as \cite{word2vec_1, word2vec_2, glove, mnih2013learning} suggesting NTP's optimal context and word embeddings also arise from matrix-factorization. However, here the factorized matrix does not explicitly depend on the word-word co-occurence counts but rather implicitly depends on these counts, i.e., the sparse soft-labels (context-word co-occurence), as encoded in $\Lbmm$ and $\Lbin$.}

The seminal contribution of \citet{levy} has inspired numerous subsequent studies, leveraging the geometric insights of embeddings to uncover linguistic phenomena such as the linear relationships underlying word analogies \citep{allen2019analogies,hashimoto2016word,arora2016latent,ethayarajh2018towards,mikolov2013linguistic}. 
Other works, investigate the \emph{anisotropy} of the learned word/context embeddings for both static models such as Word2Vec \citep{mu2017all} and contextualized models such as pre-trained BERT and GPT-2 \citep{ethayarajh2019contextual, li2020sentence}. This property of the embeddings has been attributed to the imbalanced long-tail distribution of language \citep{gao2019degeneration,li2020sentence}, shifts in the embedding space \citep{bis2021too}, and characteristic of the network architecture \citep{godey2024anisotropy}. Various techniques have been proposed to improve the isotropy of the embeddings, which yields embeddings with better language regularities (e.g., \citep{mu2017all,gao2019degeneration,bis2021too,timkey2021all}).
We envision that our fresh perspective on geometric analysis and our findings could similarly motivate further research in this domain, extending the understanding and application of language model embeddings.

Finally, we note that a \emph{sparse plus low-rank} structure for parameters learned with weighted exponential language models trained on fixed n-gram word/context features has been observed in \citet{hutchinson2012sparse}. Through experiments, the authors decompose the weights learned by optimizing a weighted maximum entropy objective into sparse and low-rank components. They demonstrate that explicitly incorporating sparsity and low-rank constraints into their weighted maximum entropy framework can enhance both performance and interpretability \citep{hutchinson2012sparse, hutchinson2013exceptions}, particularly when training data is limited \citep{hutchinson2015sparse}. However, these components serve different roles compared to our work. In their framework, the low-rank component captures frequent word co-occurrences, while the sparse component records infrequent words and special lexical items. In contrast, our analysis shows via the regularization path framework that the NTP objective produces a low-rank component from a nuclear-norm max-margin problem between in-support and out-support tokens, with the sparse component encoding the frequency of the in-support tokens.\looseness=-1

\subsection{Overparameterization and Implicit bias/regularization of GD}
% \noindent\textbf{Overparameterization and Implicit bias/regularization of GD.}
Recent studies on gradient-based optimization methods, such as SGD and its variants, have focused on overparameterized settings in empirical risk minimization without explicit regularization. This interest stems from observations that deep neural networks, especially in image classification, perform exceptionally well even with minimal or no weight decay. This research area, known as ``implicit bias'' or ``implicit regularization'' of GD emerged to explore which among the numerous potential minimizers in overparameterized settings are preferred by algorithms like GD. Notably, \citet{soudry2018implicit,ji2018risk} showed that for linear logistic regression with one-hot labels, GD aligns with the hard-margin SVM solution, essentially becoming the max-margin classifier. Additionally, \citet{ji2020directional} indicated that GD's optimization path in linear models is analogous to the regularization path, which looks at risk minimizers as regularization diminishes. This path is analytically more accessible and offers a useful framework to study GD's behavior.  In our work, we apply this idea to investigate the regularization path of \NTP training.
These foundational results on implicit bias have spurred further research into stochastic and adaptive gradient methods \citep{nacson2019convergence,pesme2021implicit,sun2022mirror,azizan2021stochastic,cattaneo2023implicit}, and more intricate architectures \citep{lyu2020Gradient,ji2020directional,gunasekar2018characterizing,gunasekar2018implicit}) including recent studies on transformers \citep{tarzanagh2023maxmargin,tarzanagh2023transformers}; for a review, refer to \citet{vardi2021implicit}. These insights have also paved the way for generalization analysis, linking GD's generalization to that of SVM-like solutions. This connection has been explored in high-dimensional contexts where traditional margin-based bounds are inadequate, revealing scenarios where no regularization yields optimal outcomes, a phenomenon known as 'harmless interpolation' or 'benign overfitting.' Most of these findings, whether on optimization or generalization, are confined to the one-hot classification framework. 

In the language modeling literature, several theoretical works suggest that achieving the empirical entropy lower-bound at the pre-training stage results in better down-stream performance \citep{liu2023same, saunshi2020mathematical}. \citet{liu2023same} attributes this better down-stream performance to the implicit bias of the pre-training algorithms, by showing that the optimizers prefer the more transferable models among the ones that achieve the same minimal loss value.
A formal deviation from one-hot to soft-label classification for \NTP is the work by \citep{ntp}, which partially inspired our research. While their analysis, assuming fixed context embeddings, applies to linear models, our work extends this by optimizing both word and context embeddings, leading to a nonlinear model. We discover that this nonlinearity yields a preference towards minimizing the nuclear norm of the logits over the Euclidean norm of the word embedding matrix, as seen in \citep{ntp}. Our theoretical contributions thus expand on their findings into the nonconvex domain. %Our proof strategies align with the underlying principles found in related implicit bias research \citep{rosset2003margin,ji2020directional}, though we remark the necessity of addressing specific technical nuances in our application of these concepts.\looseness=-1

\subsection{Neural collapse geometry}
% \vspace{-5pt}
% \noindent\textbf{Neural-collapse geometry.}
A third area of research closely related to our work is the investigation of neural-collapse (NC) geometry, which seeks to understand the geometry of last-layer features and classifier weights in deep networks optimized in the interpolating regime. Initiated by \citet{NC}, this line of research uncovered a universal phenomenon across various architectures and datasets, where classifiers and last-layer features consistently converge to an equiangular tight frame (ETF). This convergence subsumes  the aggregation of same-class features around their class mean, a phenomenon termed NC. The theoretical underpinnings for these empirical findings were later provided by analyzing the unconstrained features model (UFM) \citep{mixon2020neural,fang2021exploring,zhu2021geometric}, suggesting that highly overparameterized networks can produce arbitrary last-layer features. This concept, reminiscent of the early framework used in Word2Vec models \citep{levy} as discussed above (see also \cite{yang2017breaking}), has since spurred extensive research in diverse contexts.

Among the vast body of work on NC geometry--an exhaustive reference of which is beyond our scope--, two studies stand out for their relevance to our research. \citet{seli} first adapted the UFM to describe geometries diverging from symmetric ETF by moving away from the balanced data assumption, demonstrating that logits optimizing the UFM with vanishing regularization align with a nuclear-norm SVM problem; see also \cite{behnia2023implicit}. This approach also bridges NC geometry with the implicit bias literature previously discussed. Our work extends their principles to the \NTP objective, differing from the one-hot encoding classification focus of their study. We take a similar approach, but this time applied to the \NTP objective. This is unlike their result which only applies to one-hot encoding classification. In fact, our formulation in \ref{eq:ufm-svm} extends their findings, encompassing them as a specific instance when each context is followed by only one token—a scenario less applicable to the linguistic contexts we focus on. Moreover, we demonstrate (see Theorem~\ref{thm:reg path finite}) that the \NTP framework introduces a finite projection component to the regularization path, a feature absent in one-hot settings.

Our findings also resonate to certain extent with the recent multilabel NC geometry study in \citet{multilabel_nc}, where each example, typically an image, is tagged with multiple correct labels, represented as a $k$-hot vector. In contrast, our \NTP framework, modeled as soft-label classification, associates each context with a probability vector. This distinction makes our framework more versatile, applicable beyond the restrictive case where soft-label classification aligns with multi-label classification—specifically, when all relevant tokens are equally likely. There are also crucial differences in the outcomes. Their analysis assumes an equal number of training samples of multiplicity $k=1$ for each category, a condition that, in our context, would necessitate every vocabulary word to be followed by an identical number of contexts that are only followed by that word—a significantly limiting constraint. Our findings, however, are not bound by this assumption. %\tina{discuss mixup a bit?}

In an attempt to extend this line of work to language setups, contemporaneous work \cite{wu2024linguistic}, through experiments, suggests a correlation between the degree of convergence to the NC geometry and the validation loss for different language models trained on TinyStories for a few epochs. Different to us, they implicitly treat the NTP objective as one-hot classification and suggest that better \emph{per-class} (i.e., next-token) collapse of the context embeddings leads to improved validation loss. In contrast, our analysis reveals a more intricate finding: within a sparse soft-label formulation of NTP training with language, it is not \emph{per-class} collapse that minimizes the training objective. Instead, the objective is minimized when context embeddings with the same \emph{support set} collapse after being projected onto \emph{subspace} $\Ts_\perp$, which is defined by the sparsity pattern of the context-word co-occurrence matrix (see Sec.~\ref{sec:embeddings geo from logits}). Our approach also comes with an analysis framework rather than relying solely on experiments.

\subsection{Connections to Word2Vec}\label{sec:word2vec}
We apply our geometry analysis to Word2Vec. Consider the following Skip-gram objective: 
    \[
    \frac{1}{T}\sum_{t\in[T]}\sum_{-c\leq j\leq c, j\neq0} \log p\big(z_{t+j}|z_{t}\big)=\frac{1}{T}\sum_{t\in[T]}\sum_{-c\leq j\leq c, j\neq0} \log p\big(\sfti{z_{t+j}}{\Hb\w_{z_t}}\big)\,,
    \]
    where $\Hb\in\R^{V\times d}$ denotes the embedding matrix of \emph{context} words $z_{t+j}, j\neq 0$ and $\W\in\R^{V\times d}$ denotes the embedding matrix of \emph{target words} $z_t$.
    As presented in \cite{mikolov2013linguistic}, skip-gram actually implements a heuristic of this objective that aims to simplify the computation of the normalizing factor in the softmax. Second, consider the CBOW objective:
    \[
    \frac{1}{T}\sum_{t\in[T]} \log p\big(z_{t}|\{z_{t+j}\}_{\{|j|\in[c],j\neq 0\}}\big)=\frac{1}{T}\sum_{t\in[T]} \log p\left(\sfti{z_{t}}{\W\cdot\frac{1}{2c-1}\sum_{-c\leq j\leq c, j\neq0}\hb_{z_{t+j}}}\right)\,,
    \]
    where $\frac{1}{2c-1}\sum_{-c\leq j\leq c, j\neq0}\hb_{z_{t+j}}$ is the continuous-bags-of-word embedding of the context $\{z_{t+j}\}_{\{|j|\in[c],j\neq 0\}}$. 
    % \textbf{Note that in the \NTP setting that we consider, there is \emph{no} explicit parameterization of the embedding of context in terms of the embeddings of its consituent words.} 

    Suppose the CBOW model with context consisting of only the previous word. To keep notation consistent with the rest of the paper, let $j\in[m]=[V]$ denote index of \emph{distinct} contexts (aka previous words) and $z\in[V]$ denote index of (next) words. Note the embedding matrix $\Hb\in\R^{d\times V}$ is now of same dimension as $\W^\top$. Let $\pih_j$ denote the marginal of context words and $\ph_{j,z}$ denote the conditional probability of word $z$ following context word $j$. Then, the CBOW objective becomes equivalent to
    \[
\sum_{j\in[V]}\pih_j\sum_{z\in[V]}\ph_{j,z}\log\left(\Sc_{z}\left(\W\hb_j\right)\right)\,.
    \]
    This now fits exactly our framework and the results for NTP apply mutatis mutandis. 

    In \cite{mikolov2013linguistic}, instead of minimizing the CE objective, they actually use the Negative-Sampling (SGNS) objective. \citet{levy} rewrites the SGNS objective as follows:
    \[
    \sum_{j\in[m]}\sum_{z\in[V]} -\Pr(j,z) \log\left(\sigma(\w_z^\top\hb_j)\right)-k\Pr(j)\Pr(w)\log\left(1-\sigma(\w_z^\top\hb_j)\right)\,,
    \]
    where $\Pr(j,z):=\pih_j\ph_{j,z}$, $\Pr(j):=\pih_j$, $\Pr(z)=\sum_{j\in[m]}\pih_j\ph_{j,z}$and $k\in\mathbb{N}$ is a hyperparameter. Defining 
\[q_{j,z}:=\frac{\frac{\Pr(j,z)}{\Pr(j)\Pr(z)}}{k+\frac{\Pr(j,z)}{\Pr(j)\Pr(z)}}\in[0,1]\,,\]
we can rewrite the SGNS objective as 
\[
    \sum_{j\in[m]}\sum_{z\in[V]} \Pr(j)\Pr(w)\left(k+\frac{\Pr(j,z)}{\Pr(j)\Pr(z)}\right)\left(-q_{j,z}  \log\left(\sigma(\w_z^\top\hb_j)\right)-(1-q_{j,z})\log\left(1-\sigma(\w_z^\top\hb_j)\right)\right)\,.
    \]
    For each $(j,z)$-pair, the rightmost term in the expression above is lower bounded by $\ent(q_{j,z}):=-q_{j,z}\log(q_{j,z}) - (1-q_{j,z})\log(1-q_{j,z}).$ 

    By following the same arguments developed for \NTP, it can be checked that the lower bound is attained if and only if the following conditions hold:
    \begin{enumerate}
        \item There exists logit matrix $\Lb_1\in\R^{V\times m}$ (of rank $\leq d$) such that 
        \[
        \Lb_1[z,j]=\text{PMI}(z,j)-\log(k) = \log\left(\frac{\Pr(j,z)}{\Pr(j)\Pr(z)}\right) - \log(k),\quad \forall j\in[m], z\in\Sc_j. \].
        \item 
        There exists logit matrix $\Lb_2\in\R^{V\times m}$ (of rank $\leq d$) such that 
        \begin{subequations}
        \begin{align*}
        &\Lb_2[z,j]=0,\quad \forall j\in[m], z\in\Sc_j,
        \\
        &\Lb_2[v,j]<0,\quad \forall j\in[m], v\notin\Sc_j.
        \end{align*}
        \end{subequations}
    \end{enumerate}

    The logit matrix $\Lb_1$ is the shifted PMI matrix as also characterized in \citet[Eq.~(7)]{levy}. This component plays a similar role to the finite component $\Lbin$ in the NTP setup of our paper. The second component, which is only non-zero if there exists a target-context pair $(z,j)$ with zero co-occurrence in the training set ($\Pr(j,z)=0)$, specifies the directional component, akin to the role of $\Lbmm$ in NTP. This component is absent in the analysis of \citet{levy}, as they implicitly assume no zero entries in the co-occurence matrix when setting the derivative of the loss to zero. However, when the co-occurrence matrix is sparse, zero derivation of the loss is not achieved at any finite solution. Instead, $\Lb_2$ identifies the infinite direction that drives the loss to its lower-bound. \citet{levy} address sparsity indirectly by discussing alternative matrix factorizations such as shifted positive-PMI (PPMI). 
% \section{Additional results}\label{sec:collapse_app}
% \input{Arxiv_NewTempl/sections/collapse_app}
\section{Proofs}
%%%%%%%%%%%%%%%%%%%%%%%%%%%%%%%%%%%%%%%%%%%%%%%%
\subsection{Centering of word embeddings}
%%%%%%%%%%%%%%%%%%%%%%%%%%%%%%%%%%%%%%%%%%%%%%%%
\begin{lemma}\label{lem:W centered}
    For any $\lambda>0$, any minimizer $\What$ of \ref{eq:ufm} satisfies $\ones_V^\top\What=0$. That is, word embeddings are centered.
\end{lemma}
\begin{proof}
    The proof is the same as \citet[Lem. E.1]{seli} and is thus omitted. 
\end{proof}

There is two consequences of Lemma \ref{lem:W centered}:
\begin{enumerate}
    \item The word embedding matrix $\W$ is centered motivating the use of the \emph{centered} support matrix $\smatbar$ (rather than simply $\smat$) in Proxy \ref{P proxy}.
    \item The constraint $\ones^\top\W=0$ can be added to \ref{eq:ufm} without changing its optimal set. Thus, the non-convex rank constraint can be ignored, making the problem convex, provided $d\geq V-1$. Finally, note this also implies the constraint $\ones^\top\Lb=0$ can be added without any change in Eq. \eqref{eq:ufm relax}. \end{enumerate}

% \begin{proof}
%     For the sake of contradiction, suppose minimizer $\W$ such that $\bar\w:=\ones^\top\W\neq \mathbf{0}$. Consider centered $\Vb:=\W-\ones_V^\top\bar\w$. Because of transnational invariance of softmax:
%     \[
%     \CE(\Vb\Hb)=\CE(\W\Hb+\ones_V^\top\bar\w\Hb)=\CE(\W)\,.
%     \]
%     On the other hand,
%     \[
%     \|\Vb\|^2 = \|\W\|^2
%     \]
% \end{proof}

%%%%%%%%%%%%%%%%%%%%%%%%%%%%%%%%%%%%%%%%%%%%%%%%
\subsection{Proof of Lemma \ref{lem:relax}}
%%%%%%%%%%%%%%%%%%%%%%%%%%%%%%%%%%%%%%%%%%%%%%%%

We prove a slightly stronger version of Lemma \ref{lem:relax} which assumes the additional constraint $\ones^\top\Lb=0$ in Eq. \ref{eq:ufm}; see Lemma \ref{lem:W centered}. The proof follows identically if we ignore this centering property.

As mentioned in the main body, the proof of the lemma essentially relies on the well-known fact that
\[
\|\Lb\|_* = \min_{\W\Hb=\Lb}\frac{1}{2}\|\W\|^2+\frac{1}{2}\|\Hb\|^2\,.
\]

Firstly, fix any $\Lb$ such that $r:=\rank{\Lb}\leq d$ and $\ones_V^\top\Lb=0$. Let $\Lb=\Ub\Sigmab\Vb^\top$ be its SVD with $\Sigmab\in\R^{r\times r}$. Choose
\[
\hat\W=\Ub\Sigmab^{1/2}\Rb\qquad\text{and}\qquad\hat\Hb=\Rb^\top\Sigmab^{1/2}\Vb^\top\,
\]
for some partial orthonormal matrix $\Rb\in\R^{r\times d}$ (i.e., $\Rb^\top\Rb=\Id_{d}$). It then holds clearly that 
\[
\hat\W\hat\Hb=\Lb\qquad\text{and}\qquad \|\hat\W\|^2+\|\hat\Hb\|^2 = \tr\left({\Ub\Sigmab\Ub^\top}\right) + \tr\left({\Vb\Sigmab\Vb^\top}\right) = 2\tr\left({\Sigmab}\right) = 2\|\Lb\|_*,
\]
Thus, 
\[
\CE(\Lb)+\la\|\Lb\|_*=\CE(\hat\W\hat\Hb)+\frac{\la}{2}\left(\|\hat\W\|^2+\|\hat\Hb\|^2\right)\,.
\]
Because we can apply this for any $\Lb$, it shows that\footnote{Here, use the fact that the minimizer of the LHS of \eqref{eq:lem1 lower bound} is attained because of coercivity of the nuclear norm.} 
\begin{align}\label{eq:lem1 lower bound}
\min_{\stackrel{\Lb\in\rankset_d}{\ones^\top\Lb=0}}
\CE(\Lb)+\la\|\Lb\|_* \geq \min_{\W,\Hb}\CE(\W\Hb)+\frac{\la}{2}\left(\|\W\|^2+\|\Hb\|^2\right)\,,
\end{align}
where we define $\rankset_d:=\{\Lb\in\R^{V\times m}\,:\,\rank{\Lb}\leq d\}$ to be the manifold of rank-constrained matrices.  \looseness=-1

Secondly, consider any $\W\in\R^{V\times d},\Hb\in\R^{d\times m}$ such that $\ones_V^\top\W=0$. Let $\Lb=\W\Hb$ and denote $\Lb=\Ub\Sigmab\Vb^\top$ its SVD. We then have $\ones^\top\Lb=0$ and 
\begin{align}
\CE(\Lb)+\la\|\Lb\|_* &=\CE(\W\Hb)+\la\,\tr\left({\Ub^\top\W\Hb\Vb}\right)\nn
\\
&\leq\CE(\W\Hb)+\frac{\la}{2}\left(\|\Ub^\top\W\|^2+\|\Hb\Vb\|^2\right)\nn
\\
&\leq\CE(\W\Hb)+\frac{\la}{2}\left(\|\W\|^2+\|\Hb\|^2\right),\label{eq:equalities}
\end{align}
where in the first line we used $\|\Lb\|_*=\tr\left(\Sigmab\right)=\tr\left(\Ub^\top\Lb\Vb\right)$, in the second line we used Cauchy-Schwartz, and, in the third line we used that $\|\Ub\|_2\leq1$, $\|\Vb\|_2\leq1$. 
Thus, this proves
\begin{align}\nn
 \min_{\W,\Hb}\CE(\W\Hb)+\frac{\la}{2}\left(\|\W\|^2+\|\Hb\|^2\right) &= \min_{\stackrel{\W,\Hb}{\ones^\top\W=0}}\CE(\W\Hb)+\frac{\la}{2}\left(\|\W\|^2+\|\Hb\|^2\right)
 \\
 &\geq \min_{\stackrel{\Lb\in\rankset_d}{\ones^\top\Lb=0}}\CE(\Lb)+\la\|\Lb\|_*\,, \label{eq:lem1 upper bound}
\end{align}
where the first equality follows by Lemma \ref{lem:W centered}.
Moreover, note that equalities in \eqref{eq:equalities} hold if and only if the following three conditions hold
\begin{align*}
    &\W\Hb=\Lb=\Ub\Sigmab\Vb^\top\implies\Ub^\top\W\Hb\Vb=\Sigmab\\
    &\Ub^\top\W=\Vb^\top\Hb^\top\\
    &\|\Ub^\top\W\|^2=\|\W\|^2, \|\Vb^\top\Hb^\top\|^2=\|\Hb\|^2\,.
\end{align*}
The first two conditions combined give
\[\Ub^\top\W\W^\top\Ub=\Vb^\top\Hb^\top\Hb\Vb=\Sigmab\,.\]
Thus, $\W\W^\top=\Ub\Sigmab\Ub^\top+\Ub_\perp\Sigmab_\perp\Ub_\perp^\top$ for some non-negative diagonal matrix $\Sigmab_\perp$ and $\Ub_\perp$ the orthogonal complement of $\Ub$. But, the third condition requires
\[
\tr\left(\Ub^\top\W\W^\top\Ub\right) = \tr\left(\W\W^\top\right) \implies \tr\left(\Sigmab\right) = \tr\left(\Sigmab\right) + \tr\left(\Sigmab_\perp \right)\implies\tr\left(\Sigmab_\perp \right)=0\implies \Sigmab_\perp =0.
\]
This gives 
\begin{align}
    \W\W^\top=\Ub\Sigmab\Ub^\top\label{eq:lem1 W}\,,
\end{align} 
and similarly
\begin{align}  \Hb^\top\Hb=\Vb\Sigmab\Vb^\top\label{eq:lem1 Hb}\,.
\end{align}

Combining Eqns. \eqref{eq:lem1 lower bound} and \eqref{eq:lem1 upper bound} proves that \begin{align}\nn
 \min_{\W,\Hb}\CE(\W\Hb)+\frac{\la}{2}\left(\|\W\|^2+\|\Hb\|^2\right) =\min_{\stackrel{\Lb\in\rankset_d}{\ones^\top\Lb=0}}\CE(\Lb)+\la\|\Lb\|_*\,.
\end{align}

Moreover, the requirements for equalities in Eq. \eqref{eq:equalities} established in Eqs. \eqref{eq:lem1 W} and \eqref{eq:lem1 Hb} show the desired about the form of gram matrices $\W_\la\W_\la^\top$ and $\Hb_\la\Hb_\la^\top$ as stated in the lemma.

\subsection{Proof of Theorems \ref{thm:reg path} and \ref{thm:reg path finite}}

The proof involves several components, and so, we proceed in steps. %Before everything, we restate the theorems here in a unifying manner as Theorem \ref{thm:reg path app} below.

\subsubsection{Notations recap}\label{app:notations recap}
We begin by recaping the essential notation used throughout this section.

Recall \ref{eq:ufm-svm}: The constraints in \Eqref{eq:svm equalities} require in-support logits to be equal to each other for every context. For fixed $j\in[m]$,  the set of equality constraints for all $z\neq z'\in\Sc_j$  is equivalent to the set of the same constraints for any anchor $z_j\in\Sc_j$ and $z'\neq z_j\in\Sc_j$. That is, there is an effective total of $S_j-1$ linearly independent constraints for each $j\in[m]$. 
\begin{definition}[$\Ebin{j}$]
    For each $j\in[m]$ fix an anchor index $z_j\in\Sc_j$ and define matrix  $\Ebin{j}\in\R^{(S_j-1)\times V}$ with $S_j-1$ independent rows $(\eb_{z_j}-\eb_z)^\top, z\neq z_j, z\in\Sc_j$.
\end{definition}
With this definition, we can rewrite \Eqref{eq:svm equalities} since
\[
\Ebin{j}\ellbb_j = \textbf{0}\,, \forall j\in[m] \,\,\Leftrightarrow\,\, \Lb=[\ellbb_1,\ldots,\ellbb_m] \text{ satisfies  \Eqref{eq:svm equalities}} 
\]
% Thus,  defining $\Ebin{j}\in\R^{(S_j-1)\times V}$ with $S_j-1$ independent rows $(\eb_{z_j}-\eb_z)^\top, z\neq z_j, z\in\Sc_j$, then
% \Eqref{eq:svm equalities} reduces to $\Ebin{j}\ellbb_j = \textbf{0}, j\in[m]$.

The inequality constraints in \Eqref{eq:svm inequalities}, impose the requirement that in-support logits must be greater than out-of-support logits. Considering \Eqref{eq:svm equalities}, we can again select any anchor $z_j \in \Sc_j$ and reformulate \Eqref{eq:svm inequalities} as $\Ebout{j}\ellbb_j \geq \textbf{1}, \, j\in[m]$.
Here, $\Ebout{j} \in \mathbb{R}^{(V-S_j) \times V}$ has $V-S_j$ independent rows $(\eb_{\z_j}-\eb_v)^\top, v \notin \Sc_j$. In total, there are  $V-1$ constraints per context in \ref{eq:ufm-svm}.\looseness=-1

With the above definition of the matrices $\Ebin{j}\in\R^{(S_j-1)\times V}, j\in[m]$, we can define the subspace $\Ts$ of $V\times m$ matrices that have their $j$-th column living on the range space $\Rc(\Ebin{j}^\top)$ for all $j\in[m]$:
% Recall the definition of the matrices $\Ebin{j}\in\R^{(S_j-1)\times V}, j\in{m}$ in Sec.~\ref{sec:ntp svm}. Further recall the subspace $\Ts$ in Eq. \eqref{eq:subspace}, that is, the subspace of $V\times m$ matrices having their $j$-th column living on the range space $\Rc(\Ebin{j}^\top)$ for all $j\in[m]$:
\begin{align}\notag
    \Ts:=\left\{\Lb\in\R^{V\times m}\,:\,
    \ellbb_j\in\Rc(\Ebin{j}^\top)\,,j\in[m]\,
    \right\}=\operatorname{span}\left(\left\{
    (\eb_z-\eb_{z'})^\top\,\ebt_j
    \,:\,z\neq z'\in\Sc_j\,,j\in[m]\,
    \right\}\right)\,.
\end{align}
Note that this set depends on the support matrix $\smat$, but we suppress the dependence to simplify notation. Let
$\Qcs:\R^{V\times m}\rightarrow\R^{V\times m}$ be the projection operator on this subspace. For concreteness, note that this can be explicitly defined as the matrix operator projecting each column $\ellbb_j$ of $\Lb$ onto $\Rc(\Ebin{j}^\top)$, i.e., \[\Qcs(\Lb)=\begin{bmatrix}
\Qcs(\Lb)_1 & \Qcs(\Lb)_2 & \cdots & \Qcs(\Lb)_m   
\end{bmatrix}\,\,\text{with}\,
\Qcs(\Lb)_j:=\Ebin{j}^\top(\Ebin{j}\Ebin{j}^\top)^{-1}\Ebin{j}\ellbb_j, j\in[m].
\]
 Note that $\Qcs$ is a projection operator since it is clearly linear and $\Qcs(\Qcs(\Lb))=\Qcs(\Lb)$. Also, it is clear that $\Qcs(\Lb)\in\Ts$ for all $\Lb\in\R^{V\times m}$. Operationally, note that $\Qcs(\Lb)$ has zero-entries on off-support elements. We will use this property indiscriminately throughout this section.

 Finally, define the  orthogonal complement to $\Ts$ as
\begin{align}
\Ts_\perp:=\left\{\Lb\in\R^{V\times m}\,:\,
     \ellbb_j\in\Nc(\Ebin{j})\,, j\in[m]\,
    \right\}\,=\left\{\Lb\in\R^{V\times m}\,:\,
    \Ebin{j} \ellbb_j=0\,, j\in[m]\,
    \right\}\,.
\end{align}
Clearly, $\R^{V\times m}=\Ts\oplus\Ts_\perp\,.$ Operationally, note that $\Lb\in\Ts_\perp$ implies that on-support elements of every column are identical, that is,
$\Lb[z,j]=\Lb[z',j]$ for all $z\in\Sc_j$ and $j\in[m]$. In other words, it holds  \begin{align}
    \label{eq:F perp operational}
\Lb\in\Ts_\perp \,\,\Longleftrightarrow \,\, \Lb \text{ satisfies \eqref{eq:svm equalities}},
\end{align}
which we will use without further explicit reference.

We also let $\Qcsperp$ denote the projection operator to $\Ts_\perp$.
%Also note that we can also define $\Ts_\perp$ such that \[\R^{V\times m}=\Ts\oplus\Ts_\perp\,,\] as follows:

\subsubsection{When does the NTP loss attain its entropy lower-bound?}
For word/context embeddings $\W,\Hb$ and logits $\Lb=\W\Hb$, recall the NTP loss is written as
\begin{align}\label{eq:CE_T}
    \CE(\W\Hb)= \CE(\Lb) = \sum_{j\in[m]}\pih_j\sum_{z\in\Sc_j}\ph_{j,z}\log\left(1+\sum_{z'\neq z\in\Sc_j}e^{-(\Lb[z,j]-\Lb[z',j])}+\sum_{v\not\in\Sc_j}e^{-(\Lb[z,j]-\Lb[v,j])}\right)\,,
\end{align}
and is lower bounded by the  empirical $T$-gram entropy  of the data \citep{shannon1948mathematical}, i.e., for all $\thetab'$:
$\CE(\W\Hb)\geq \ent:=\hat{\E}_{(\x,z)\sim\Tc_n}\left[-\log\left(\hat{p}(z|\x)\right)\right]\,.$

We now state explicit conditions on the logit matrix $\Lb$ for which the lower bound is attained by a softmax model $\Lb=\sft{\W\Hb}$ with unconstrained context embeddings $\Hb$. Due to sparsity of the conditional probability matrix $\Pbb$, the lower bound is \emph{not} attained for any
 finite logit matrix, but it can be reached asymptotically provided the conditions stated in the proposition below hold. 
 
 The statement below is a slight extension of \citet[Prop.~1]{ntp}.

 \begin{proposition}
 The NTP objective of a softmax model $\sft{\W\Hb}$ with $\W\in\R^{V\times d}, \Hb\in\R^{d\times m}$ can reach the empirical entropy  lower-bound if the following three conditions hold simultaneously
 \begin{enumerate}
     \item \textbf{[\NTPH-compatibility]} There exists logit matrix $\Lb_1\in\R^{V\times m}$ such that 
    \begin{align}\label{eq:NTP-comp}
        \forall j\in[m], z\neq z'\in\Sc_j\,:\, \Lb_1[j,z]-\Lb_1[j,z'] = \log\left(\frac{\ph_{j,z}}{\ph_{j,z'}}\right)\,.
    \end{align}
    \item \textbf{[\NTP-separability]} There exists logit matrix $\Lb_2\in\R^{V\times m}$ that is feasible in \ref{eq:ufm-svm}.
    \item \textbf{[Rank constraint]} For all arbitrarily large positive scalars $\rho$, 
    \[
    \rank{\Lb_1+\rho\cdot\Lb_2}\leq d.
    \]
 \end{enumerate}
 Concretely, if the above conditions hold, then for all arbitrarily large $\rho$, there exist $\W_\rho\in\R^{V\times d}$ and $\Hb_\rho\in\R^{d\times m}$ such that \[
\lim_{\rho\rightarrow\infty}\CE(\W_\rho\Hb_\rho)=\ent\,.
 \]
 \end{proposition}
\begin{proof}
    It suffices to prove that for $\Lb_\rho:=\Lb_1+\rho\cdot\Lb_2$:
    \begin{align}
\lim_{\rho\rightarrow\infty}\CE(\Lb_\rho)=\ent. \label{eq:interpolation logits}       
    \end{align}
To see that this is sufficient, note from condition \emph{C} of the theorem that for all arbitrarily large $\rho$, $\Lb_\rho$ admits a factorization $\Lb_\rho=\W_\rho\Hb_\rho$ where $\W_\rho\in\R^{V\times d}$ and $\Hb_\rho\in\R^{d\times m}$. Since this holds for all large $\rho$, Eq. \eqref{eq:interpolation logits} implies the desired.

Thus, we now prove Eq. \eqref{eq:interpolation logits}. For any $j\in[m]$, it is easy to check, using conditions \emph{A} and \emph{B} of the theorem, that the following two statements are true for all $\rho>0$:
\begin{enumerate}%[(a)]
    \item For all $z\neq z'\in\Sc_j$:
    \begin{align}
\Lb_\rho[z,j] - \Lb_\rho[z',j] &= \left(\Lb_1[z,j] - \Lb_1[z,j]\right)+ \rho\cdot\left(\Lb_2[z,j] - \Lb_2[z',j]\right)          \nn\\  
&=\log(\ph_{j,z}/\ph_{j,z'})\nn
    \end{align}
    \item For all $z\in\Sc_j,v\not\in\Sc_j$:
    \begin{align}
\Lb_\rho[z,j] - \Lb_\rho[v,j] &= \left(\Lb_1[z,j] - \Lb_1[v,j]\right)+ \rho\cdot\left(\Lb_2[z,j] - \Lb_2[v,j]\right)          \nn\\  
&\geq \left(\Lb_1[z,j] - \Lb_1[v,j]\right)+\rho \nn
\\  
&\geq -2\|\Lb_1\|_2+\rho. \nn
    \end{align}
\end{enumerate}
In the last inequality above, we used the loose bound $\|\Lb_1\|_2\geq \max_{j\in[m],z\in[V]}|\Lb_1[z,j]|$ for the spectral norm of $\Lb_1$.

Using these in Eq. \eqref{eq:CE_T} gives:
\begin{align}
    \CE(\Lb_\rho)&\leq \sum_{j\in[m]}\pih_j\sum_{z\in\Sc_j}\ph_{j,z}\log\Big(1+\sum_{z'\neq z\in\Sc_j}\frac{\ph_{j,z'}}{\ph_{j,z}}+V e^{-\rho}e^{-2\|\Lb_1\|_2}\Big)\nn
   \\ &=
\sum_{j\in[m]}\pih_j\sum_{z\in\Sc_j}\ph_{j,z}\log\Big(\frac{1}{\ph_{j,z}}+V e^{-\rho}e^{-2\|\Lb_1\|_2}\Big)\nn
\\ &=
\sum_{j\in[m]}\pih_j\sum_{z\in\Sc_j}\ph_{j,z}\log\Big(\frac{1}{\ph_{j,z}}\Big)+\sum_{j\in[m]}\pih_j\sum_{z\in\Sc_j}\ph_{j,z}\log\Big(1+\ph_{j,z}V e^{-\rho}e^{-2\|\Lb_1\|_2}\Big)\nn\,.
\\ &\leq
\ent+V e^{-2\|\Lb_1\|_2} e^{-\rho}\,.\label{eq:LR clean bound}
%\\ &\leq
%\sum_{j\in[m]}\pih_j\sum_{z\in\Sc_j}\ph_{j,z}\log\Big(\frac{1}{\ph_{j,z}}\Big) + \log\Big(V e^{-\rho}e^{-2\|\Lb_1\|_2}\Big)\nn
\end{align}
The last inequality uses $\log(1+x)\leq x, x>0$.

Clearly the bound above converges to $\ent=\sum_{j\in[m]}\pih_j\sum_{z\in\Sc_j}\ph_{j,z}\log\Big(\frac{1}{\ph_{j,z}}\Big)$ as $\rho\rightarrow\infty$. From this and $\CE(\Lb_\rho)\geq\ent$, the desired Eq. \eqref{eq:interpolation logits} follows. This completes the proof. 
\end{proof}

Note that the first two conditions of the theorem are easy to satisfy.

On the one hand, Eq. \eqref{eq:NTP-comp} always has a solution. For later use, we also note that the solution is unique on the subspace $\Ts$ and we denote this solution as $\Lbin$.  This is easy to check: for any $j\in[m]$, $\ellbb_j$ is a solution of Eq. \eqref{eq:NTP-comp} if and only if it takes the form $\ellbb_j=\Ebin{j}^\top\left(\Ebin{j}\Ebin{j}^\top\right)^{-1}\ab_j+\vb, \,\vb\in\Nc\left(\Ebin{}\right)$ where $\ab_j\in\R^{S_j-1}$ has entries $\log(\ph_{j,z}/\ph_{j,z'}), z'\neq z\in\Sc_j$. Thus, $\Lbin$ is explicitly defined as follows.
\begin{lemma}[$\Lbin$]\label{lem:Lin def}
    Let $\Lbin$ be the unique solution of  \Eqref{eq:NTP-comp} on the subspace $\Ts$. Then, $\Lbin$ has columns 
    \[
\ellbb_{\rm{j}}^{\rm{in}}=\Ebin{j}^\top\left(\Ebin{j}\Ebin{j}^\top\right)^{-1}\ab_j, \,\,\forall j\in[m],
    \]
    where $\ab_j\in\R^{S_j-1}$ has entries $\log(\ph_{j,{z_j}}/\ph_{j,z'})$ for $z'\neq z_j\in\Sc_j$.
\end{lemma}

On the other hand, as mentioned in Sec. \ref{sec:ntp svm}, \ref{eq:ufm-svm} is always feasible. 

The challenge is in satisfying conditions \emph{A}-\emph{C} simultaneously. 

However, when $d\geq V$, then the rank condition \emph{C} of the theorem is trivial. Thus, we have the following corollary.

\begin{corollary}\label{cor:entropy attain d>V}
Suppose $d\geq V$, then the entropy lower bound can be asymptotically attained by the NTP loss. 
\end{corollary}

\subsubsection{Ball-constrained CE minimization and main result} Instead of the regularized problem \eqref{eq:ufm relax}, it is more convenient (but, equivalent due to convexity) to study the following ball-constrained version:
\begin{align}\label{eq:ball relax}
    \Lbhat_B\in\arg\min_{\substack{\|\Lb\|_*\leq B}} \CE(\Lb)\,\,,
\end{align}
parameterized by $B>0$. We will study properties of the minimizers $\Lbhat_B$ of \eqref{eq:ball relax} as $B\rightarrow\infty.$ The theorem below subsumes Theorems \ref{thm:reg path} and \ref{thm:reg path finite}.

\begin{theorem}\label{thm:reg path app}
    Suppose $d\geq V$. Denote $\Lbhat_B$ any solution of \eqref{eq:ball relax} for regularization $B$.  Then, in the limit of diverging regularization the following statements are true.
    \begin{enumerate}[leftmargin=*]
        \item[(i)] $\lim_{B\rightarrow\infty} \CE(\Lbhat_B) = \Hc$.
        \item[(ii)] $\lim_{\Binf} \Qcs(\Lbhat_B)=\Lbin\,$, where $\Lbin$ is defined in Lemma \ref{lem:Lin def}.
        \item[(iii)] 
        %$\lim_{\Binf} \|\Qcsperp(\Lbhat_B)\|=\infty.$ Thus, 
        The solution diverges in norm: $\lim_{\Binf} \|\Lbhat_B\| = +\infty$.
        \item[(iv)] The solution converges in direction to the optimal set of \ref{eq:ufm-svm}. That is, there exists minimizer $\Lbmm$ of \ref{eq:ufm-svm} such that
        \[
        \lim_{\Binf}\,\,\Big\|{\frac{\Lbhat_B}{\|\Lbhat_B\|_*} - \frac{\Lbmm}{\norm{\Lbmm}_*}}\Big\|=0.
        \]
        \end{enumerate}     
\end{theorem}

% The first condition, called \NTPH-compatibility, is as follows. 
% \begin{definition}\label{def:NTP-comp}
%     There exists logit matrix $\Lb=[\ellbb_1,\ldots,\ellbb_m]\in\R^{V\times m}$ such that $\rank{\Lb}\leq d$ and for all $j\in[m]$ it holds that
%     \[
%    \forall z\neq z'\in\Sc_j\,:\, \Lb[j,z]-\Lb[j,z'] = \log\left(\frac{\ph_{j,z}}{\ph_{j,z'}}\right)\,.
%     \]
% \end{definition}
% This is a slight modification of the condition presented in \cite{ntp} to include the necessary rank constraint $\rank{\Lb}\leq d$ that ensures logits can be factorized in word and context embeddings of dimension $d$.

% The second condition, termed \NTP-separability in \cite{ntp} can be phrased in our case as the requirement that \ref{eq:ufm-svm} is feasible with the additional constraint that $\rank{\Lb}\leq d$.

% \begin{lemma}
%     If $d\geq V$, then there exists $\Lb$ such that the \NTPH-compatibility logit equations in Definition \ref{def:NTP-comp} hold.
% \end{lemma}
% \begin{proof}
%     The equations in Definition \ref{def:NTP-comp} become:
%     \[
%     \forall j\in[m]\,:\,\Ebin{j}\ellbb_j=\ab_j\,
%     \]
%     where $\ab_j\in\R^{S_j}$ has entries $\log\left(\frac{\ph_{j,z}}{\ph_{j,z'}}\right)$.
%     Since $\Ebin{j}$ is full row rank, choosing $\ellbb_j=\Ebin{j}^\top(\Ebin{j}\Ebin{j}^\top)^{-1}\ab_j$ satisfies the equations. The resulting matrix $\Lb$ is of rank at most $V$, hence also $\Lb\in\rankset_d$.
% \end{proof}

\subsubsection{The in-support NTP loss and its minimizer $\Lbin$}

Define the in-support NTP loss as follows:
\begin{align*}%\label{eq:CEin}
    \CEin(\Lb) &= \sum_{j\in[m]}\pih_j\sum_{z\in\Sc_j}\ph_{j,z}\log\left(1+\sum_{z'\neq z\in\Sc_j}e^{-(\Lb[z,j]-\Lb[z',j])}\right)
\end{align*}
Contrary to \eqref{eq:CE_T}, the summation inside the logarithm is only over word indices $z$ that belong to the support set of context $j$.

By non-negativity of $\log$, for all $\Lb$ it holds 
$
\CE(\Lb)\geq \CEin(\Lb).$
In fact, because of the sparsity ansatz ($\exists j\in[m]$ such that $S_j<V$) and strict nonegativity of exponential, the inequality is strict for all \emph{finite} $\Lb$, i.e.
\begin{align}\label{eq:CEin lower bound}
\CE(\Lb)>\CEin(\Lb),\, \forall \Lb\in\R^{V\times d}\,.
\end{align}

Moreover, observe that for any $\Lb$, it holds that
\begin{align}\label{eq:CEin proj property}
    \CEin(\Lb)=\CEin(\Qcs(\Lb))\,.
\end{align}
To see this, decompose $\Lb$ in orthogonal components onto $\Ts$ and $\Ts_\perp$, respectively, i.e., $\Lb=\Qcs(\Lb)+\Pc_\perp(\Lb)$. By  Eq. \eqref{eq:F perp operational}, $\Lb_\perp:=\Pc_\perp(\Lb)$ satisfies Eq. \eqref{eq:svm equalities}, i.e.  $\Lb_\perp[z,j]-\Lb_\perp[z',j]=0$ for all $z,z'\in\Sc_j$ and $j\in[m]$.

The lemma below uses this to show that $\CEin$ has a unique minimizer in $\Ts$.

\begin{lemma}\label{lem:CEin unique minimizer}
    The in-support loss $\CEin(\Lb)$ has a unique finite minimizer in $\Ts$, which we denote $\Lbin$: 
    \begin{align}\label{eq:Lbin}
        \Lbin = \arg\min_{\Lb\in\Ts} \CEin(\Lb)\,.
    \end{align}
    The minimizer $\Lbin$ satisfies the \NTPH-compatibility condition in Eq. \eqref{eq:NTP-comp}
    % . Explicitly, $\Lbin$ has columns 
    % \[\ellbb_j^{\text{in}}:=\Ebin{j}^\top\left(\Ebin{j}\Ebin{j}^\top\right)^{-1}\ab_j,\]
    % where
    % $\ab_j:=[\ph_{1,1},\ldots,\ph_{1,S_1}]^\top$.
    % Finally, 
    and 
    \[
    \CEinstar:=\CEin(\Lbin)= \ent.
    \]
\end{lemma}
\begin{proof}
    Note that $\Lb\in\Ts$ is a column-wise separable constraint $\ellbb_j\in\Rc(\Ebin{j}^\top)$ that restricts the non-zero entries $\Lb[z,j]=\ellb_{j,z}$ to indices $z\in\Sc_j$. The objective is also separable with respect to columns. Without loss of generality fix $j=1$ and assume $\Sc_j=[S_j].$ It then suffices proving that the function  $f:\R^{S_1}\rightarrow\R$:
    \begin{align}\label{eq:f coercive}
    f(\ellbb)=-\sum_{z\in\Sc_1}\ph_{1,z} \log\left(\frac{\exp({\ell_z})}{\sum_{z'\in\Sc_1} \exp({\ell_{z'}})}\right)\,,
    \end{align}
    has a unique minimizer in $\Rc(\Ebint{1}^\top)$, where $\Ebint{1}:=\Ebin{1}[:,\Sc_1]\in\R^{(S_1-1)\times S_1}$ is the non-zero block of $\Ebin{1}$. 
    
    By non-negativity of the KL divergence, 
    \[
    f(\ellbb) \geq -\sum_{z\in\Sc_1}\ph_{1,z} \log\left(\ph_{1,z}\right)
    \]
    with equality if and only if,  
    \begin{align}\label{eq:eqv 1}
        \forall z\in\Sc_1\,:\,\frac{\exp({\ell_z})}{\sum_{z'\in\Sc_1} \exp({\ell_{z'}})} = \ph_{1,z}.
    \end{align}
    In turn, this is equivalent to 
    \begin{align}\label{eq:eqv 2}
       \forall z,z'\in\Sc_1\,:\,  \ell_{z} - \ell_{z'} = \log(\ph_{1,z}/\ph_{1,z'})\,.
    \end{align}    
    To see the equivalence  note the following: On the one hand, starting with Eq.~\eqref{eq:eqv 1} for $z,z'\in\Sc_1$, dividing both sides of the two equalities, and taking the logarithm, we arrive at Eq. \eqref{eq:eqv 2}. On the other hand, if Eq.~\eqref{eq:eqv 2} holds, then Eq.~\eqref{eq:eqv 1} holds since
    \[
    \frac{\sum_{z'\in\Sc_1} \exp({\ell_{z'}})}{\exp({\ell_z})} =  \sum_{z'\in\Sc_1} \exp(\ell_{z'}-\ell_{z'}) = \frac{\sum_{z'\in\Sc_1} \ph_{1,z'}}{\ph_{1,z}}= \frac{1}{\ph_{1,z}}\,.
    \]

Now, observe that \eqref{eq:eqv 2} is the same as the $\NTP_\ent$-compatibility equation for $j=1$. Moreover, note that a logit vector $\ellbb\in\R^{S_1}$ solves \eqref{eq:eqv 2} if and only if for 
    \[
    \ellbb = \Ebint{1}^\top\left(\Ebint{1}\Ebint{1}^\top\right)^{-1}\ab_1 + \vb, \,\, \vb\in\Nc\left(\Ebint{1}\right),
    \]
    where $\ab\in\R^{S_1-1}$ with entries $\log(\ph_{1,z_1}/\ph_{1,z'})$ for $z'\neq z_1\in\Sc_1$. 
This proves that there is a unique minimizer  $\ellbb_1^{\text{in}}:=\Ebint{1}^\top\left(\Ebint{1}\Ebint{1}^\top\right)^{-1}\ab_1$ on $\Rc\left(\Ebint{1}^\top\right).$   
  
    Since the choice of $j=1$ above was arbitrary, this proves the lemma.
\end{proof}

\subsubsection{Auxiliary lemmas}
This section proves a sequence of auxiliary lemmas that are used to prove Thm. \ref{thm:reg path app} in the next section.

The first lemma shows that entropy can only be approached asymptotically, i.e., provided the logit matrix diverges.

\begin{lemma}
    For all $\Lb\in\R^{V\times m}$ and all $\Lb'\in\Ts_\perp$ that additionally satisfy $\Ebout{j}\ellbb'_{j}\geq 1, \forall j\in[m]$ (i.e., $\Lb'$ is \ref{eq:ufm-svm} feasible), it holds
    \[
\CEin(\Lb) = \lim_{R\rightarrow\infty}\CE(\Lb+R\Lb').
\]
Using this, it holds for all $\Lb\in\R^{V\times m}$ that
    \begin{align}\label{eq:suboptimality of finite L}
\lim_{R\rightarrow\infty}\CE(\Lbin+R\Lb') = \CEin(\Lbin) = \CEinstar = \ent \leq \CEin(\Lb) < \CE(\Lb).
\end{align}
\end{lemma}
\begin{proof}
    To see the first claim note that $\Lb'\in\Fc_\perp$ implies
    \begin{align*}
     &\forall j\in[m],z\neq z'\in\Sc_j\,:\,\Lb'[j,z]=\Lb'[j,z'] \implies \\
     &\qquad\qquad\forall j\in[m]\,,z\in\Sc_j\,:\,\sum_{z'\in\Sc_j}e^{-\left(\Lb[z,j]-\Lb[z',j]+R(\Lb'[j,z]-\Lb'[j,z'])\right)}=\sum_{z'\in\Sc_j}e^{-\left(\Lb[z,j]-\Lb[z',j]\right)}
    \end{align*}
    and $\Ebout{j}\ellbb'_{j}\geq 1, \forall j\in[m]  $ implies
    \begin{align*}
    &\forall j\in[m],z\in\Sc_j, v\notin\Sc_j\,:\,\Lb'[j,z]-\Lb'[j,v]>0 \implies
    \\
     &\qquad\qquad\forall j\in[m]\,, z\in\Sc_j\,:\,\sum_{v\notin\Sc_j}e^{-\left(\Lb[z,j]-\Lb[v,j]+R(\Lb'[j,z]-\Lb'[j,v])\right)}\stackrel{\Rinf}{\longrightarrow} 0\,.
    \end{align*}
    The second claim applies the first claim for $\Lb=\Lbin$ and uses Eq. $\eqref{eq:CEin lower bound}.$
\end{proof}

We now use the previous lemma to show that the ball-consrtained minimizer diverges as the ball constraint approaches infinity.

\begin{lemma}\label{lem:norm diverges}
    The norm of the ball-constrained minimizer diverges, that is
    \[
    \lim_{\Binf}\|\Lbhat_B\|_*=\infty
    \]
\end{lemma}
\begin{proof}
Suppose on the contrary that $\lim_{B\to\infty}\|\Lbhat_B\|_*<\infty$. From Eq. \eqref{eq:suboptimality of finite L}, we have 
\[
\CE(\Lbhat_B) > \lim_{\Rinf}\CE(\Lbin+R \Lbmm).
\]
Moreover, by triangle inequality for nuclear norm $\|\Lb_R\|_*\leq\|\Lbin\|_*+R\|\Lbmm\|_*$.  Thus, for any arbitrarily large $R>0$, there exists large enough $B$ such that $\|\Lb_R\|_*\leq B$, making $\Lb_R:=\Lbin+R \Lbmm$ feasible in \eqref{eq:ball relax}.
 These together, contradict optimality of $\Lbhat_B$, proving that $\lim_{\Binf}\|\Lbhat_B\|_*=\infty$. Finally, recalling the norm inequality
 $
 \|\Lbhat_B\|\geq \|\Lbhat_B\|_*/\sqrt{V}\,,
 $
 it also holds that $\lim_{\Binf}\|\Lbhat_B\|=\infty.$
\end{proof}

The last lemma constructs a sequence of logits approaching the entropy lower bound while remaining feasible in problem \eqref{eq:ball relax} as $B\rightarrow\infty$.
\begin{lemma}\label{lem:CEB approaches CEin}
    The objective function approaches the loss lower bound, i.e.,
    \[
    \lim_{\Binf}\CE(\Lbhat_B)=\CEinstar=\ent\,.
    \]
    Concretely, there exists logit matrix $\Lb_R$ parameterized by $R=R(B)$ that is feasible in \eqref{eq:ball relax} and sastisfies: \looseness=-1
    \begin{align}
\CE(\Lb_R)\leq     
     \CEinstar + {V e^{2\|\Lbin\|_2}\,e^{-R}}\,.\label{eq:reg path up}        
    \end{align}
     Moreover, $\lim_{\Binf}R=\infty$ and $\lim_{\Binf}\frac{R}{\|\Lbhat_B\|_*}=\frac{1}{\|\Lbmm\|_*}$ where $\|\Lbmm\|_*$ is the optimal cost of \ref{eq:ufm-svm}.
\end{lemma}
\begin{proof} 
We evaluate the loss achieved by the following candidate good point \[\Lb_R:=\Lbin+R(B)\cdot\Lbmm\,.\]
Here, (for large enough $B>\|\Lbin\|_*$) we 
 set
\[
R:=R(B)=\frac{\|\Lbhat_B\|_*-\|\Lbin\|_*}{\|\Lbmm\|_*}.
\]
Note this is chosen so that $\|\Lb_R\|_*\leq\|\Lbhat_B\|_*\leq B$ towards making $\Lb_R$ feasible in \eqref{eq:ball relax}. Also, recall from Lemma \ref{lem:norm diverges} that $\|\Lbhat_B\|_*\rightarrow\infty$ as $B\rightarrow\infty$; thus, also $R\rightarrow\infty$.

It remains to prove \Eqref{eq:reg path up}. This follows directly from the sequence of equations in \eqref{eq:LR clean bound}.

\end{proof}

% \begin{lemma}\label{lem:Lb finite component convergence}
% The finite component of $\Lbhat_B$ converges to $\Lbin$. Concretely,
%     \[
%     \lim_{\Binf}\Qcs(\Lbhat_B) = \Lbin.
%     \]
% \end{lemma}
% \begin{proof}
%     Recall from Lemma \ref{lem:CEB approaches CEin} that $\lim_{\Binf}\CE(\Lbhat_B)=\CEinstar$. Since $\CE(\Lb)\geq \CEin(\Lb)\geq\CEinstar$ for all $\Lb$, this implies that
%     \[
%     \lim_{\Binf}\CEin(\Lbhat_B)=\CEinstar\,.
%     \]
%     Now, note that for all $\Lb$ it holds 
%     $\CEin(\Lb)=\CEin(\Qcs(\Lb))$. Combining with the above display yields
%     \[
%     \lim_{\Binf}\CEin(\Qcs(\Lbhat_B))=\CEinstar\,.
%     \]
%     Since $\Lbin$ is unique minimizer of $\CEin$ on $\Tc$ by Lemma \ref{lem:CEin unique minimizer}, the desired follows.
% \end{proof}

\subsubsection{Proof of Theorem \ref{thm:reg path app}}

We prove each statement separately.

\noindent\textbf{\underline{Statement (i).}}~Since for all $\Lb$, it holds $\CE(\Lb)\geq \ent$, it suffices to prove that $\ent$ can be asymptotically attained.  Since $d\geq V$, this follows from Corollary \ref{cor:entropy attain d>V}. (For a more explicit proof, see Lemma \ref{lem:CEB approaches CEin}).

\noindent\textbf{\underline{Statement (ii).}}~  Since  for all $\Lb$ it holds $\CE(\Lb)\geq \CEin(\Lb)\geq\CEinstar=\ent$, it follows from Statement (i) that
    \[
    \lim_{\Binf}\CEin(\Lbhat_B)=\CEinstar\,.
    \]
But
    $\CEin(\Lb)=\CEin(\Qcs(\Lb))$ (see Eq. \eqref{eq:CEin proj property}). Thus,
    \[
    \lim_{\Binf}\CEin(\Qcs(\Lbhat_B))=\CEinstar\,.
    \]
    The desired now follows ince $\Lbin$ is unique minimizer of $\CEin$ on $\Tc$ by Lemma \ref{lem:CEin unique minimizer}.

\noindent\textbf{\underline{Statement (iii).}}~Follows directly from Lemma \ref{lem:norm diverges}.

\noindent\textbf{\underline{Statement  (iv).}}~
From Lemma \ref{lem:CEB approaches CEin}, there exists $\Lb_R$ for which the loss is close to $\ent$ (see \Eqref{eq:reg path up}).

Towards arriving at a contradiction, 
we will show that if $\Lbhat_B$ is not in the direction of $\Lbmm$, then it incurs a loss that is larger than the upper bound of $\CE(\Lb_R)$ computed in \Eqref{eq:reg path up}. 

To do this, assuming the statement of the theorem is not true, 
we will lower bound 
\begin{align}\label{eq:towards Wb}
\CE(\Lbhat_B)-\ent = \sum_{j\in[m]}\pih_j\sum_{z\in\Sc_j}\ph_{j,z}\log\Big(\ph_{j,z}\,\big({\sum_{z'\in\Sc_j}e^{-(\Lbhat_B[j,z]-\Lbhat_B[j,z'])}+\sum_{v\notin\Sc_j}e^{-(\Lbhat_B[j,z]-\Lbhat_B[j,v])}}\big)\Big).
\end{align}
By our assumption, there exists $\eps>0$, such that there exists arbitrarily large $B$ satisfying:
\begin{align}\label{eq:contra eps}
\Big\|\frac{\|\Lbmm\|_*}{\|\Lbhat_B\|_*}\Lbhat_B-\Lbmm\Big\|>\eps,
\end{align}
for all minimizers $\Lbmm$ of \ref{eq:ufm-svm}. 
Define
\[
\Lbhat=\frac{1}{R'(B)}\big(\Lbhat_B-\Lbin\big),
\]
where, $R':=R'(B)>0$ is chosen so that $\|\Lbhat\|_*< \|\Lbmm\|_*$. 
Concretely,  this can be ensured by setting (to see this use the triangle inequality for nuclear norm):
\[
R'=\frac{(\|\Lbhat_B\|_*+\|\Lbin\|_*+\zeta)}{\|\Lbmm\|_*},
\]
%Note that it holds  $\lim_{B\rightarrow\infty}\frac{R'}{B}=\frac{\sqrt{V}}{\|\Lbmm\|_*}.$

for some $\zeta>0$. Recall that $\|\Lbhat_B\|_*\rightarrow\infty$ as $B\rightarrow\infty$. On the other hand $\|\Lbin\|_*$ is not dependent on $B$. Thus, in the large limit $B\rightarrow\infty$, it holds $\abs{\frac{\|\Lbmm\|_*}{\|\Lbhat_B\|_*}-\frac{1}{R'}}\rightarrow0$ and $\frac{\|\Lbin\|_*}{\|\Lbhat_B\|_*}\rightarrow0$. These combined with \eqref{eq:contra eps} show that we can always choose $B$ large enough so that Eq. \eqref{eq:contra eps} guarantees, for some $\eps'>0$, that
\[
\|\Lbhat-\Lbmm\|\geq \eps'\,.
\]
But $\Lbhat$ achieves the optimal cost of \ref{eq:ufm-svm}, since $\|\Lbhat\|_*<\|\Lbmm\|_*$. Thus,
 there exists $\delta\in(0,1)$ and $j\in[m]$ such that at least one of the following is true

(i) $\exists z$ and $z'\neq z\in\Sc_{j}$ such that 
\begin{align}
    \label{eq:contra 1}
    |\Lbhat[j,z]-\Lbhat[j,z']|\geq \delta\,,
\end{align}

(ii) $\exists z\in\Sc_{j}, v\notin\Sc_{j}$ such that 
\begin{align}\label{eq:contra 2}
\Lbhat[j,z]-\Lbhat[j,v]\leq 1-\delta.    
\end{align}

\noindent\emph{\underline{Case (i):}}~ Without loss of generality $(\eb_z-\eb_{z'})^\top\ellbhat_j\leq -\delta$ (otherwise, flip $z,z'$). Thus, ignoring all but one term in \eqref{eq:towards Wb} gives
\begin{align}\label{eq:towards Wb 2}
\CE(\Lbhat_B)-\ent &\geq \pih_j\ph_{j,z}\log\Big(\ph_{j,z}\,{e^{-(\Lbhat_B[j,z]-\Lbhat_B[j,z'])}}\Big).
\end{align}
But,
\begin{align}%\label{eq:contra logits}
\Lbhat_B[j,z]-\Lbhat_B[j,z'] = R'(\Lbhat[j,z]-\Lbhat[j,z'])+(\Lbin[j,z]-\Lbin[j,z'])\leq -R'\delta+2\|\Lbin\|_2
\end{align}
Put in \eqref{eq:towards Wb} and using $ \ph_{j,z}\geq1/V$ , $\pih_j \geq 1/m$ shows 
\begin{align}\nn
    \CE(\Lbhat_B) \geq \ent + \frac{1}{mV}
    \log\Big(\frac{1}{V}{e^{R'\delta}\exp(-2\|\Lbin\|_2)}\Big)
    \geq \ent + \frac{1}{mV}
    \log\Big(\frac{e^{R'\delta}}{V\exp(2\|\Lbin\|_2)}\Big)
\end{align}
Compare this with \eqref{eq:reg path up}, it is clear that as $B\rightarrow\infty$, so that $R'\rightarrow\infty$, it holds $\frac{1}{mV}
    \log\Big(\frac{e^{R'\delta}}{V\exp(2\|\Lbin\|_2)}\Big)>1> V\exp(2\|\Lbin\|_2) \, e^{-R}$. Thus, $\CE(\Lbhat_B)>\CE(\Lbin_R)$, a contradiction.

\noindent\emph{\underline{Case (ii):}}~We can assume $\Ebin{j}\ellbhat_j=0$ for all $j\in[m]$,
since otherwise we are in Case (i). Now, again ignoring all but the $(j,z)$ term in the CE loss for which \eqref{eq:contra 2} holds for some $v\notin\Sc_j$, we find
\begin{align}\label{eq:towards Wb 3}
&\CE(\Lbhat_B)-\ent\geq \pih_j\ph_{j,z}\log\Big(p_{j,z}\,{\big(\sum_{z'\in\Sc_j}e^{-(\Lbhat_B[j,z]-\Lbhat_B[j,z'])}+e^{(\Lbhat_B[j,v]-\Lbhat_B[j,z])}\big)}\Big)\nn.
\end{align}
On the one hand, using $\Ebin{j}\ellbhat_j=0$, we have
\[
\sum_{z'\in\Sc_j}e^{-(\Lbhat_B[j,z]-\Lbhat_B[j,z'])} = \sum_{z'\in\Sc_j}e^{-(\Lbin[j,z]-\Lbin[j,z])}=\sum_{z'\in\Sc_j}\frac{\ph_{j,z'}}{\ph_{j,z}}=\frac{1}{\ph_{j,z'}}.
\]
On the other hand, by \eqref{eq:contra 2}:
\[
e^{\Lbhat_B[j,v]-\Lbhat_B[j,z]}\geq e^{-R'(1-\delta)} \,e^{\Lbin[j,v]-\Lbin[j,z]} \geq  e^{-R'(1-\delta)}\exp(-2\|\Lbin\|_2), 
\]
Putting the above together yield:
\begin{align}\nn
    \CE(\Lbhat_B) - \ent &\geq \pih_j\ph_{j,z}\log\Big(1+\frac{ e^{-R'(1-\delta)}\exp(-2\|\Lbin\|_2)}{V}\Big) \geq \frac{e^{-R'(1-\delta)}}{mV^2\exp(2\|\Lbin\|_2)+mV}\,.
\end{align}
where the second inequality uses $\log(1+x)\geq \frac{x}{1+x}, x>0$. 

Compare this with \eqref{eq:reg path up}. As $B\rightarrow\infty$, and noting that $R,R'$ grow at the same rate, it holds $\frac{e^{-R'(1-\delta)}}{mV^2\exp(2\|\Lbin\|_2)+mV}> V\exp(2\|\Lbin\|_2)\, e^{-R}$. Thus, $\CE(\Lbhat_B)>\CE(\Lb_R)$, a contradiction.

In either case, we arrive at a contradiction, which completes the proof.

\subsection{Proof of Proposition \ref{propo:collapse}}\label{sec:collapse_proof}
We will prove that for $j,j'\in[m]$ with same support set $\Sc_j=\Sc_{j'}$ the $j$-th and $j'$-th columns of $\Lbmm=[\ellbb_1,\ldots,\ellbb_m]$ are same. For the sake of contradiction, suppose that this is not the cases, i.e. $\ellbb_j\neq\ellbb_{j'}$. Consider the following three candidate solutions of \ref{eq:ufm-svm}:
    \begin{align*}
        \Lbmm&=[\ellbb_1,\ldots,\ellbb_j,\ldots,\ellbb_{j'},\ldots\ellbb_m]
        \\
        \Lbmm_{\rm{flip}}&=[\ellbb_1,\ldots,\ellbb_{j'},\ldots,\ellbb_{j},\ldots\ellbb_m]
        \\
        \Lbmm_{\rm{avg}}&=\big[\ellbb_1,\ldots,\frac{\ellbb_j+\ellbb_{j'}}{2},\ldots,\frac{\ellbb_j+\ellbb_{j'}}{2},\ldots\ellbb_m\big]\,.
    \end{align*}

By assumption $\Lbmm$ is optimal in \ref{eq:ufm-svm}. We now show that the other two matrices are also optimal.\looseness=-1

Firstly, note that $\Lbmm_{\rm{flip}}$ is feasible in \ref{eq:ufm-svm} because: (i) the affine constraints only depend on the support sets, which are same for $j,j's$, (ii) $\rank{\Lbmm_{\rm{flip}}}=\rank{\Lbmm}\leq d$ since permuting columns is rank-preserving operation. Moreover, permuting columns also preserves nuclear-norm of a matrix. To see this note that $\Lbmm_{\rm{flip}}=\Lbmm\Pbb=\Ub\Sigmab(\Pbb^\top\Vb)^\top$ for some permutation matrix $\Pbb\in\R^{m\times m}$. Since $\Vb$ is partial orthonormal, the same is true for $\Pbb^\top\Vb$. Thus, $\|{\Lbmm_{\rm{flip}}}\|_*=\tr({\Sigmab})=\|{\Lbmm}\|_*$, which proves optimality of $\rank{\Lbmm_{\rm{flip}}}$.

Similarly, we can argue that $\Lbmm_{\rm{avg}}$ is feasible in \ref{eq:ufm-svm} because: (i) the affine constraints only depend on the support sets, which are same for $j,j's$, (ii) $\rank{\Lbmm_{\rm{avg}}}=\rank{\Lbmm}\leq d$ since adding columns and multiplying them by scalars are rank-preserving operations. But now, from convexity of the objective function, and noting that $\Lbmm_{\rm{avg}}=\frac{1}{2}\left(\Lbmm+\Lbmm_{\rm{flip}}\right)$  we have
\[
\|\Lbmm_{\rm{avg}}\|_* \leq \frac{1}{2}\|\Lbmm\|_*+\frac{1}{2}\|\Lbmm_{\rm{flip}}\|_* = \|\Lbmm\|_*\,.
\]
Thus, $\Lbmm_{\rm{avg}}$ is also optimal. 

Now, from Lemma \ref{lem:relax}, we know that     for some partial orthonormal matrix $\Rb$,  in the limit of vanishing regularization:  ${\Hb_\la}\propto\Rb^\top\sqrtm{\Sigmab}\Vb^\top$ where $\Vb,\Sigmab$ are SVD factors of $\Lbmm_{\rm{avg}}$. Since the $j,j'$ columns of $\Lbmm_{\rm{avg}}$ are the same by construction, the same is true for the $j,j'$ columns of $\Vb^\top$. This proves the desired.
\section{Additional discussions on the solution of \texorpdfstring{\ref{eq:ufm-svm}}{NTP-SVM*}}\label{sec:ntpsvm_sol_app}
% \subsection{On the solution of \texorpdfstring{\ref{eq:ufm-svm}}{NTP-SVM*}}
% \ct{change framing and call it something like: on the structure of $\Lbmm$}
% \vspace{-0.05in}
% 
% To better understand how the SVD factorization informs the geometry, we investigate further how the structure of $\Lbmm$ (solution to \ref{eq:ufm-svm}) depends on the support matrix $\Sb$.
% \vspace{-0.1in}

In this section, we complement the results of Sec.~\ref{sec:ntpsvm_sol_main} on the properties of the \ref{eq:ufm-svm} solution. \looseness=-1

\subsection{Special case: support sets of equal sizes}
% \vspace{-0.05in}
Recall that in Prop.~\ref{prop:all symmetric geometry}, we introduced a special case of the sparsity pattern of the support sets for which we could find the optimal solution of \ref{eq:ufm-svm} in closed-form. The proposition below provides a more detailed characterization of the implicit geometry of the embeddings. 
% \begin{proposition}\label{prop:all symmetric geometry}
%     Fix $k\in[V-1]$ and suppose $\smat$ contains all $m={V\choose k}$ support sets of size $k$.  Then,
%     \begin{enumerate}[noitemsep,leftmargin=*]
%        % \setlength{\itemindent}{-2em}
%        % \vspace{-0.15in}
%         \item[(i)] The logit matrix takes the form $\Lbmm=(\Id_V-\frac{1}{V}\ones\ones^\top)\Sb$.
%         %
%         \item[(ii)] Word embeddings form equiangular tight frame (ETF) being equinorm and maximally separated. 
%         %
%         \item[(iii)] Context embeddings are equinorm and the embedding $\hb_j$ of context $j$ is co-linear to $\sum_{z\in\Sc_j}\w_z$.
%     \end{enumerate}
% \end{proposition}

% In the symmetric setting of the proposition, we can analytically solve \ref{eq:ufm-svm} giving $\Lbmm$ as expressed in the statement (i). In fact, it is possible to obtain an explicit characterization of the SVD of $\Lbmm$, which enables the precise definition of embeddings' geometries in statements (ii) and (iii). This explicit characterization leads to closed-form formulas for the angles and norms of embeddings, detailed in Prop.~\ref{prop:sym geo app}. 

\begin{proposition}\label{prop:sym geo app}
       Assume the setting of Proposition \ref{prop:all symmetric geometry}. The geometry of context and word embeddings  is described by the following relations:
    \begin{subequations}
    \begin{align}
        \forall v\neq v'\in[V]\,&:\,\cosa{\w_v}{\w_{v'}} = \frac{-1}{V-1} \quad\text{and}\quad \|\w_v\|=\|\w_{v'}\|
        \\
        \forall j\neq j'\in[m]\,&:\,\cosa{\hb_j}{\hb_{j'}}=\frac{|\Sc_j\cap\Sc_{j'}|-\frac{k^2}{V}}{k-\frac{k^2}{V}}\quad\text{and}\quad \|\hb_j\|=\|\hb_{j'}\|
        \\
        \forall j\in[m], v\in[V]\,&:\
        \cosa{\w_v}{\hb_j} = \begin{cases}
    \sqrt{\frac{V-1}{k(V-k)}} & v\in\Sc_j\\
     \frac{-1}{\sqrt{k(V-k)(V-1)}} & v\notin\Sc_j
    \end{cases}\,
    \\
    \forall j\in[m], v\in[V]\,&:\ \frac{\|\w_v\|^2}{\|\hb_j\|^2} = \frac{(V-1)\,\binom{V-2}{k-1}}{k(V-k)}\,.
    \end{align}
    \end{subequations}
\end{proposition}

\begin{figure}[t]
    \centering 
    \includegraphics[width=0.7\linewidth]{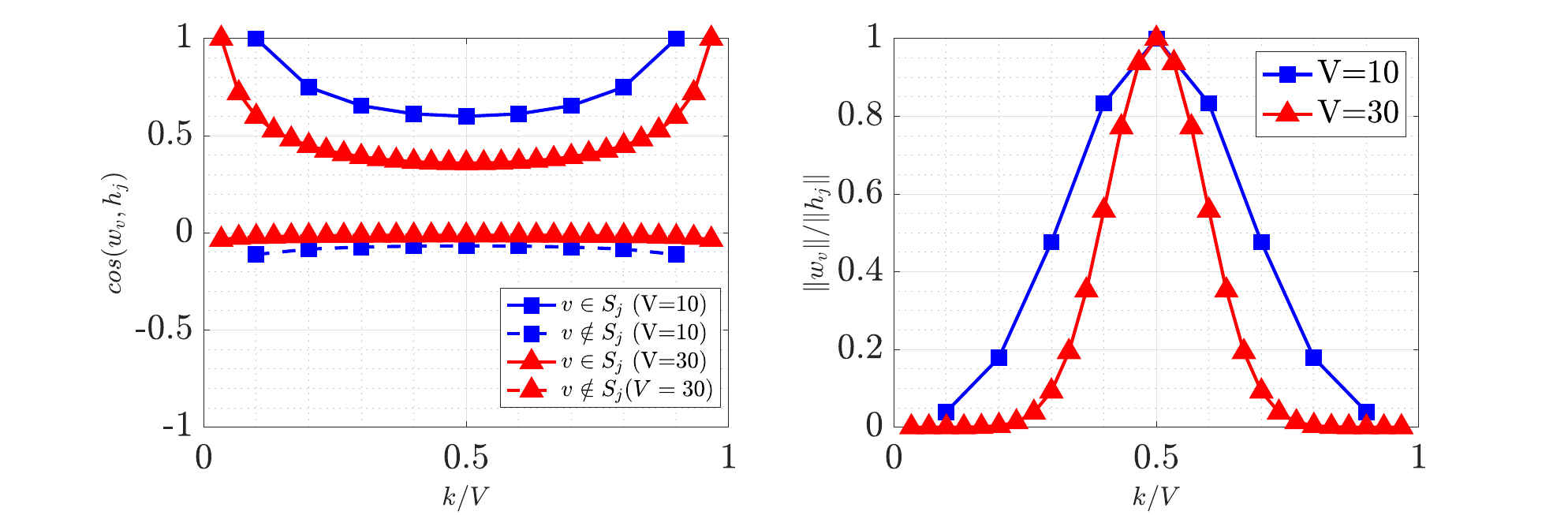}  
\captionsetup{width=\textwidth}
    \caption{Geometry properties illustration for Proposition \ref{prop:all symmetric geometry}. Shown two values of $V$ for varying values of $k\in[V-1]$. (Left) angles between context and word embeddings. (Right) normalized norm ratio of word to context embeddings.}
    \label{fig:symmetric_geo}
\end{figure}

% For $k=1$ these formulas recover the ETF geometry that has been previously shown for the setting of one-hot classification \citep{NC}. For general $k>1$, word embeddings continue forming an ETF, but the geometry of context embeddings changes although they remain equinorm. Embedding $\hb_j$ forms the same \emph{acute} angle with all in-support word vectors $\w_v, v\in\Sc_j$ and  the same \emph{obtuse} angle with all out-of-support word vectors $\w_v, v\notin\Sc_j$.  Additionally, for $k>1$, norms of word embedddings become larger than norms of context embeddings. See also 
Fig.~\ref{fig:symmetric_geo} visualizes these properties. For a brief discussion on this proposition, see Sec.~\ref{sec:special_case}. We defer the proof of these two propositions to App.~\ref{sec:props_proof}.

% \vspace{-0.05in}
\subsection{A simple candidate solution}\label{sec:candidate_app}
Recall that in Sec.~\ref{sec:ntpsvm_sol_main}, we introduced $\smatbar=(\Id_V-\frac{1}{V}\ones_V\ones_V^\top)\smat$ as a candidate solution for the \ref{eq:ufm-svm}. %the candidate solution introduced in Sec.~\ref{sec:candidate}, which we restate here for convenience.
Further recall that Prop.~\ref{propo:svm sufficient condition} provides a dual certificate condition that specifies whether $\Lbtina$ is optimal. Our empirical evaluations show that $\Lbtina$ successfully solves \ref{eq:ufm-svm} in numerous  realizations of $\Sb$ we have examined. However, this is not the case in all realizations of the support set structure. 

In Fig.~\ref{fig:L_counterEx}, we present an example of two realizations of $\Sb$, where the condition in Prop.~\ref{propo:svm sufficient condition} is met in one instance but not in the other. The second and third column of the figure display the heatmap of $\Lbtina$ and the optimal solution of \ref{eq:ufm-svm} as determined by CVXPY \citep{cvx} for the support sets $\Sb$ shown in the first column. Notably, although $\Lbtina$ is not optimal in the second row, it provides a close and simple approximation of the optimal $\Lbmm$ which almost captures the overall structure of the optimal solution found by CVXPY. Additionally, in Fig.~\ref{fig:L_conj_synthetic}, we show similar heatmaps of the optimal logits for the setup of the  \simTS~dataset in Sec.~\ref{sec:exp_main}. Once again, while $\Lbtina$ is not necessarily optimal, it serves as a close proxy for the optimal solution. Note that in the case of one-hot classification under STEP imbalanced data, $\Lbtina$ is previously shown to be the unique minimizer of \ref{eq:ufm-svm} \citep{seli}.

% \tina{worth mentioning that cvxpy solution has lower rank?}
% \ct{motivate the conjecture based on the prop above for the symmetric case}
% \begin{conjecture}\label{conj:Lmm}
%     The optimal solution of \eqref{eq:ufm-svm} is given as follows in terms of the support matrix $\smat\in\{0,1\}^{V\times m}$:
%     \begin{align}\label{eq:Lmm explicit}
%     \Lbmm [j,z] = \begin{cases}
%         1-|\Sc_j|/V, & z\in\Sc_j\,\\
%             -{|\Sc_j|}/{V}, & z\not\in\Sc_j\,
%         \end{cases}
%         \,.
%     \end{align}
%     That is, in matrix form: $
%         \Lbmm = \left(\Id_V-\frac{1}{V}\ones\ones^\top\right)\smat\,.
%     $
% \end{conjecture}
% \tina{check the counter example}

% 1. The diagonal element contributes \( k \) to the sum.

% 2. Each off-diagonal element corresponds to the inner product of two different columns of $\smat$, which equals the size $r$ of the intersections of the support sets that they represent and ranges from $1$ to $k-1$. Now, for each \( r \), the number of support sets that share exactly \( r \) elements with a given set is \( \binom{k}{r} \binom{V-k}{k-r} \).

% Thus, the sum of the entries of a column in \( S'S \) can be expressed as:
% \[
%  k + \sum_{r=1}^{k-1} r \binom{k}{r} \binom{V-k}{k-r} = \frac{k \,(V-1)!}{(k-1)! \,(V - k)!}\,.
% \]

\begin{figure*}[t]
	\vspace{-0pt}
	\centering
	\hspace{-67pt} 
 \begin{subfigure}{0.9\textwidth}
		\centering
		\begin{tikzpicture}
		% \node at (-0,-0) {\includegraphics[scale=0.35]{figures/Lconjecture.png}};
  \node at (-0,0){\includegraphics[scale=0.35]{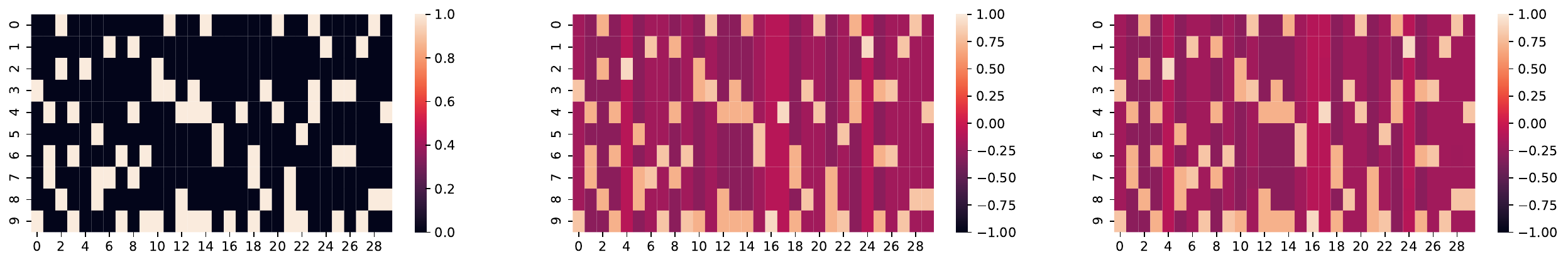}};
        \node at (-5.3,1.3) [scale=0.8]{$\Sb$};
        \node at (-0,1.3) [scale=0.8]{$\Lbtina$};
        \node at (4.8,1.3) [scale=0.8]{$\Lb_\text{CVXPY}$};
		\end{tikzpicture}
    \end{subfigure}

 \vspace{-10pt}
 \hspace{-60pt}
 \begin{subfigure}{0.9\textwidth}
		\centering
		\begin{tikzpicture}
		% \node at (-0,-0) {\includegraphics[scale=0.35]{figures/Lconjecture.png}};
  \node at (-0,0){\includegraphics[scale=0.35]{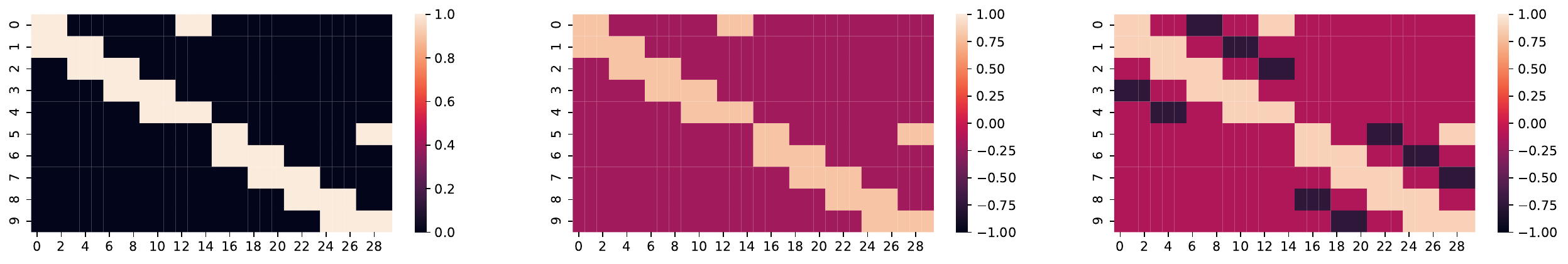}};
		\end{tikzpicture}
    \end{subfigure}
 \captionsetup{width=\linewidth}
 \vspace{-5pt}
    \caption{Comparison between the solution found by CVXPY for \ref{eq:ufm-svm} and $\Lbtina$ defined in Sec.~\ref{sec:ntpsvm_sol_main} for two realizations of the support set matrix $\Sb$. See text for details.
     }\label{fig:L_counterEx}
	\vspace{-15pt}
\end{figure*}

\begin{figure*}[t]
	\vspace{-30pt}
	\centering
	\hspace{-210pt} 
    \resizebox{0.8\textwidth}{!}{
    \begin{subfigure}{0.9\textwidth}
		\centering
		\begin{tikzpicture}
		% \node at (-0,-0) {\includegraphics[scale=0.35]{figures/Lconjecture.png}};   
        \node at (-0,0) {\includegraphics[scale=0.45]{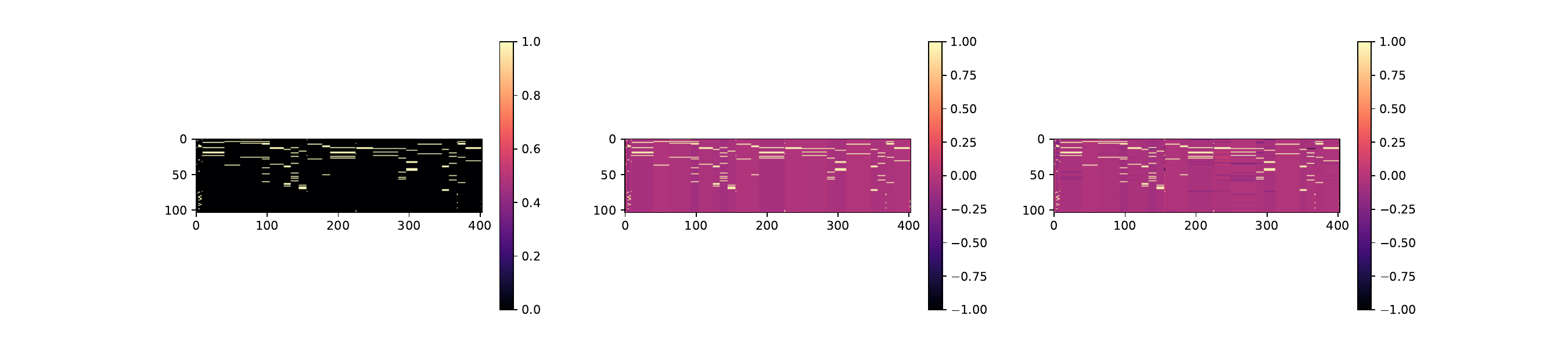}};
        \node at (-5.8,.7) [scale=0.9]{$\Sb$};
        \node at (-0.3,0.7) [scale=0.9]{$\Lbtina$};
        \node at (5.3,0.7) [scale=0.9]{$\Lb_\text{CVXPY}$};
		\end{tikzpicture}
    \end{subfigure}
    }
 \captionsetup{width=\linewidth}
 \vspace{-5pt}
    \caption{Same as Fig.~\ref{fig:L_counterEx} for the support set configuration in the \simTS~of Sec.~\ref{sec:exp_main}.\looseness=-1
     }\label{fig:L_conj_synthetic}
	\vspace{-15pt}
\end{figure*}

% \subsection{{Relationship between word/context embeddings geometry and support set}}
% \tina{section or paragraph?}
% % \input{Arxiv/sections/Lmm_conj}
% \input{Arxiv/sections/Lmm_conj_app}

\subsection{Proof of Proposition \ref{propo:svm sufficient condition}}
% \tina{we haven't defined the non-convex svm anywhere else ig}
% 

To prove the proposition, we first derive the dual of \ref{eq:ufm-svm}. 

We begin with a more convenient, but equivalent, formulation of \ref{eq:ufm-svm}. One way to do this is by recalling that \ref{eq:ufm-svm} is a convex relaxation to 
\begin{align*}
    \min_{\W,\Hb} \frac{1}{2}\norm{\W}_F^2 + \frac{1}{2}\norm{\Hb}_F^2, \quad \text{s.t.} \quad \forall j\in[m], z_j\neq z\in\Sc_j~&:~~(\eb_{z_j}-\eb_{z'})^\top\W \hb_j =0\,
        \\
        \forall j\in[m], v\notin\Sc_j~&:~~(\eb_{z_j}-\eb_{v})^\top\W \hb_j \geq 1\,.
\end{align*}
where, for each sample $j$, we choose an anchor label $z_j\in\Sc_j$. We can now relax this in terms of the matrix $\Xb = \begin{bmatrix}
    \W\\
    \Hb^\top
\end{bmatrix}\begin{bmatrix}
    \W^\top & \Hb
\end{bmatrix}$. (Note that $\Lb$ essentially corresponds to the off-diagonal blocks of $\X$).
This gives us the following reformulation of \ref{eq:ufm-svm}. (Formally, it is easy to show that \eqref{eq:svm X formulation} is the dual of the dual of \ref{eq:ufm-svm}. This result about nuclear minimization is classic (e.g., \cite{RechtFazel}) and we omit the details.): \looseness=-1
\begin{align}\label{eq:svm X formulation}
    \min_{\Xb\succcurlyeq0}\,\, \frac{1}{2}\tr(\Xb)\quad \text{s.t.}\quad \forall j\in[m], z_j\neq z\in\Sc_j~&:~~\Xb[z_j,V+j] - \Xb[z',V+j] =0\,
        \\
        \forall j\in[m], v\notin\Sc_j~&:~~\Xb[z_j,V+j] - \Xb[v,V+j] \geq 1\,. \nn
\end{align}

Define $\Ab_{V\times m}$ such that $\Ab[z,j] = -\lambdab_{j,z},\,z\neq z_j$ and $\Ab[z_j,j] = \sum_{z\neq z_j}\lambdab_{j,z}$. ($\oneb_V^\top\Ab=0$).

The dual problem is 
\begin{align*}
&
\max_{\Ab} \,\min_{\Xb\succcurlyeq0} \,\sum_{j\in[m]}\sum_{v\not\in\Sc_j}\lambdab_{j,v} + \frac12\tr\bigg(\begin{bmatrix}
    \Id & -\Ab\\
    -\Ab^\top & \Id 
\end{bmatrix}\Xb\bigg)\,\quad\text{s.t.}\quad \oneb_V^\top\Ab=0 \quad \Ab[v,j]\leq0,\,v\not\in\Sc_j\\
=&\max_{\Ab} \,\sum_{j\in[m]}\sum_{v\not\in\Sc_j}\lambdab_{j,v} \,\quad\text{s.t.}\quad \oneb_V^\top\Ab=0 \quad {\Ab[v,j]\leq0,\,v\not\in\Sc_j},\quad \norm{\Ab}_2\leq 1\\
=&\max_{\Ab} \,\,\tr(\Ab^\top\Lbtina) \,\quad\text{s.t.}\quad \oneb_V^\top\Ab=0 \quad {\Ab[v,j]\leq0,\,v\not\in\Sc_j},\quad \norm{\Ab}_2\leq 1
\end{align*}
%\tina{for $|\Sc_j|=1$ boils down to $\Ab \odot \Lbmm \geq 0$ similar to what we had before.}
% 
where here $\Lbtina$ is defined in Prop.~\ref{prop:all symmetric geometry}. The last line above holds since,
\begin{align*}
    \sum_{j\in[m]}\sum_{v\not\in\Sc_j}\lambdab_{j,v} &= \sum_{j\in[m]}\sum_{v\not\in\Sc_j}\lambdab_{j,v} + \sum_{j\in[m]}\sum_{\substack{v\in\Sc_j \\ v\neq z_j}}\lambdab_{j,v} - \sum_{j\in[m]}\sum_{\substack{v\in\Sc_j \\ v\neq z_j}}\lambdab_{j,v}\\
    &= \sum_{j\in[m]}\sum_{v\neq z_j}\lambdab_{j,v} - \sum_{j\in[m]}\sum_{\substack{v\in\Sc_j \\ v\neq z_j}}\lambdab_{j,v}\\
    &= \sum_{j\in[m]} (\frac{1-|\Sc_j|}{V} + \frac{|\Sc_j|}{V}) \sum_{v\neq z_j}\lambdab_{j,v}  - \sum_{j\in[m]}\sum_{\substack{v\in\Sc_j \\ v\neq z_j}}\lambdab_{j,v}\\
    &= \sum_{j\in[m]} \frac{1-|\Sc_j|}{V} \Ab[z_j,j] - \sum_{j\in[m]} \frac{|\Sc_j|}{V} \sum_{v\neq z_j} \Ab[v,j] + \sum_{j\in[m]}\sum_{\substack{v\in\Sc_j \\ v\neq z_j}}\Ab[v,j]\\
    &= \sum_{j\in[m]} \Lbtina[z_j,j] \Ab[z_j,j] + \sum_{j\in[m]} \sum_{v\not\in\Sc_j} \Lbtina[v,j]  \Ab[v,j] + \sum_{j\in[m]}\sum_{\substack{v\in\Sc_j \\ v\neq z_j}}\Lbtina[v,j] \Ab[v,j]\\
    &=\sum_{j\in[m]} \sum_{v\in[V]}\Lbtina[v,j] \Ab[v,j]\,.
\end{align*}

We remark the following about the dual problem derived above. First, because $\Lbtina$ is feasible in \ref{eq:ufm-svm} and constraints are linear, Slater's conditions hold, thus, we have strong duality. Second, by complementary slackness, if the nonegativity conditions $\Ab[v,j]\leq0, v\notin\Sc_j$ are strict, then, at optimality, $\Lb$ satisfies the inequality constraints \eqref{eq:svm inequalities} with equality. 

Now, regarding solving the dual, observe that if we remove the constraint $\Ab[v,j]\leq0, v\notin\Sc_j$, then the solution to the dual is simply $\Ab=\Ub\Vb^\top$ where $\Ub\Sigmab\Vb^\top$ is the SVD of $\Lbtina$. In fact, it is not hard to show that this is the unique solution of the relaxed dual; see for example \cite[Lemma C.1]{seli} for an analogous result. Therefore, if $\Ub\Vb^\top$ satisfies the condition of the proposition, then it is clearly the unique optimal solution to the dual problem. 

From this and complementary slackness discussed above, it must be that any minimizer of \ref{eq:ufm-svm} must satisfy the inequality constraints with equalities. Formally, the optimal solution $\hat\Lb$ satisfies simultaneously: \looseness=-1
\[
\hat\Lb^\top\ones_V=0_m\quad\text{and}\quad \forall j\in[m]: \Ebin{j}\hat\ellbb_j=\ones_{S_j},\,\, ~\Ebout{j}\hat\ellbb_j=0_{V-1-S_J}.
\]
Note that this is a system of $m+\sum_{j\in[m]}(S_j+V-1-S_j)=Vm$ linear equations. It can also be easily verified that these equations are linearly independent. Thus, their unique solution is $\Lbtina$. This completes the proof of the proposition.

\subsection{Proof of Propositions \ref{prop:all symmetric geometry} and \ref{prop:sym geo app}} \label{sec:props_proof}
At the heart of proving both Propositions \ref{prop:all symmetric geometry} and \ref{prop:sym geo app} is the following result, which determines the SVD factors $\Ub$, $\Sigmab$, and $\Vb$ of $\Lbtina=(\Id_V-\frac{1}{V}\ones\ones^\top)\Sb$. This shows immediately that word embeddings form an ETF. It also forms the basis for computing the geometry of context embeddings. We show this after proving the lemma.
\begin{lemma}
Fix any $k\in[V]$ and suppose $\smat$ contains all $V\choose k$ support sets of size $k$, i.e. $m={V\choose k}$. Let $\Pbb\in\R^{V\times (V-1)}$ be orthonormal basis of the subspace orthogonal to $\ones_V$. Then, the $\Ub,\Sigmab$ factors of the SVD of $\Lbtina$ are:
$\Ub=\Pbb$ and $\Sigmab=\sqrt{\binom{V-2}{k-1}}\,\Id_{V-1}$. Thus,
\[
\Ub\Sigmab\Ub^\top \propto \Big(\Id_V-\frac{1}{V}\ones\ones^\top\Big)\,.
\]
Moreover, the matrix $\Vb\in\R^{m\times(V-1)}$ of right-singular vectors is such that its rows $\vb_j, j\in[m]$ can be expressed as follows with respect to the rows $\ub_z, z\in[V]$ of $\Ub$:
\[
\vb_j = \frac{1}{\sqrt{\binom{V-2}{k-1}}}\sum_{z\in\Sc_j}\ub_z\,~~j\in[m].
\]
\end{lemma}
\begin{proof}
    We will compute explicitly $\Lbtina{\Lbtina}^\top$. Start with computing $\smat\smat^\top$.  For diagonal elements,  the dot product is between a row of $\smat$ and itself, counts the number of $1$s in that row. Since each element is included in $\binom{V-1}{k-1}$ support sets (choosing the remaining $k-1$ elements from the other $V-1$ elements), each diagonal entry of $\smat\smat^\top$ is $\binom{V-1}{k-1}$. For off-diagonal elements, the dot product counts the number of support sets in which both elements of the corresponding rows appear. This requires choosing the remaining $k-2$ elements from  $V-2$ elements, giving $\binom{V-2}{k-2}$ for each off-diagonal entry of $\smat\smat'$. Overall, after algebraic simplification, we find that
    \[
    \smat\smat^\top = \binom{V-2}{k-1}\left(\Id_V+\frac{k-1}{V-k}\ones\ones^\top\right)\,.
    \]
    We may now use the closed form of $\Lbtina$ and few more algebra simplifications to find that
    \[
    \Lbtina{\Lbtina}^\top=\binom{V-2}{k-1}\left(\Id_V-\frac{1}{V}\ones\ones^\top\right)\,.
    \]
    Let $\Pbb\in\R^{V\times (V-1)}$ denote an orthonormal basis of the subspace orthogonal to $\ones_V$. Since $ \Lbtina{\Lbtina}^\top=\binom{V-2}{k-1}\Pbb\Pbb^\top$ and $\Lbtina{\Lbtina}^\top=\Ub\Sigmab^2\Ub^\top$, we have shown that $\Ub=\Pbb$ and $\Sigmab=\sqrt{\binom{V-2}{k-1}}\,\Id_V$.
\end{proof}

Using $\ones^\top\smat=k\ones_m^\top$, we can compute
\[
{\Lbtina}^\top{\Lbtina}=\smat^\top\smat-\frac{1}{V}(\smat^\top\ones_V)(\smat^\top\ones_V)^\top = \smat^\top\smat-\frac{k^2}{V}\ones_m\ones_m^\top\,.
\]
From this and the fact that $\diag{\smat^\top\smat}=k\Id_m$, and recalling that ${\Lbtina}^\top\Lbtina=\Vb\Sigmab^2\Vb^\top$,  we find that all context embeddings are equinorm with squared norm (i.e. diagonal entries of $\Vb\Sigmab\Vb^\top$) equal to 
\[
\|\hb_j\|^2= \frac{k-\frac{k^2}{V}}{\sqrt{\binom{V-2}{k-1}}},\,\,\forall j\in[m].
\]
Recall also that the diagonal entries of $\Ub\Sigmab\Ub^\top$ are
\[
\|\w_v\|^2= \left(1-\frac{1}{V}\right)\sqrt{\binom{V-2}{k-1}},\,\,\forall v\in[V].
\]

From the off-diagonal entries of ${\Lbtina}^\top\Lbtina$, we can infer the angles between different context embeddings. Let $j\neq j'\in[m]$, then the $(j,j')$ of ${\Lbtina}^\top\Lbtina$ is $|\Sc_j\cap\Sc_{j'}|-\frac{k^2}{V}$. Thus, the  $(j,j')$ off-diagonal entries of $\Vb\Sigmab\Vb^\top$ are \[ \hb_j^\top\hb_{j'}=\frac{|\Sc_j\cap\Sc_{j'}|-\frac{k^2}{V}}{\sqrt{\binom{V-2}{k-1}}}\implies \cosa{\hb_j}{\hb_{j'}}=\frac{|\Sc_j\cap\Sc_{j'}|-\frac{k^2}{V}}{k-\frac{k^2}{V}}\,\,\,j\neq j'\in[m]\,.
\]

Recalling that $\W\Hb^\top=\Lbmm$, we may now compute the angle between word and context embeddings as follows:
\begin{align}
    \cosa{\w_v}{\hb_j} = \begin{cases}
    \sqrt{\frac{V-1}{k(V-k)}} & v\in\Sc_j\\
     \frac{-1}{\sqrt{k(V-k)(V-1)}} & v\notin\Sc_j
    \end{cases}\,.
\end{align}
This completes the proof of Proposition \ref{prop:sym geo app} and the proof of statements (ii) and (iii) of Proposition \ref{prop:all symmetric geometry}.

It remains to prove statement (i) of Proposition \ref{prop:all symmetric geometry}, i.e. proving that $\Lbtina=(\Id_V-\frac{1}{V}\ones\ones^\top)\Sb$ is the unique minimizer of \ref{eq:ufm-svm}.

To prove this, we appeal to Proposition \ref{propo:svm sufficient condition}. First, it is straightforward checking that $\Lbmm$ satisfies the SVM constraints (in fact, with equality). Thus, from Proposition \ref{propo:svm sufficient condition}, it suffices that the matrix $\A:=\Ub\Vb^\top$ satisfies $\A[v,j]<0$ for all $v\notin\Sc_j$ and all $j\in[m]$. But, in our case $\Ab=\Ub\Vb^\top=\binom{V-2}{k-1}^{-\frac{1}{2}}\Lbtina$ because $\Sigmab=\Id_{V-1}$. Thus, the desired condition holds for $\A$ since $\Lbtina[v,j]=-k/V<0$ for all $j\notin\Sc_j$ and $j\in[m]$.

% \tina{where dual certificate is discussed? proof should comeafter that}

% \section{Properties of \texorpdfstring{$\Lbmm$}{Lmm}}\label{sec:Lmm_app}
% \input{Arxiv_NewTempl/sections/Lmmprops}
% \input{Arxiv_NewTempl/sections/Lmmprops_app}

% \section{Geometry properties}
% \input{Arxiv_NewTempl/sections/geometry_app}
% \input{Arxiv_NewTempl/sections/expyize}
% \input{Arxiv_NewTempl/sections/expvala}
% \input{Arxiv_NewTempl/sections/rank_feas}

% \newpage
% \newpage
% \input{Arxiv_NewTempl/sections/ufm_do_not_upload}
% \input{Arxiv_NewTempl/sections/misc}
\FloatBarrier
% \newpage

\section{Numerical experiments}\label{Numerical_experiments_appx}
% \subsection{Numerical simulation:  \ref{eq:ufm}}\label{sec:ufmexp}
% \input{Arxiv_NewTempl/sections/expufm}

\subsection{Additional details and results: Sec.~\ref{sec:exp_main}} \label{sec:exp_details}

\subsubsection{Datasets}
We use a total of three datasets. We first employ two smaller-scale datasets, which are well-suited for verifying our theoretical solutions. Then, we use one larger-scale dataset designed to investigate the geometric properties of text in conditions that approximate real-world text scenarios. \looseness=-1 

% \begin{itemize}
    % \item 
    % \vspace{5pt}
    \noindent\textbf{\synth.}~We generate each context $\xd_j$ as follows: We select $T-1=5$ words (tokens), and manually come up with the $T=6$-th tokens that are consistent with the given context. To model the probabilities $\pih_j$ and $\pbh_j$, we sample each support independently to a number of repeats to emulate the behavior of repeated context in natural language. For instance, the context $\xd_j=$ \texttt{``Lily wants to try the''} is followed by $\Sc=\{\texttt{soup}, \texttt{game}, \texttt{movie}\}$ with respective probabilities $ [0.5, 0.25, 0.25]$. {The dataset consists of $n=116$ samples, containing $m=16$ distinct contexts, and a vocabulary size of $V=30$. Each context has a fixed support set length of $|\mathcal{S}_j|=3$.} In Fig.~\ref{fig:verysmalldataset}, we show the $m=16$ unique contexts and the soft labels on their next token. In this dataset, the tokenization is done at the word level and the empirical entropy lower-bound is $\Hc= 1.6597$.\looseness=-1% In Fig.~\ref{fig:verysmall_heatmaps}, we provide analogous plots to Fig.~\ref{fig:syn_heatmaps_w} for this setting.
    \begin{figure}[h!]
        \vspace{-10pt}
        \centering 
        \includegraphics[width=0.8\linewidth]{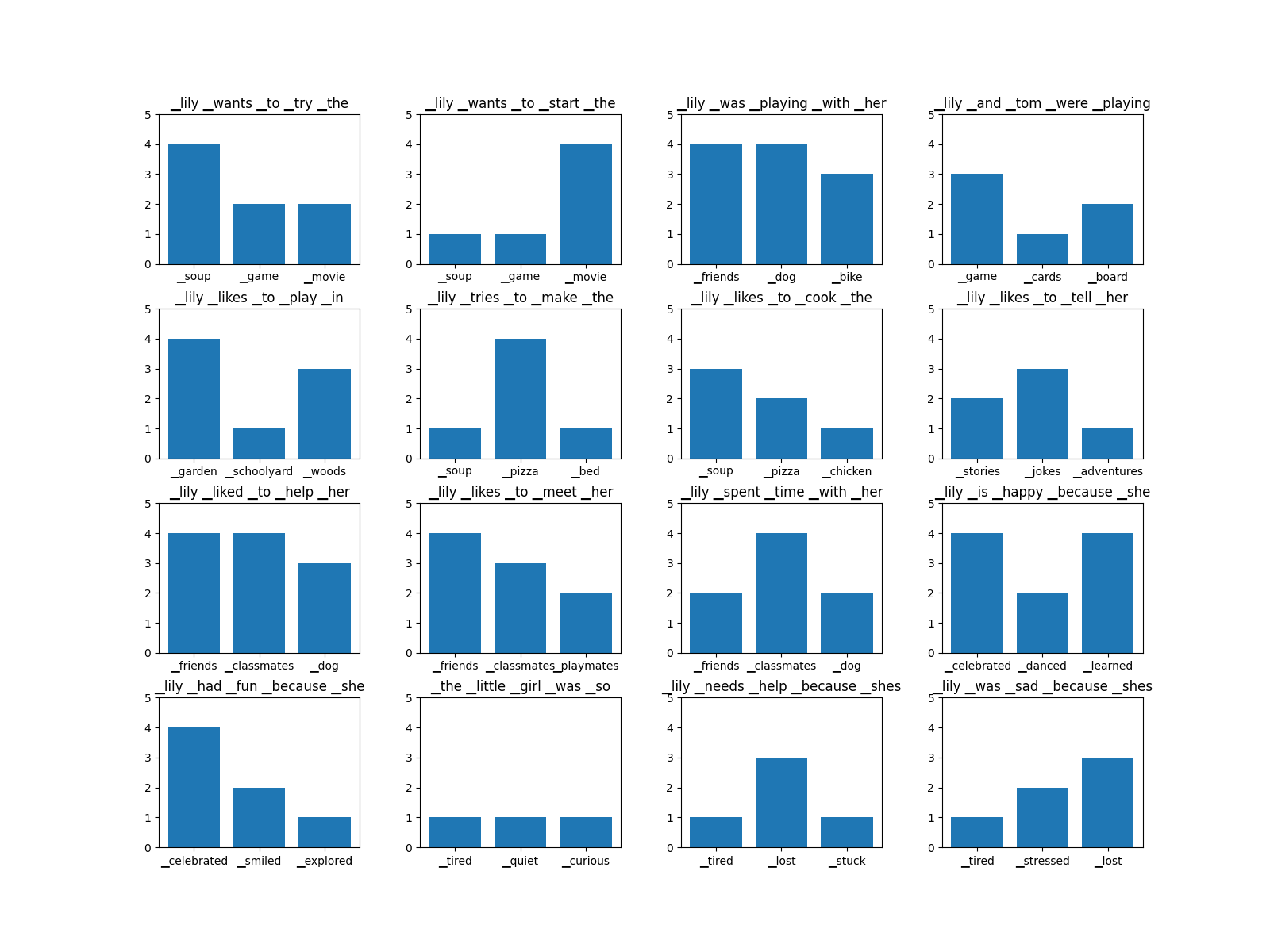}  
        \captionsetup{width=\textwidth}
        \vspace{-10pt}
        \caption{The contexts, support sets, and soft labels in the \synth~dataset.}
        \label{fig:verysmalldataset}
    \end{figure}
    %\item 
    % \vspace{5pt}

    \noindent\textbf{\simTS.}~Advancing towards a more realistic but still controlled dataset, we use contexts and support sets from the \TS~corpus. 
    We derive contexts $\xd_j$ and support sets $\Sc_j$ from the \TS~dataset and we create the training data from the most frequent word-level contexts with length 5, such as $\xd_j = [\texttt{``once'', ``upon'', ``a'', ``time'', ``,''}]$. For the support sets $\Sc_j$, we record all next tokens of $\xd_j$ in the original dataset. Then, we replace the words in the contexts with their synonyms to generate new contexts with identical support sets. This allows us to have multiple contexts sharing common support sets, while controlling the vocabulary size, which is set to $V=104$. {The final dataset consists of $n\approx3050$ samples with $m\approx400$ unique contexts.} The frequencies $\pbh_j$ are determined by  independently sampling each next token from $\Sc_j$ several times. Here, $\Hc= 1.0166$.
    
    %\item 
    % \vspace{5pt}
    \noindent\textbf{\TS.}~To experiment with more a standard dataset, we use 100 stories sampled from \TS. Here, we do not make any sample selections for contexts and support sets, unlike the previously mentioned datasets. We use a tokenizer with vocabulary size $V=128$, and fixed context length $T-1=6$. We choose a small context window and vocabulary size to make tracking the distinct contexts and their support sets computationally manageable. In this setup, the number of distinct contexts $m\sim 10^5$ and $\Hc=0.3112$. \looseness=-1

\vspace{-5pt}
% \noindent\textbf{Training Data.}~
% \noindent\textbf{Models.}~
% \tina{@Yize, check if any details are missing} 
% \vspace{-5pt}

\subsubsection{Training details}
\noindent\textbf{\synth~and \TS~experiemtns.}~We train the models long enough to ensure that the loss approached the empirical entropy lower-bound $\Hc$ within an order of $10^{-4}$. In all experiments, we use AdamW optimizer with a weight decay \(\lambda=10^{-6}\). The learning rate is initially set to \(10^{-4}\), with a step-decay schedule. \looseness=-1

\noindent\textbf{\TS~experiemtns}.~We train TF with AdamW optimizer until it reaches the empirical entropy lower-bound within an order $10^{-3}$. We use warm-up over the first 5 epochs to increase the learning rate to $5 \times 10^{-4}$ and use a cosine learning rate scheduler to decay its values over the course of the training. We set the weight decay to $10^{-5}$. \looseness=-1

\subsubsection{Additional results}\label{sec:additional_results_exp}
Figs.~\ref{fig:synthetic_heatmap_small}-\ref{fig:sim_W_label_verysmall} complement the discussions and experimental results in Sec.~\ref{sec:exp_main}. Below we only discuss the missing details of the \TS~experiments.

\noindent\textbf{Visualization details of \TS~experiments.}~For visualizing the heatmap of embeddings similarity in Fig.~\ref{fig:intro_TinyStories}, we choose $10$ unique support sets in the dataset with $|\Sc_j|>2$. For each of the chosen support sets $\Sc$, we choose $10$ distinct contexts $j$ such that $\Sc_j=\Sc$ (a total of $100$ samples), and illustrate the Gram matrix of the centered support matrix $\smatbar$ and context embeddings $\Hb$ for this subset of the contexts. Given that the support sets are sparse, we choose to use contexts with the largest support set size which would result in a more pronounced pattern on the correlation matrices. We skip visualizing the word embeddings $\W$ in this experiment, as the heatmaps did not display any noticeable visual structures. \looseness=-1

For the metrics displayed in Fig.~\ref{fig:TinyStories_curves}, due to the larger number of context embeddings, we only compute the metrics such as $\norm{\Qcs(\Lb_k) - \Lbin}$, norm growth and correlation with proxies using a sample of $1000$ token/context pairs chosen randomly from the training set. We observe qualitatively similar behavior compared to other experiments, with the exception of the norm growth of $\Hb$, which we conjecture to be a consequence of layer normalization in the final embedding layer. \looseness=-1

\vspace{-10pt}
\input{Arxiv_NewTempl/sections/figure_verysmall_heatmaps}

\begin{figure}[h!]
    % \centering 
    \vspace{-10pt}
    \hspace{-80pt}
    \includegraphics[width=1.3\linewidth]{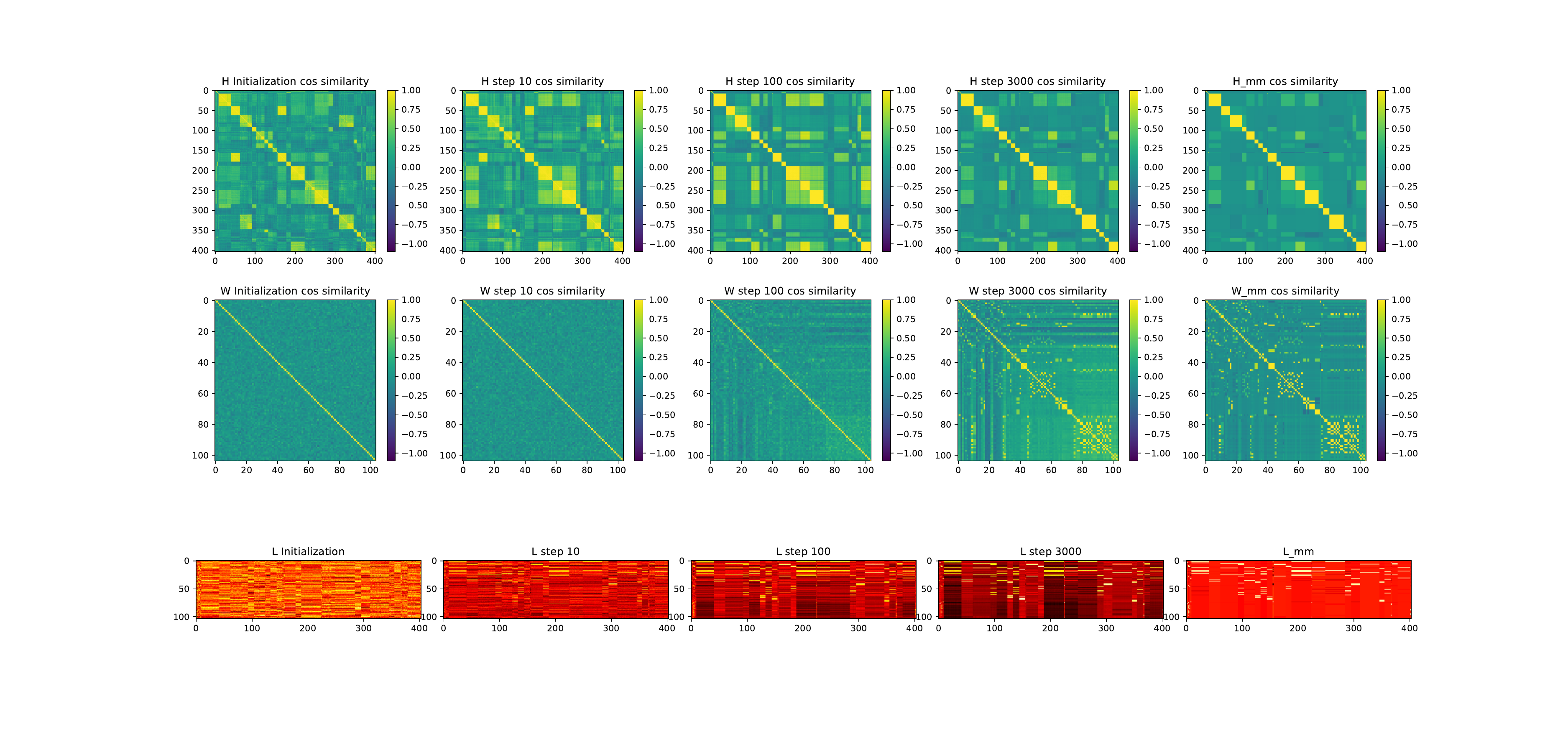}  
\captionsetup{width=\textwidth}
\vspace{-40pt}
    \caption{Evolution of $\Hb_k$, $\W_k$ and $\Lb_k$ to \ref{eq:ufm-svm} solution $\Hmm$, $\Wmm$ and $\Lbmm$ throughout training. The embeddings are trained by the transformer on the \simTS~dataset. See Sec.~\ref{sec:exp_main}. \looseness=-1 }
    \label{fig:m404steps}
\end{figure}

\begin{figure}[h!]
    % \centering 
    \hspace{-80pt}
    \includegraphics[width=1.3\linewidth]{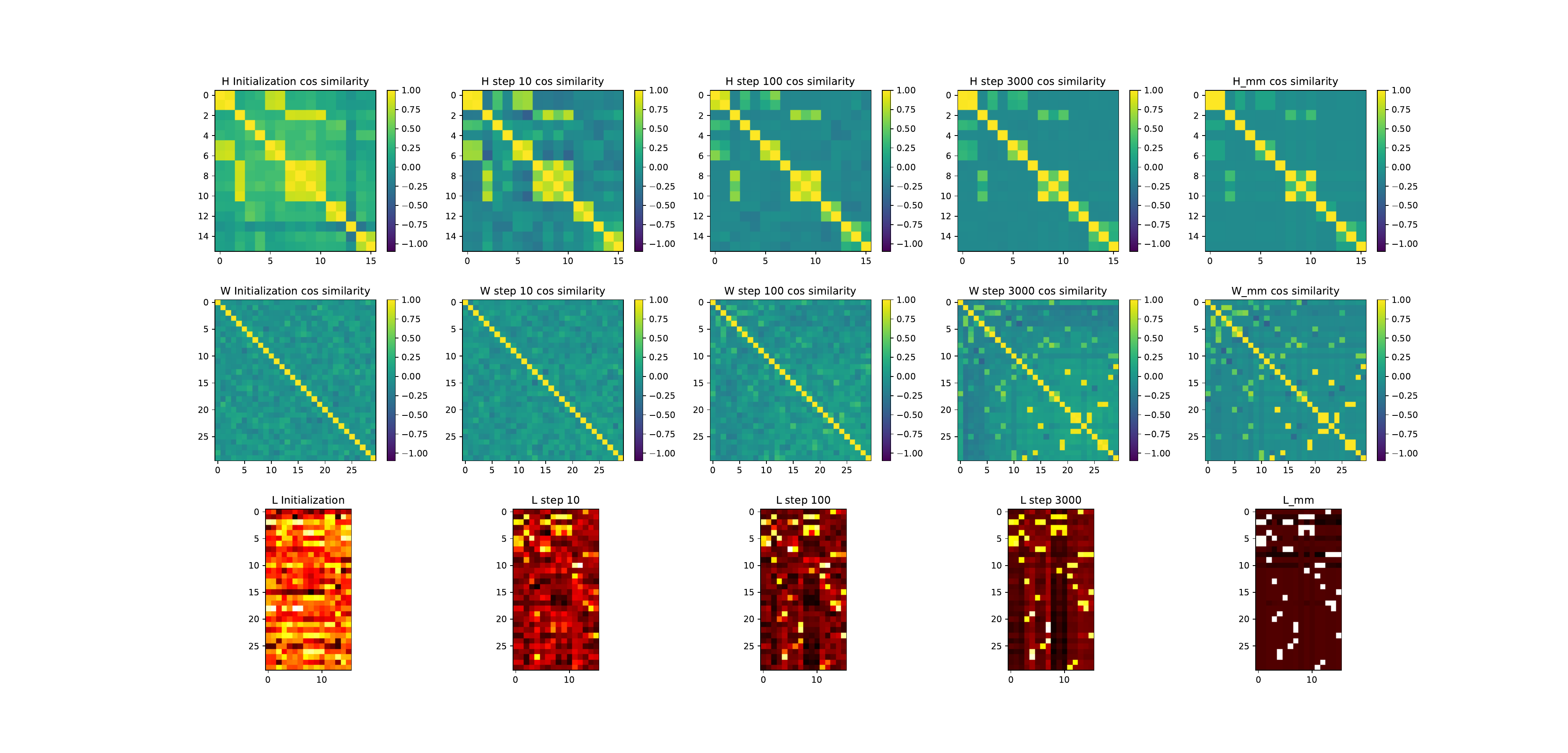}  
\captionsetup{width=\textwidth}
    \caption{Same as Fig.~\ref{fig:m404steps}, this time for the \synth~dataset. See Sec.~\ref{sec:exp_main}.  }
    \label{fig:verysmallsteps}
\end{figure}

% \begin{figure}[h!]
%     \vspace{-10pt}
%     \centering 
%     \includegraphics[width=0.8\linewidth]{exp_results_yize/dataset.png}  
% \captionsetup{width=\textwidth}
% \vspace{-10pt}
%     \caption{The contexts, support sets, and soft labels in the \synth~dataset.}
%     \label{fig:verysmalldataset}
% \end{figure}

\begin{figure}[h]
    \centering
    % \vspace{-20pt}
    \begin{tikzpicture}
        % Include the image in the figure
        \node[inner sep=0] (image) at (0,0) {\includegraphics[width=0.85\linewidth]{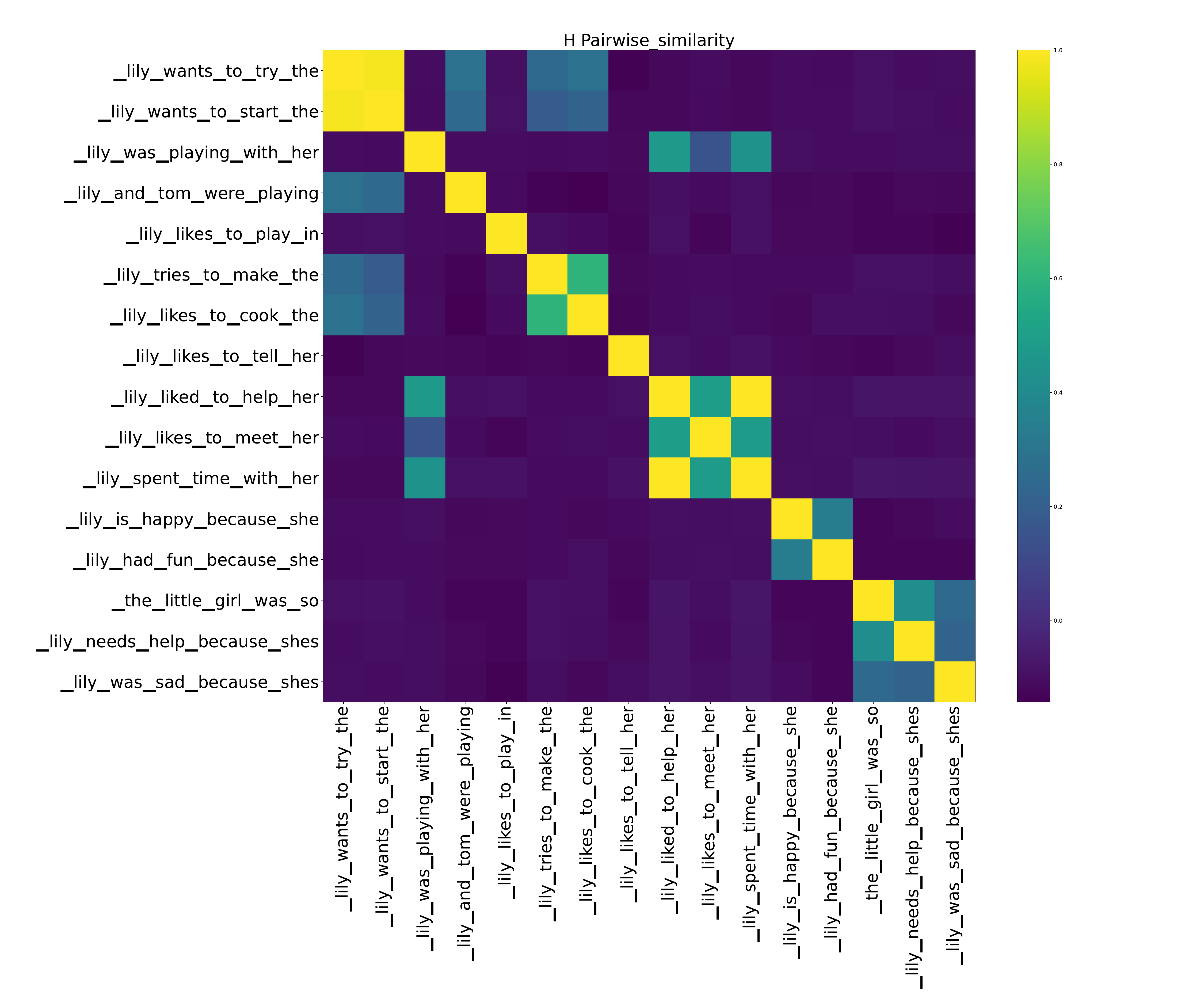}};
        \node[align=center, fill=white, fill opacity=1, text width=0.5\linewidth] at (0.5,4.8) {$\corr{\Hb}$}; 
    \end{tikzpicture}
    \captionsetup{width=\textwidth}
    % \caption{Geometry of context embeddings for \synth~dataset. A lighter color indicates higher similarity in the embedding space.}
    \label{fig:sim_H_label_verysmall}
% \end{figure}

% \begin{figure}[h]
    \centering
    \vspace{-20pt}
    \begin{tikzpicture}
        % Include the image in the figure
        \node[inner sep=0] (image) at (0,0) {\includegraphics[width=0.85\linewidth]{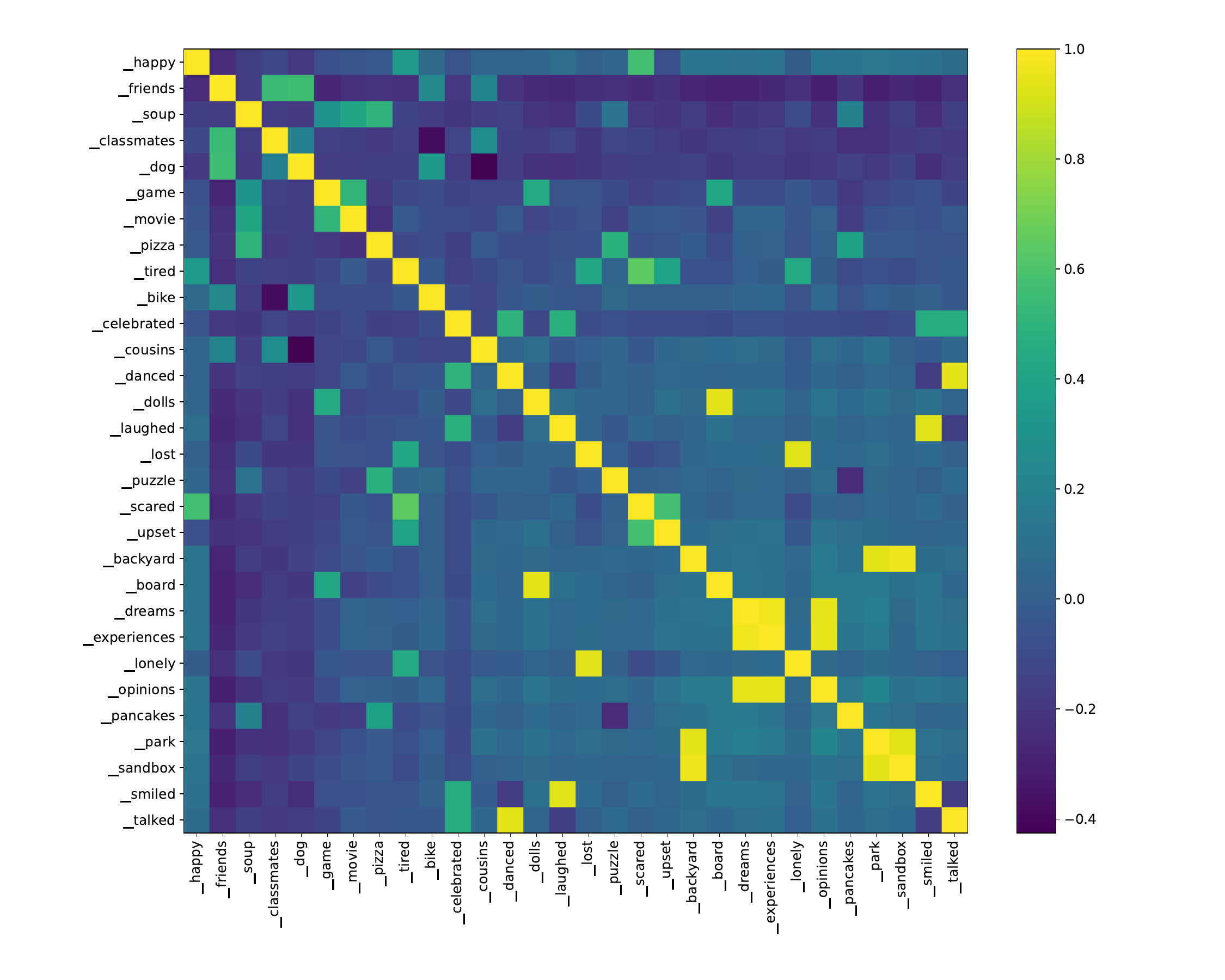}};
        \node[align=center, fill=white, fill opacity=1, text width=0.5\linewidth] at (0.0,4.8) {$\corr{\W^\top}$}; 
    \end{tikzpicture}
    \captionsetup{width=\textwidth}
    \caption{Geometry of context and word embeddings for \synth~dataset. A lighter color indicates higher similarity in the embedding space.}
    \label{fig:sim_W_label_verysmall}
\end{figure}

% \begin{figure}[h]
%     \centering
%     % \vspace{-10pt}
%     \begin{tikzpicture}
%         % Adding a text node as title above the image
%         % Include the image in the figure
%         \node[inner sep=0] (image) at (0,0) {\includegraphics[width=\linewidth]{exp_results_yize/H_sim_selected_label_4 layer Transformer pos off_tiny_extract_m404.pdf}};
%         \node[align=center, fill=white, fill opacity=1, text width=0.5\linewidth] at (-0.3,5.355) {$\corr{\Hb}$}; 
%     \end{tikzpicture}
%     % \captionsetup{width=\textwidth}
%     % \caption{Geometry of context embeddings for \simTS~dataset. A lighter color indicates higher similarity in the embedding space.}
%     \label{fig:sim_H_lable_404}
% % \end{figure}

% % \begin{figure}[h]
%     \centering
%     \vspace{-40pt}
%     \begin{tikzpicture}
%         % Adding a text node as title above the image
%         % Include the image in the figure
%         \node[inner sep=0] (image) at (0,0) {\includegraphics[width=\linewidth]{exp_results_yize/W_sim_label_404.pdf}};
%         \node[align=center, fill=white, fill opacity=1, text width=0.5\linewidth] at (-0.3,5.52) {$\corr{\W^\top}$}; 
%     \end{tikzpicture}
%     \captionsetup{width=\textwidth}
%     \caption{Geometry of context and word embeddings for \simTS~dataset. A lighter color indicates higher similarity in the embedding space.}
%     \label{fig:sim_W_lable_404}
% \end{figure}

% \newpage
\FloatBarrier

\vspace{-10pt}
\subsection{Other architectures}\label{sec:mlp}
The results in Thms.~\ref{thm:reg path} and \ref{thm:reg path finite} hold provided the model is expressive enough to generate (approximately) unconstrained embeddings and the loss can be minimized close to the lower-bound. Depending on the expressiveness of a specific network design, the size of the network required for achieving the entropy lower-bound can vary significantly. 

To explore this, we hereby repeat our experiments on \simTS~by replacing the TF model with a multi-layer perceptron (MLP). We use an MLP consisting of 20 layers organized into four blocks, each containing five layers, with hidden dimension sizes 1024, 512, 256, and 128. This leads to a network with around $8$ times more parameters than the TF model. The final embeddings geometry is depicted in Fig.~\ref{fig:Sim_SHWFixedLMLP_d128}. We observe that the geometrical patterns appearing in $\corr{\Hb}$ and $\corr{\W^\top}$ at a coarse level are similar to those of the TF model. However, the MLP, even with $\sim10$ times larger size, struggles to recover the fine-grained patterns. In terms of loss convergence, we find the MLP converges to the empirical entropy in the order of $10^{-2}$ but TF converges in the order of $10^{-4}$ on the same dataset. We conjecture that an even larger MLP might be able to achieve better convergence to our theoretical prediction, but we also caution of potential optimization bottlenecks. We encourage additional experiments with other architectures such as LSTMs or state-space models as future work. 
 
\begin{figure}[h]
    \centering
    \hspace{-10mm}
    \vspace{-10pt}
    \begin{tikzpicture}
        \node[anchor=south west, inner sep=0] (image) at (0,0) {\includegraphics[width=1.\linewidth]{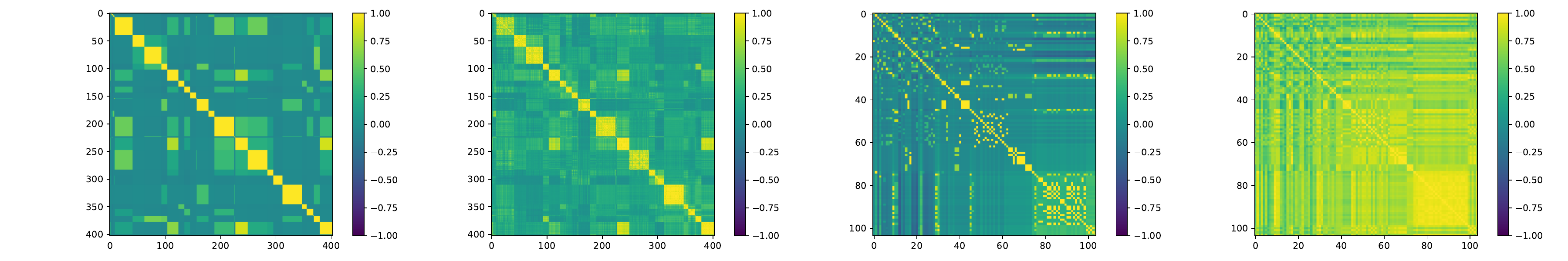}};
        \begin{scope}[x={(image.south east)}, y={(image.north west)}]
            % Position text nodes
            \node[align=center] at (0.39, 1.05) {$\corr{\Hb}$};
            \node[align=center] at (0.15, 1.05){$\corr{\smatbar}$};
            \node[align=center] at (0.885, 1.05) {$\corr{\W^\top}$};
            \node[align=center] at (0.64, 1.05) {$\corr{\smatbar ^\top}$};
        \end{scope}
        \draw[black, line width=2pt] (7.3,3.) -- (7.3,-0.1);
    \end{tikzpicture}
\captionsetup{width=\linewidth}     \vspace{4pt}
    \caption{Embedding geometries trained by MLP. \textbf{Left:} Geometry comparison between $\corr{\smatbar}$ and $\corr{\Hb}$; \textbf{Right:} Geometry comparison between $\corr{\smatbar ^\top}$ and $\corr{\W^\top}$. See App.~\ref{sec:mlp} for details.\looseness=-1}
    \label{fig:Sim_SHWFixedLMLP_d128}
\end{figure}
\vspace{-10pt}
\subsection{Auto-regressive training}\label{sec:ar}
% \vspace{-5pt}
% \yz{explain character level done.}
% \yz{try explain connection to AR better (or not needed)}
% \yz{fix notation in the caption, done}
Here, we train the model auto-regressively. In the previous settings, we focused on training samples that were all of fixed length $T-1=5$. In this section, we let the model learn the embeddings for different sequence lengths $1\leq T-1\leq 16$. Note that the length of the context does not have an impact on our theoretical analysis other than affecting the sparsity pattern $\Sb$ of the training set which in turn influences the implicit geometry as we have seen. 
% Therefore, we expect similar behavior in the AR setup as long as the model is overparameterized and trained to convergence. If the training regime does not satisfy these two conditions, the independence assumption on the embeddings, made in the UFM, can be strict for the input phrases that are chosen auto-regressively with significant overlap in their contexts. \looseness=-1

We train an 8-layer TF on 200 stories from \TS~using character-level tokenizer, which limits the vocabulary to approximately 40 characters and promotes higher entropy in next-token distribution. We denote the loss value across different positions $T\in\{2,\cdots,17\}$ at iteration $k$ by $\CE(\Lb_{T,k})$. In Fig.~\ref{fig:positional_losses}, we display the distance of each loss component $\CE(\Lb_{T,k})$ from its empirical entropy lower-bound $\Hc_T$, i.e., the T-gram entropy of the training set. We observe generally better convergence for shorter sequence lengths $T$.
% 
% initial tokens (smaller $T$) and a larger convergence gap for longer subsequences (larger $T$), suggesting that shorter subsequences are more effectively learned in auto-regressive training.
In Fig.~\ref{fig:H_S_sim_label_ar}, we illustrate the context embeddings similarities.  For visualization, we select 5 contexts per each sequence length that end in token ``y\_'' or ``\_t''. The context embeddings and the proxy \ref{P proxy} show similar patterns at a coarse level. However, at finer scale, the learned context embeddings that end with the same tokens exhibit strong correlation on average, even when their support sets do not align. We defer further investigation into whether the embeddings' dependence is due to insufficient network capacity or an optimization bottleneck to future work. \looseness=-1
% 
% 
% The heatmaps display pairwise similarities for contexts ending with identical tokens (top-left and bottom-right) and dissimilarities for contexts ending with different tokens (top-right and bottom-left). Additionally, Similar to heatmaps in the main body, the comparison also shows similar patterns between embedding correlation $\corr{\Hb}$ and the proxy $\corr{\smatbar}$. \tina{are the patterns that similar?}
%, facilitating a more structured comparison and labeling of contexts in heatmaps, which show clear blocks of similarity and dissimilarity.

% Loss convergence is measured against $T$-gram entropy for context lengths from 1 to 16, as illustrated in Fig.~\ref{fig:H_S_sim_label_ar}.
 % We specifically select contexts ending with the same two tokens because they are likely to have similar support sets and embeddings, facilitating a more structured comparison and labeling of contexts in heatmaps, which show clear blocks of similarity and dissimilarity. In Fig.~\ref{fig:positional_losses}, we include the loss convergence across different positions, ranging from \(T=2\) to \(T=17\).

 % \tina{@Yize: can you please discuss the details? what are you showing in the loss curve and heatmap (how have you sorted/chosen the contexts, etc...) + can you put these two figures together?}

\begin{figure}[ht]
    % \centeri ng
    \hspace{-50pt}
   \begin{tikzpicture}
        % Include the image in a node for precise control
        \node[inner sep=0pt] (image) at (0,0) {\includegraphics[width=1.13\linewidth]{yize_figures_final/H_S_sim_label_ar.pdf}};
        
        % Add a text node at the top of the image
        \node[align=center] at ($(image.north)-(3,0.5)$) {$\corr{\smatbar}$};

        % Add a text node at the bottom of the image
        \node[align=center] at ($(image.north)+(3,-0.50)$) {$\corr{\Hb}$};
    \end{tikzpicture}
    \vspace{-7pt}\captionsetup{width=\textwidth}
        \caption{
        % Comparison of geometric cosines between $\corr{\smatbar}$ and $\corr{\Hb}$ in auto-regressive experiments. We selected 5 contexts per length from 2 to 16 for endings ``y\_'' and ``\_t'' to display. The heatmaps display pairwise similarities for contexts ending with identical tokens (top-left and bottom-right) and dissimilarities for contexts ending with different tokens (top-right and bottom-left). Additionally, Similar to heatmaps in the main body, the comparison also shows similar patterns between embedding correlation $\corr{\Hb}$ and the proxy $\corr{\smatbar}$.  
        Geometry of context embeddings $\corr{\Hb}$ and the heuristic proxy \ref{P proxy}, $\corr{\smatbar}$, in the autoregressive experiment of App.~\ref{sec:ar}.}
    % tina{@Yize: larger label font (both heatmap and colorbar) + x labels are written over each other, make them vertical.}}
    \label{fig:H_S_sim_label_ar}
\end{figure}

\begin{figure}[ht]
\vspace{-5pt}
    \centering
    \resizebox{0.7\textwidth}{!}{
    \begin{tikzpicture}
        % Include the image in a node
        \node[anchor=south west,inner sep=0] (image) at (0,0) {\includegraphics[width=0.8\linewidth]{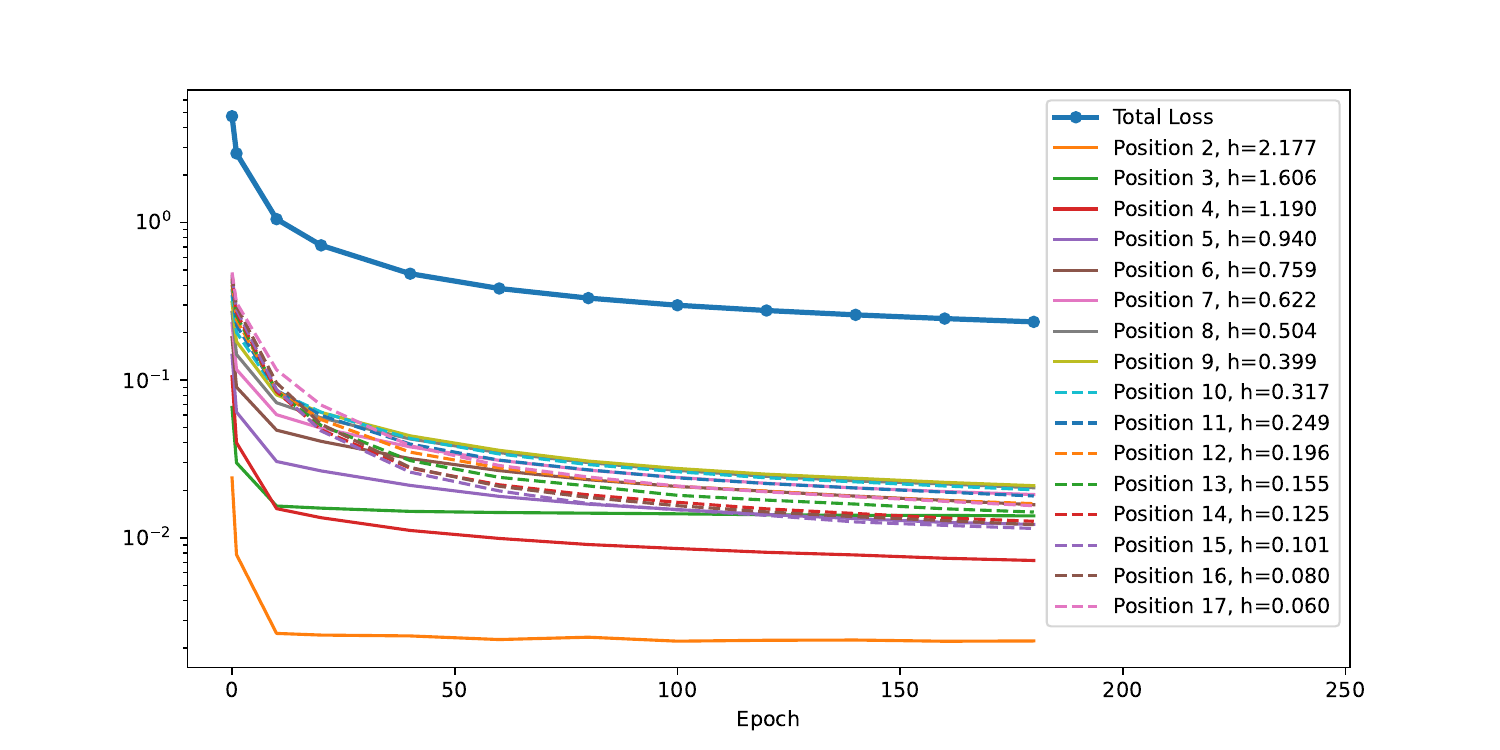}};
        % Add a text node at a specific location
        \node[align=center, fill=white, fill opacity=0.5, text opacity=1, text width=10cm] at ($(image.center)+(0,2.7)$) {Loss convergence across positions};
        \node[align=center, fill=white, fill opacity=0.5, text opacity=1, text width=5cm, rotate=90] at ($(image.center)-(5.5,0)$) {\small $\text{CE}(\Lb_{T,k})-\Hc_T$};

    \end{tikzpicture}
    }
    \captionsetup{width=\textwidth}
    \caption{TF trained autoregressively on a subset of \TS~dataset: Loss convergence to its empirical entropy lower-bound for each sequence length $T=2,\cdots,17$. See App.~\ref{sec:ar} for details.
    % The results demonstrate varied convergence curves across positions, with generally better convergence for initial tokens (smaller $T$) and a larger convergence gap for longer subsequences (larger $T$), suggesting that shorter subsequences are more effectively learned in auto-regressive training.
    }
    \vspace{-5pt}\label{fig:positional_losses}
\end{figure}

\subsection{Discussion on large vocabulary setting}\label{sec:d<V}
% \vspace{-5pt}
% 
Throughout our analysis, we require the embedding dimension to be larger than the vocabulary size, i.e., $d\geq V$. This condition was necessary to make the non-convex problem in \eqref{eq:ufm relax} convex. 
However, if there exists a low-rank optimal solution in \ref{eq:ufm-svm}, i.e., $\rank{\Lbmm}<V$, the condition $d\geq V$ can be relaxed in the analysis. In general, the smallest rank among all minimizers depends on the sparsity pattern of the language. It is intriguing to investigate this dependence further.

% For now, as an initial experiment, we focus on the setup of Simplified TinyStories as described in Sec.~\ref{sec:exp_main}. 
% \tina{@Yize: complete with your exp setup and results please}
% This section examines the network's performance when the decoder dimension $d$ is less than $V$, contrary to the requirement of $d>V$ specified in Equation \ref{eq:ufm relax}. By fixing the context embeddings' expressibility, we investigate the impact of varying $d$ on the geometry embeddings. To maintain adequate expressivity, we configure the network with ten transformer layers, each with a dimensionality of $d_{TF}=64$. We then introduce an extra linear layer to adjust $d$ to analyze the effects on loss convergence and geometric properties. We observe that the geometry of embeddings show similar patterns as also follows the proxy $\smatbar$. We would like to understand this more and we leave this as future work. 

This section examines the network's performance when the decoder dimension \(d\) is smaller than the vocabulary size \(V\). To ensure a fair comparison between different setups, we fix the transformer blocks' inner embedding dimension $d_\text{TF}=64$ to maintain comparable expressivity for all networks. To vary the final layer embedding dimension $d$, we add an additional linear layer on top of the network to adjust the context embedding dimension. 

In Fig.~\ref{fig:d<V}, we show the learned context embeddings for $d=128$ and $d=64$, for a 10-layer TF trained on the \simTS~dataset with $V=104$. We observe that with smaller $d$, the speed of convergence to the entropy lower-bound $\Hc$ decreases. However, the learned context embeddings exhibit relatively similar geometry with moderate values of $d<V$. We leave more theoretical and empirical exploration of this setting to future works.

\begin{figure}[h]
\centering
\vspace{-10pt}
		% \centering
  % \hspace{-55pt}
    \begin{tikzpicture}
        \node at (-0,-0) 
        {\includegraphics[scale=0.27]{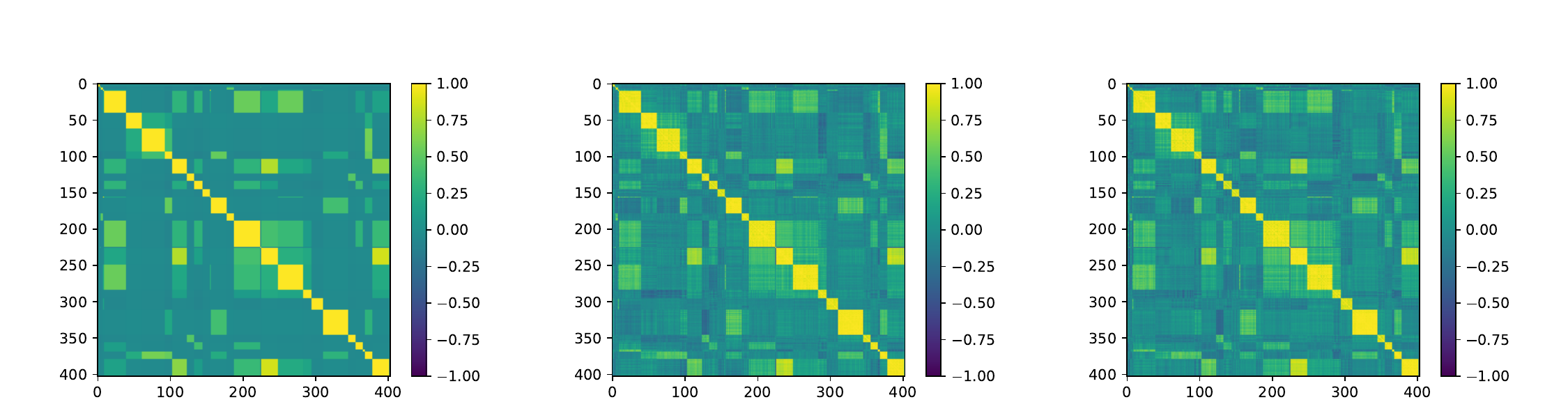}};
        \node at (-3.5,1.2) [align=center,scale=0.9]{$\corr{\smatbar}$};
        \node at (-0.2,1.2) [align=center,scale=0.9]{$\corr{\Hb}$, $d=128$};
        \node at (3.5,1.2) [align=center,scale=0.9]{$\corr{\Hb}$, $d=64$};
        \node at (0.0,-1.7) [scale=0.9]{(a) cosine similarity heatmaps};
        \node at (7,-0) {\includegraphics[scale=0.27]{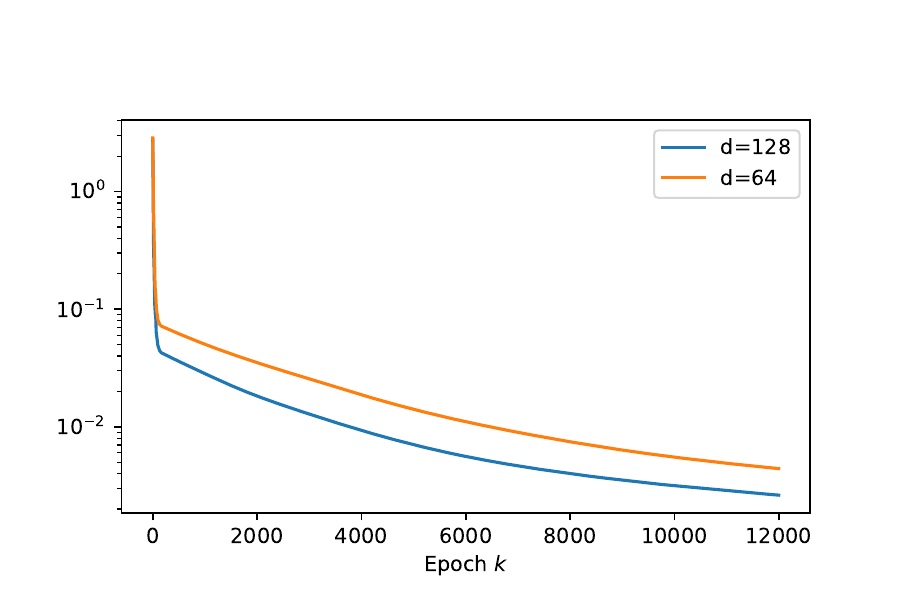}};
    
       % \node at (0,1.4) [align=center,scale=0.9]{Loss Convergence};
       \node at (7,-1.7) [scale=0.9]{(b) Loss convergence};
       \node at (5,0) [align=center,scale=0.4, rotate=90]{$\text{CE}(\Lb_{k})-\Hc$};
   \end{tikzpicture}

    \captionsetup{width=\textwidth}
    \vspace{-7pt}
    \caption{Varying embedding dimension $d$ in the \simTS~experiment. \textbf{(a)} Left: Proxy \ref{P proxy} for context embeddings. Middle and Right: Geometry of context embeddings trained with TF with $d=128>V$ and $d=64<V$, respectively, \textbf{(b)} Loss convergence to the empirical entropy lower bound $\Hc$. See App.~\ref{sec:d<V} for details.\looseness=-1}

    \label{fig:d<V}
    \end{figure}

\newpage
% \tina{sections below are extra: anything still needed to be moved to previous sections??}
 % \input{Arxiv_NewTempl/sections/literature_notes}
% \input{Arxiv_NewTempl/sections/ufm_do_not_upload}
% \input{Arxiv_NewTempl/sections/ToDo}

% \input{Arxiv_NewTempl/sections/autoregressive}

%%%%%%%%%%%%%%%%%%%%%%%%% old
% \input{Arxiv_NewTempl/sections/proofs}
% \input{Arxiv_NewTempl/sections/geometry_app}
% % \input{Arxiv_NewTempl/sections/expyize}
% \input{Arxiv_NewTempl/sections/expvala}
% % \input{Arxiv_NewTempl/sections/rank_feas}

% \newpage
% \newpage
% % \input{Arxiv_NewTempl/sections/ufm_do_not_upload}
% % \input{Arxiv_NewTempl/sections/misc}
% \input{Arxiv_NewTempl/sections/apdxyize}
% \input{Arxiv_NewTempl/sections/ToDo}

% % \input{Arxiv_NewTempl/sections/autoregressive}

\end{document}